%% file: arxiv.tex
\documentclass[letterpaper,11pt]{article}
\usepackage[margin=1in]{geometry}
\input{preamble.tex}

\usepackage{orcidlink}
\usepackage[mathcal]{euscript}
\usepackage{bbold}
\usetikzlibrary{arrows.meta, positioning}
\usepackage{bibunits}
\usepackage{makecell}
\usepackage{cancel}
\usepackage{subcaption}
\usetikzlibrary{calc}
\usetikzlibrary{decorations.pathreplacing}
\def\tilde{\widetilde}

\newcommand{\pdfu}{\texorpdfstring{$U$}{U}}
\newcommand{\pdfv}{\texorpdfstring{$V$}{V}}
\newcommand{\libref}[2]{\href{#2}{\texttt{#1}}}
\def\package{\libref{u-stats}{https://github.com/zrq1706/U-Statistics-Python}}
\def\igraph{\libref{igraph}{https://igraph.org/}}

\def\tilde{\widetilde}

\newtheorem{subexample}{Example}

\def\Einsum{\textsc{Einsum}}

\usepackage[justification=centering]{caption}
\AtBeginDocument{\captionsetup[table]{width=\textwidth}}

\newcommand{\mobius}{M\"obius}

\newcommand{\numpy}{\libref{numpy}{https://numpy.org/}}
\newcommand{\torch}{\libref{pytorch}{https://pytorch.org/}}

\newcommand{\sPi}{\pi}
\newcommand{\sper}{\sigma}

\newcommand{\uset}{\calU}
\newcommand{\usetnm}[2]{\uset\left(#1,#2\right)}
\newcommand{\vset}{\calV}
\newcommand{\vsetp}[2]{\vset(#1,#2)}

\newcommand{\ustat}{\mathbb{U}}
\newcommand{\vstat}{\mathbb{V}}

\newcommand{\ustatp}[2]{\ustat(#1,#2)}
\newcommand{\vstatp}[2]{\vstat(#1,#2)}

\def\sfr{\mathsf{r}}

\newcommand{\samplespace}{\bbX}
\newcommand{\takesetp}[1]{\mathsf{set}(#1)}
\newcommand{\taketuplep}[1]{\mathsf{tuple}\left( #1 \right)}

\newcommand{\perm}[1]{\mathsf{Perm}(#1)}
\newcommand{\partition}{\Pi}

\newcommand{\partitionp}[1]{\partition_{#1}}

\newcommand{\indexset}{\textsc{Indices}}

\newcommand{\tensorcontraction}{\textsc{Einsum}}

\newcommand{\pleq}{\preceq}

\newcommand{\pless}{\prec}

\newcommand{\degp}[2]{\deg(#1,#2)}

\newcommand{\verticesp}[1]{V(#1)}
\newcommand{\edges}[0]{E}
\newcommand{\edgesp}[1]{E(#1)}
\newcommand{\neighborvg}[2]{N_{#2}(#1)}

\newcommand{\tw}{\operatorname{tw}}
\newcommand{\twp}[1]{\operatorname{tw}(#1)}

\newcommand{\graphset}{\mathcal{G}}
\newcommand{\graphsetne}[2]{\graphset_{#1,#2}}

\newcommand{\degeneracy}{D}
\newcommand{\degeneracyg}[1]{\degeneracy(#1)}
\newcommand{\completegraphp}[1]{K_{#1}}

\newcommand{\motifrg}[2]{\mathsf{C}(#1,#2)}
\newcommand{\graphform}[1]{G_{#1}}

\newcommand{\less}{O}

\newcommand{\appequal}{\Theta}

\newcommand{\eliminatevg}[2]{\mathsf{elim}_{#1}(#2)}

\newcommand{\Zhang}[1]{{\color{blue}{Zhang: #1}}}
\newcommand{\Chen}[1]{{\color{orange}{Chen: #1}}}
\newif\ifshowchanges
\showchangesfalse

\newcommand{\rev}[1]{%
  \ifshowchanges
    {\color{red} #1}%
  \else
    #1%
  \fi
}

\def\package{\libref{u-stats}{https://github.com/zrq1706/U-Statistics-Python}}
\def\igraph{\libref{igraph}{https://igraph.org/}}
\def\peregrine{\libref{Peregrine}{https://github.com/pdclab/peregrine}}
\def\cugraph{\libref{cuGraph}{https://docs.rapids.ai/api/cugraph/stable/}}

\newcommand{\numpyeinsum}{\libref{numpy.einsum}{https://numpy.org/doc/2.1/reference/generated/numpy.einsum.html}}

\newcommand{\torcheinsum}{\libref{pytorch.einsum}{https://pytorch.org/docs/stable/generated/torch.einsum.html}}
\newcommand{\opteinsum}{\libref{opt-einsum}{https://dgasmith.github.io/opt_einsum/}}

\definecolor{backcolour}{rgb}{0.95,0.95,0.92}
\definecolor{darkblue}{rgb}{0.0, 0.0, 0.55}

\def\bbX{\mathbb{X}}
\def\calA{\mathcal{A}}
\def\calB{\mathcal{B}}
\def\calT{\mathcal{T}}
\def\calN{\mathcal{N}}
\def\calH{\mathcal{H}}
\def\calV{\mathcal{V}}
\def\calU{\mathcal{U}}

\begin{document}


\title{\bf On computing and the complexity of computing higher-order $U$-statistics, exactly}
\author[1]{Xingyu Chen\orcidlink{0009-0008-0823-4406}%
	\thanks{E-mail: \href{mailto:xingyuchen0714@sjtu.edu.cn}
		{xingyuchen0714@sjtu.edu.cn}}}

\author[2,3]{Lin Liu\orcidlink{0000-0002-9883-7962}%
	\thanks{E-mail: \href{mailto:linliu@sjtu.edu.cn}
		{linliu@sjtu.edu.cn}}}

\author[2,3]{Ruiqi Zhang\orcidlink{0009-0006-7369-0489}%
	\thanks{E-mail: \href{mailto:zrq1706@outlook.com}
		{zrq1706@outlook.com}.
		All authors are alphabetically ordered and contribute equally to the paper.
		The authors thank the two anonymous referees and the associate editor for
		insightful comments that help to improve the quality of our paper.
		The authors also benefit from discussions with Lei Li, Shan Luo,
		Rajarshi Mukherjee, Yixuan Qiu, Jamie Robins, Xinwei Shen,
		Cheng Wang, and Qingyuan Zhao.
		The authors are supported by National Key R\&D Program of China
		Project Number 2025YFA1016700, NSFC Grant No.~12471274,
		and Science and Technology Talent and Platform Program of
		Yunnan Province Grant No.~202605AF35007.}}

\affil[1]{School of Mathematical Sciences,
	Shanghai Jiao Tong University}

\affil[2]{Institute of Natural Sciences, MOE--LSC,
	Shanghai Jiao Tong University}

\affil[3]{School of Mathematical Sciences, CMA--Shanghai,
	SJTU--Yale Joint Center for Biostatistics and Data Science,
	Shanghai Jiao Tong University}
\maketitle
\thispagestyle{empty}

\begin{abstract}
	{}Higher-order $U$-statistics abound in fields such as statistics, machine learning, and computer science, but are known to be highly time-consuming to compute in practice. Despite their widespread appearance, a comprehensive study of their computational complexity is surprisingly lacking. This paper aims to fill this gap by presenting several results related to the computational aspect of $U$-statistics. First, we derive a useful decomposition from a $m$-th order $U$-statistic to a linear combination of $V$-statistics with orders not exceeding $m$, which are generally more feasible to compute. Second, we explore the connection between exactly computing $V$-statistics and Einstein summation, a tool often used in computational mathematics and quantum computing to accelerate tensor computations. Third, we provide an optimistic estimate of the time complexity for exactly computing $U$-statistics, based on the treewidth of a particular graph associated with the $U$-statistic kernel. The above ingredients lead to (1) a new, much more runtime-efficient algorithm to exactly compute higher-order $U$-statistics whose kernels admit a multiplicative-decomposable structure, and (2) a more streamlined characterization of runtime complexity of computing $U$-statistics. We develop an accompanying open-source package called \texttt{u-stats} in both \href{https://github.com/zrq1706/U-Statistics-Python}{Python} and \href{https://cran.r-project.org/web/packages/ustats/index.html}{R} (available on CRAN). We demonstrate through three examples in statistics that \texttt{u-stats} achieves impressive runtime performance compared to existing benchmarks.  This paper also aspires to achieve two goals: (1) to capture the interest of researchers in both statistics and other related areas to further advance the algorithmic development of $U$-statistics and (2) to lift the burden of implementing higher-order $U$-statistics from practitioners.
\end{abstract}

\noindent%
\vfill

\newpage

\onehalfspacing

\setcounter{page}{1}

\section{Introduction}
\label{sec:intro}
Higher-order $U$-statistics are prevalent in various disciplines, including statistics \citep{robins2008higher, bonvini2026fast, bonvini2024doubly, yao2018testing, breunig2024adaptive, shao2025u}, machine learning \citep{clemenccon2008ranking, clemenccon2013maximal}, random graphs \citep{chatterjee2026higher}, and Bayesian networks \citep{shpitser2011efficient}, just to name a few. It is well-known that $U$-statistics are highly time-consuming to compute at large. In the worst-case, an $m$-th order $U$-statistic over a sample of size $n$ takes roughly $O (n^{m})$ steps to compute, because by definition (see Definition~\ref{def:u_stat}) one needs to sum over all $m$-out-of-$n$ permutations. If $n = 10^{4}$ and $m = 7$, $n^{m} \approx 10^{28}$!

More specifically, in causal inference, \citet{robins2008higher} developed a so-called higher-order influence function (HOIF) framework for constructing nearly rate-optimal estimators of the average treatment effects under unconfoundedness and various structural assumptions, using higher-order $U$-statistics with order roughly of $\log n$. In spite of their appealing statistical properties, the implementation remains a challenge in practice, especially computationally. Here we quote from \citet{wager2024causal} to highlight this dilemma:
\begin{quote}
	\emph{One challenge with the HOIF estimator ... is that to date it has been challenging to implement in practical applications ...}
\end{quote}
This point is echoed in a recent review article on the future of causal inference \citep{cinelli2025challenges} (also see \citet{zheng2025perturbed} for a related discussion):
\begin{quote}
	\emph{... where estimators are available that can improve on doubly robust style methods ... the (higher-order) estimators are often computationally intensive and require careful tuning. There are many opportunities to make these methods more practical, automatic, and user-friendly.}
\end{quote}

Our paper was initially motivated by (at least partially) addressing the computational aspect of this challenge, but we will present a general-purpose algorithm (and some theory) for computing generic higher-order $U$-statistics beyond those for HOIF estimators. We will revisit HOIF estimators in Section~\ref{sec:example_hoif}.

To tackle the difficulty of computing higher-order $U$-statistics, statisticians and computer scientists have long resorted to the so-called ``incomplete $U$-statistics'', where one sums over only a (random) subset of all $m$-out-of-$n$ permutations \citep{blom1976some, nowicki1988subgraph, shao2025u}. Another natural approach is to choose a subsample of size $n' < n$ from the whole sample and then compute a complete $U$-statistic from that subsample. The randomization approach injects extra randomness into the final statistical results, which can produce inconsistent statistical conclusions between different randomly selected permutations/subsamples \citep{guo2025rank}. For a restricted class of higher-order $U$-statistics with decomposable and permutation-invariant kernels, \citet{kong2018estimating} astutely recognized that exact computation can be accomplished by repeated matrix multiplications. For asymmetric kernels, their strategy essentially computes an incomplete $U$-statistic instead. A somewhat similar strategy has also been considered in \citet{zhang2025adaptive}, in which an incomplete $U$-statistic with a special kernel structure is computed more efficiently via dynamic programming. All of the above approaches may suffer from efficiency loss by not using all the available information.

The literature on (the complexity of) exactly computing higher-order $U$-statistics is rather scarce. This is hardly surprising, because the exact time complexity of computing higher-order $U$-statistics seems to simply equal the total number of summations. However, as we will demonstrate in the paper, exactly computing $U$-statistics can be improved compared to a brute-force summation over all summands (using for-loops) when the kernel enjoys a particular decomposable structure specified in Definition~\ref{def:mul_decomp} later. Interestingly, as we will exhibit in Section~\ref{sec:examples}, such a structure holds in many statistical applications. It will also be clear later that it is helpful to bring in tools such as the Einstein summation from the computational mathematics community to obtain improved algorithms of computing $U$-statistics by leveraging such a structure. It is worth noting that Section~3 of \citet{liu2024hypothesis} considered to use di-graphs to represent the structure of a class of $U$-statistics, a strategy conceptually similar to ours, but we aim to cover more general higher-order $U$-statistics. In a similar vein, \citet{he2021asymptotically} proposed an iterative algorithm to compute a specific class of $U$-statistics for high-dimensional testing; however, their approach is not readily generalizable to more complex settings. By contrast, our strategy subsumes their problem as a special case, as demonstrated in Appendix~\ref{app:examples_complexity}.

\subsection*{Main contributions}

The main contributions of our paper can be summarized as follows.

\begin{itemize}
	\item We design a new method for exactly computing higher-order $U$-statistics and another related object called $V$-statistics by leveraging the Einstein summation operation and related concepts from graph theory. Using graph theory, we provide an optimistic estimate of the time complexity of exactly computing $U$-statistics (see Section~\ref{sec:alg_u}). Although the idea is conceptually simple, to the best of our knowledge, it has not been thoroughly explored in the statistical literature.

	\item Despite the ubiquity of $U$-statistics, a general-purpose software package to compute them is lacking. To fill in this void, we developed a package \package{} (available in \href{https://github.com/zrq1706/U-Statistics-Python}{python} and \href{https://github.com/cxy0714/U-Statistics-R}{R}) implementing the aforementioned algorithm. We demonstrate that its runtime achieves the state-of-the-art performance in various examples, including computing HOIF estimators, motif counts of networks, and distance covariances in independence testing problems. In fact, as will be clear in Section~\ref{sec:examples}, \package{} can dramatically improve the runtime compared to popular benchmarks. \package{} currently supports both CPU and GPU infrastructures. Having access to GPU can further improve the runtime because our algorithm heavily utilizes tensor operations such as the Einstein summation. 

\end{itemize}

\subsection*{Related Literature}

In this part, we provide an incomplete survey of the related literature on $U$-statistics and the relevant techniques/tools to develop our new algorithm.

\subsubsection*{Statistics and machine learning}

$U$-statistics are frequently found in statistics and machine learning. For example, in certain machine learning tasks \citep{clemenccon2008ranking, clemenccon2013maximal, shen2025engression}, the empirical training risk can be written as a $U$-statistic instead of a sample mean. Various (conditional) dependence measures can also be represented as $U$-statistics \citep{yao2018testing, he2021asymptotically, shao2025u, zhang2025adaptive}; see Section~\ref{sec:dependence} for an example. As mentioned before, in causal inference, HOIF has been shown to be a theoretically appealing framework for constructing nearly rate-optimal estimators for various causal effect parameters. Estimators based on HOIF are higher-order $U$-statistics; see Section~\ref{sec:example_hoif} for concrete forms. Our work is also related to the literature on probabilistic graphical models \citep{pearl1988probabilistic, shpitser2011efficient}, in which a common problem is to compute certain marginal distributions from the joint distribution, factorized according to some Bayesian network. It can be reduced to computing $U$-statistics with kernels satisfying a particular structure to be defined later.

\subsubsection*{Motif counts in random graphs}

Higher-order $U$-statistics also appear very often in statistical inference over networks, often modeled as random graphs \citep{chatterjee2026higher, jin2025counting, jin2026counting}. For instance, \citet{chatterjee2026higher} estimate motif counts (such as triangles, 4-cliques, etc.) by higher-order $U$-statistics and use these estimated motif counts to test certain properties of the underlying network. Motif counts also have widespread applications in bioinformatics \citep{wernicke2006fanmod}. As will be shown in Section~\ref{sec:example_motif}, using the state-of-the-art software package \igraph{}, for a moderately dense network, computing motif counts can be quite time consuming, necessitating the development of new algorithmic tools for $U$-statistics.


\subsubsection*{Tensor operations}

Since computing $U$/$V$-statistic is closely related to numerical integration using Monte Carlo sampling for which tensor operations are quite useful, our work is inspired by related techniques developed in the computational mathematics and physics communities. In particular, our algorithm utilizes the Einstein summation operation, an important tool for tensor computations. For example, in quantum computing/information, tensor networks are often used to describe many-body quantum systems and a very useful operation called tensor contraction is implemented by the Einstein summation in practice \citep{markov2008simulating}. As will be seen later, our proposed algorithm is similar to tensor contraction in spirit.

\subsection*{Notation and organization}

In this section, we first collect some frequently used notation. We adopt the following convention: we use the demarcation $\{\}$ for a set, representing an unordered collection of distinct elements; whereas we use the demarcation $()$ for a tuple, meaning an ordered collection of elements which may contain duplicates. Given a positive integer $m$, let $[m] \coloneqq \{1, 2, \ldots, m\}$. To reduce clutter, we write $\bar{i}_{m} \equiv (i_{1}, \cdots, i_{m})$ for short. $\naturals$ denotes the set of natural numbers that includes zero, while $\positivenaturals$ denotes the set of all positive integers. Given a tuple $\bm{s}$, we let $\takesetp{\bm{s}} \coloneqq \{s_k \mid k\in [m]\}$ denote the set containing all the elements that appear in $\bm{s}$. Given a non-empty and finite set $S$ and a positive integer $m$, all $m$-permutations of elements in $S$ are denoted by $\perm{S, m}$. All partitions of $S$ are denoted by $\partitionp{S}$. If $S$ happens to be an integer set $[m]$, $\partitionp{S}$ is simplified to $\partitionp{m}$. Given sets $S_{1}, \cdots, S_{m}$, $\times_{i = 1}^{m} S_{i}$ denotes their Cartesian product. When $S_{1} = \cdots = S_{m} = S$, we simply write $S^{m} \equiv \times_{i = 1}^{m} S_{i}$. $S^* \coloneqq \cup_{m \in \positivenaturals} S^m$ denotes all tuples possibly generated by $S$. Given a tuple $\bm{s} \in S^m$ and a tuple $\alpha \in [m]^*$, we let $\bm{s}[\alpha] \coloneqq (s_
	{\alpha_k})_{k=1}^{|\alpha|}$. To describe computational complexity, we write $\appequal(\cdot)$ to indicate equality up to a constant factor and $\less(\cdot)$ to indicate an upper bound up to a constant factor.

Terminologies from graph theory will be frequently encountered. Throughout this paper, all graphs are assumed to be finite and undirected without self-loops, but are allowed to have isolated vertices, unless stated otherwise. A graph $G$ is associated with a pair $(V, E)$, where $V \equiv V (G)$ denotes the set of vertices and $E \equiv E (G)$ denotes the set of edges in $G$. Two vertices $u$ and $v$ are said to be adjacent if $\{u, v\} \in E$. A complete graph is a graph with every pair of vertices connected by an edge, and a complete graph with $m$ vertices is denoted by $\completegraphp{m}$. For any vertex $v \in V$, its neighbor set is denoted as $N_G (v) \coloneqq \{ u \in V \mid \{u, v\} \in E \}$, and the degree of $v$ is defined as $\degp{v}{G} \coloneqq |N_G(v)|$. A graph $G^\prime$ is a subgraph of graph $G$, if $V (G^\prime) \subseteq V (G)$ and $E (G^\prime) \subseteq E (G)$. For any subset $V' \subseteq V$, the induced subgraph $G_{V'}$ has a vertex set $V'$ and an edge set $E_{V'} \coloneqq E (G_{V'}) \equiv \{\{u, v\} \in E \mid \forall \, u, v \in V'\}$. A subset of vertices $C \subseteq V$ is said to be a clique if $G_{C}$ is complete.

The remainder of our paper is structured as follows. Section~\ref{sec:preliminaries} introduces and reviews several concepts central to the paper, including $U$/$V$-statistics and tensors (Section~\ref{sec:pre_u_v}), the Einstein summation (Section~\ref{sec:pre_tensor_contraction}), and notions (including vertex elimination, treewidth, and quotient graphs) from graph theory (Section~\ref{sec:pre_graph}). Readers who are familiar with the related materials are recommended to skip the relevant subsections. In Section~\ref{sec:alg_u_v}, the new algorithm is presented. We then demonstrate the practical merit of our algorithm using three numerical examples in Section~\ref{sec:examples}, including the computations of HOIF estimators, motif counts in networks (Section~\ref{sec:example_motif}), and distance covariances (Section~\ref{sec:dependence}). Section~\ref{sec:conclusions} concludes the paper with a discussion of future directions.

\section{Preliminaries}
\label{sec:preliminaries}

In this section, we will review concepts important for later development, including $U$/$V$-statistics and tensors (Section~\ref{sec:pre_u_v}), the Einstein summation operation (Section~\ref{sec:pre_tensor_contraction}), and useful notions from graph theory (Section~\ref{sec:pre_graph}). The way in which we define $U$- or $V$-statistics may look strange to statisticians, but we find such definitions more convenient to work with when combined with the Einstein summation formulation.

\subsection{\pdfu/\pdfv-Statistic and Tensor}\label{sec:pre_u_v}

We define $U$-statistics and $V$-statistics as follows. As we just focus on computation, the sample space is regarded as a non-empty set.

\begin{definition}[$U$-Statistic]\label{def:u_stat}
	Let $\bbX$ be a non-empty set, $m, n$ be positive integers satisfying $m \leq n$, and $h: \bbX^{m} \rightarrow \bbR$ be a function. Given a tuple $\bm{X} \in \bbX^n$, any $U$-statistic of order $m$ with kernel $h$ takes the following form
	\begin{align*}
		\ustatp{h}{\bm{X}} \coloneqq \sum_{\alpha \in \perm{[n],m}} h(\bm{X}[\alpha]).
	\end{align*}
\end{definition}

\begin{definition}[$V$-Statistic]\label{def:v_stat}
	Let $\bbX$ be a non-empty set, $m, n$ be positive integers and $h: \bbX^{m} \rightarrow \bbR$ be a function. Given a tuple $\bm{X} \in \bbX^n$, any $V$-statistic of order $m$ with kernel $h$ takes the following form
	\begin{align*}
		\vstatp{h}{\bm{X}} \coloneqq \sum_{\alpha \in [n]^m} h(\bm{X}[\alpha]).
	\end{align*}
\end{definition}

\begin{definition}[Tensor]
	\label{def:tensors}
	Given positive integers $m, n > 0$, an $m$-th order tensor (i.e. an $m$-dimensional array) $T$ of size $n$ is defined as a map from $[n]^m$ to $\reals$: for any $\alpha \in [n]^{m}$, $T (\alpha) \in \bbR$. Here we assume that each dimension of a tensor has the same size (i.e. number of entries in each dimension) unless explicitly mentioned otherwise.
\end{definition}

\begin{remark}
	\label{rem:oracle}
	In contrast to a general function, the evaluation of which may incur nontrivial computational cost, we ignore the cost of evaluating a tensor $T$ at a given index/entry $\alpha$, because computational complexity in this paper is measured in terms of floating-point operations, such as additions and multiplications of floating-point numbers, following the convention in the analysis of numerical algorithms \citep[\S 1.1.15]{golub2013matrix}. Accessing an entry $T (\alpha)$ of an explicitly represented tensor is treated as direct data access, and therefore the cost of such access is not included in our floating-point operation complexity analysis. This assumption can be further justified under the standard RAM model \citep[\S2.2]{cormen2022introduction}, wherein each entry access costs $O (1)$ time, and each floating-point operation involves only a constant number of such accesses. Hence ignoring the cost of accessing tensor entries will not change the asymptotic order of the runtime.
\end{remark}

In statistical applications, $\bbX$ and the tuple $\bm{X}$, respectively, correspond to the sample space and the observed data sample of size $n$. Unlike $U$-statistics, computing $V$-statistics amounts to summing over all $[n]^{m}$ without being restricted to permutations. As will be clear later in this paper, $V$-statistics can be represented as summations over higher-order tensors. This unique property renders computing $V$-statistics, but not $U$-statistics, amenable for tensor operations such as the Einstein summation. In particular, if the kernel function $h$ satisfies certain decomposable structure (see Definition~\ref{def:mul_decomp} later), the computational complexity of $V$-statistics can be notably reduced. Our algorithmic framework is thus to first turn a $U$-statistic into a linear combination of $V$-statistics and then compute each individual $V$-statistic via the Einstein summation.

\subsection{\Einsum{} Notation and \Einsum{} Operation}
\label{sec:pre_tensor_contraction}

As mentioned, our algorithm to be revealed will rely on the Einstein summation operation (abbreviated as the \Einsum{} operation) over tensors. In this section, we give a brief but (hopefully) pedagogical overview for readers unfamiliar with these related concepts. We also steer readers to \citet{blacher2024einsum} for more details.

The \Einsum{} operation to be defined serves as a fundamental computational primitive in our proposed algorithm to efficiently and exactly compute the $U$/$V$-statistics. The \Einsum{} operation generalizes the classical Einstein summation convention. The modern \Einsum{} operation, now implemented in popular Python packages such as \numpy{} and \torch{}, can be used for efficient tensor computation. To explain the \Einsum{} operation, we first introduce a concept called the \Einsum{} notation.

\begin{definition}[\Einsum{} Notation]\label{def:tensor_expression}
	Let $m$ and $K$ be two positive integers. An \Einsum{} notation with respect to [w.r.t.] $[m]$ is a pair
	\begin{align*}
		\calN = \big(\calA; B\big) \quad \text{with} \quad \calA \coloneqq (A_1, \cdots, A_K), A_k,\, B \in [m]^*, \; k \in [K],
	\end{align*}
	where we recall that $[m]^{\ast}$ denotes all tuples possibly generated by elements in $[m]$, and $B$ needs to satisfy the following constraint:
	\begin{itemize}
		\item $B[i] \neq B[j]$, for any $i, j \in [|B|]:i \neq j$,
		\item $\takesetp{B} \subseteq \bigcup_{k \in [K]} \takesetp{A_k}$.
	\end{itemize}
	$\calA \coloneqq (A_k)_{k = 1}^{K}$ is referred to as the input tuple list of $\calN$ and $B$ is referred as the output tuple. The index set of $\calN$ is the set of all indices appearing in the input tuple list, and is denoted by $\indexset(\calN)$. Without loss of generality, we always take $m$ to be the smallest $m$ such that $[m] = \indexset(\calN)$. If $B = \emptyset$, we then simply write $\calA$ as $\calN = (\calA; \emptyset)$ and call $\calA$ as the \Einsum{} notation.
\end{definition}

It should be noted that, in the above definition, the index set $[m]$ is simply chosen for notational convenience and, in principle, the indices need not be integers or numbers. We now define the \Einsum{} operation based on the \Einsum{} notation.

\begin{definition}[\Einsum{} Operation]
	\label{def:tensor_contraction}
	Let $\calN = (\calA; B)$ be an \Einsum{} notation w.r.t. the index set $I = [m]$. Suppose that a set of tensors $\calT \coloneqq (T_k)_{k=1}^K$, each of size $n$ and of order $(m_k)_{k = 1}^{K}$, can be represented by $\calN$, i.e. $m_k = |A_k|$ for any $k \in [K]$. An \Einsum{} operation, denoted as $\tensorcontraction(\calT, \calN)$, takes $\calT$ and $\calN$ as inputs and outputs another tensor $T^{\rm out}: [n]^{|B|} \rightarrow \reals$ as follows: for any $\alpha \in [n]^{|B|}$,
	\begin{align*}
		T^{\rm out} (\alpha) = \sum_{\substack{\alpha^\prime \in [n]^m \\ \alpha'[B] = \alpha}} \prod_{k=1}^K T_k(\alpha^\prime[A_k]).
	\end{align*}
	As a special case, if $B = \emptyset$, then $T^{\rm out} (\alpha)$ reduces to the following scalar:
	\begin{align*}
		T^{\rm out} = \sum_{\alpha^\prime \in [n]^m}\prod_{k=1}^K T_k(\alpha^\prime[A_k]).
	\end{align*}
\end{definition}

\begin{example}
	\label{eg:running}
	The running example that we consider in the main text will also be used in Section~\ref{sec:alg_u_v} at various stages. Given three second-order tensors (matrices) $\calT = (T_1, T_2, T_3)$, suppose we are interested in computing a scalar $T^{\rm out}$:
	\begin{align*}
		T^{\rm out} = \sum_{i = 1}^{n} \sum_{j = 1}^{n} \sum_{k = 1}^{n} \sum_{l = 1}^{n} T_1 (i, j) T_2 (j, k) T_3 (k, l).
	\end{align*}
	The corresponding \Einsum{} notation is $\calA = (A_1 = (1, 2), A_2 = (2, 3), A_3 = (3, 4))$ with $B = \emptyset$ and the index set $I = [4]$. $T^{\rm out}$ can be written in the \Einsum{} operation form $\tensorcontraction(\calT, \calA)$ because by definition:
	\begin{align*}
		\tensorcontraction(\calT, \calA) & \equiv \sum_{\alpha \in [n]^{4}} \prod_{k = 1}^{3} T_{k} (\alpha [A_k]) = \sum_{\alpha \in [n]^{4}} T_1 (\alpha [A_1]) T_2 (\alpha [A_2]) T_3 (\alpha [A_3]) \\
		                                 & = \sum_{\alpha \in [n]^{4}} T_1 (\alpha[1, 2]) T_2 (\alpha[2, 3]) T_3 (\alpha[3, 4])                                                                         \\
		                                 & = \sum_{\alpha \in [n]^{4}} T_1 (\alpha[1], \alpha[2]) T_2 (\alpha[2], \alpha[3]) T_3 (\alpha[3], \alpha[4]),
	\end{align*}
	so the claim holds by identifying $\alpha [1] = i, \alpha [2] = j, \alpha [3] = k, \alpha [4] = l$.
\end{example}

\begin{remark}
	\label{rem:empty}
	Example~\ref{eg:running} illustrates the case $B = \emptyset$, or equivalently the output of the \Einsum{} operation being a scalar; this is the case encountered in the main text, since we focus on scalar-valued $U$/$V$-statistics in this paper; see Section~\ref{sec:examples} for numerical experiments. The general case $B \neq \emptyset$, corresponding to general tensor-valued output, is nevertheless needed in the proof of Proposition~\ref{pro:complexity_einsum} in Appendix~\ref{app:complexity_einsum}. Two further examples with $B \neq \emptyset$ are provided in Appendix~\ref{app:examples}.
\end{remark}

\begin{remark}
	\label{rem:Python}
	The \Einsum{} operation has been implemented in highly optimized Python libraries such as \numpyeinsum{} and \torcheinsum{}. These efficient algorithmic subroutines, enabling efficient \Einsum{} operation in practice, serve as the building blocks of our algorithm. Notably, developed for training large overparameterized neural networks in tensor forms, \torcheinsum{} is also designed to leverage GPU computing to further enhance the scalability and speed of the \Einsum{} operation, as we will demonstrate using an example from statistics in Section~\ref{sec:example_hoif}.
\end{remark}

\begin{remark}
	\label{rem:tensor networks}
	The \Einsum{} operation is frequently employed in scientific computing problems. For example, the \Einsum{} operation serves as a basic algorithmic subroutine for computation over tensor (hyper)networks, a useful mathematical language adopted by the quantum computing and quantum information community \citep{markov2008simulating}. 
	Although we do not directly adopt the tensor network formulation, our strategy is quite similar in that we also establish a correspondence between the \Einsum{} operation and operations on graphs -- see Section~\ref{sec:alg_v} for more details.
\end{remark}


\subsection{Useful Concepts from Graph Theory}
\label{sec:pre_graph}

In this section, we recall three key concepts from graph theory that are central
to our complexity-theoretic characterization. We follow the terminology of
\citet{bodlaender2010treewidth} for \emph{vertex elimination} and \emph{treewidth},
which captures the dominant computational complexity of our problem.
We also recall the notion of \emph{quotient graph}, adopted from
\citet{sanders2012high}, which underlies the decomposition of a $U$-statistic
into a linear combination of $V$-statistics.
\begin{definition}[Partition]\label{def:partition}
	Let $S$ be a non-empty set, $\sPi$ a collection of non-empty subsets of $S$. $\sPi$ is a partition of $S$, if
	\begin{itemize}
		\item $A \cap B = \emptyset, \quad \forall A, B \in \sPi: A \neq B$,
		\item $\cup_{B \in \sPi} B = S$.
	\end{itemize}
	An element $B \in \pi$ is called a block of partition $\pi$.

	All partitions of set $S$ is denoted by $\partitionp{S}$. If $S = [m]$ for positive integer $m$, it is denoted by $\partitionp{m}$.
\end{definition}

\begin{example}[Partition]\label{exa:partition}
	We use a finite set of natural numbers as an example to illustrate the concept of a partition.
	Let $S = \{1,2,3,4,5,6\}$. Then $\pi = \{\{1,2\}, \{3,4\}, \{5,6\}\}$
	is a collection of subsets of $S$ and is a partition of $S$. An element $B = \{1,2\} \in \pi$, which is also a subset of $S$, is called a block of the partition $\pi$. All blocks of the partition $\pi$ are non-empty and pairwise disjoint, and their union is $S$.
\end{example}

\begin{definition}[Vertex Elimination]
	\label{def:vertex_elimination}
	Let $G = (V, E)$ be a graph and $v \in V$. We define the operation of eliminating $v$ from $V$ as first connecting $\neighborvg{v}{G}$ (the neighbor of $v$ w.r.t. $G$) by edges to form a clique and then removing the vertex $v$ and all edges incident to $v$. The resulting graph is denoted by $\eliminatevg{v}{G}$.
\end{definition}

\begin{definition}[Treewidth]
	\label{def:treewidth_elimination}
	Let $G = (V, E)$ be a graph with $|V| = n$. For $\sigma \in \perm{V}$, $(G_i^{\sigma})_{i=1}^n$ is a sequence of graphs recursively defined as $G_i^\sigma = \eliminatevg{\sigma[i]}{G_{i-1}^\sigma}$ with $G_0^\sigma = G$. Then the treewidth of the graph $G$ is defined as
	\begin{equation*}
		\twp{G} \coloneqq \min_{\sigma \in \perm{V}} \max_{i \in [n]} \degp{\sper[i]}{G_{i-1}^\sigma}.
	\end{equation*}
\end{definition}

\begin{figure}[htbp]
	\centering
	\resizebox{\textwidth}{!}{%
		\begin{tikzpicture}[
				every node/.style={font=\small},
				graphbox/.style={draw, rounded corners, align=center, inner sep=10pt},
				arrow/.style={-{Latex}, thick},
				point/.style={circle, draw, minimum size=7mm, inner sep=0pt, font=\small},
				elim/.style={point, fill=black!65, text=white}
			]
			\node[graphbox] (g0) {%
				\begin{tikzpicture}
					\node[elim]  (1) at (0,1) {$1$};
					\node[point] (2) at (1,1) {$2$};
					\node[point] (3) at (1,0) {$3$};
					\node[point] (4) at (0,0) {$4$};
					\draw (1) -- (2) -- (3) -- (4) -- (1);
				\end{tikzpicture}};
			\node[below=3pt of g0] {$G^\sigma_0 = G$};

			\node[graphbox, right=2.6cm of g0] (g1) {%
				\begin{tikzpicture}
					\node[elim]  (2) at (1,1) {$2$};
					\node[point] (3) at (1,0) {$3$};
					\node[point] (4) at (0,0) {$4$};
					\draw (2) -- (3) -- (4);
					\draw[dashed] (2) -- (4);
				\end{tikzpicture}};
			\node[below=3pt of g1] {$G^\sigma_1$};

			\node[graphbox, right=2.6cm of g1] (g2) {%
				\begin{tikzpicture}
					\node[elim]  (3) at (1,0) {$3$};
					\node[point] (4) at (0,0) {$4$};
					\draw (3) -- (4);
				\end{tikzpicture}};
			\node[below=3pt of g2] {$G^\sigma_2$};

			\node[graphbox, right=2.6cm of g2] (g3) {%
				\begin{tikzpicture}
					\node[elim] (4) at (0,0) {$4$};
				\end{tikzpicture}};
			\node[below=3pt of g3] {$G^\sigma_3$};

			\node[graphbox, right=2.6cm of g3] (g4) {$\varnothing$};
			\node[below=3pt of g4] {$G^\sigma_4$};

			\draw[arrow] (g0) -- node[above, align=center] {eliminate $1$\\[2pt] $\degp{1}{G^\sigma_0}=2$} (g1);
			\draw[arrow] (g1) -- node[above, align=center] {eliminate $2$\\[2pt] $\degp{2}{G^\sigma_1}=2$} (g2);
			\draw[arrow] (g2) -- node[above, align=center] {eliminate $3$\\[2pt] $\degp{3}{G^\sigma_2}=1$} (g3);
			\draw[arrow] (g3) -- node[above, align=center] {eliminate $4$\\[2pt] $\degp{4}{G^\sigma_3}=0$} (g4);
		\end{tikzpicture}%
	}
	\caption{Eliminating the $4$-cycle $G$ one vertex at a time in the order $\sigma = (1,2,3,4)$; see Example~\ref{eg:treewidth}.}
	\label{fig:treewidth_c4}
\end{figure}

To illustrate the above two important concepts, vertex elimination and treewidth, we consider the following example.
\begin{example}
	\label{eg:treewidth}
	Consider the $4$-cycle $G = (V, E)$ with $V = \{1,2,3,4\}$ and $E = \{\{1,2\}, \{2,3\}, \{3,4\}, \{4,1\}\}$. Since $\twp{G}$ is the minimum over all elimination orderings, we in principle need to check every possible ordering. We first work through one ordering $\sigma = (1,2,3,4)$, forming the graph sequence $(G^\sigma_i)_{i=0}^{4}$ of Definition~\ref{def:treewidth_elimination}, shown in Figure~\ref{fig:treewidth_c4}. Eliminating $1$ from $G^\sigma_0 = G$ has $\degp{1}{G^\sigma_0} = 2$, and by Definition~\ref{def:vertex_elimination} its neighbors $\neighborvg{1}{G^\sigma_0} = \{2,4\}$ are joined by a fill-in edge (drawn dashed), so $G^\sigma_1$ is the triangle on $\{2,3,4\}$. Eliminating $2$ next again has degree $2$, but its neighbors $\{3,4\}$ are already adjacent, so no fill-in edge is created and $G^\sigma_2$ is the single edge $\{3,4\}$. The remaining two eliminations have degrees $1$ and $0$. The largest degree along this ordering is thus $\max\{2,2,1,0\} = 2$. By the symmetry of the $4$-cycle, every other ordering yields the same maximum degree, so no ordering can do better, and $\twp{G} = 2$.
\end{example}

Here we state two simple facts of treewidth from \citet{bodlaender2011treewidth} that will be useful in what follows.

\begin{lemma}\label{lem:property_tw}
	Let $G = (V,E)$ be a graph. Then
	\begin{enumerate}
		\item $\twp{G} +1 \leq |V|$, and the equality holds if and only if $G$ is a complete graph,
		\item $\twp{G'} \leq \twp{G}$, if $G'$ is a subgraph of $G$.
	\end{enumerate}
\end{lemma}

\begin{definition}[Quotient Graph]
	\label{def:quotient_graph}
	Let $G = (V, E)$ be a graph. A quotient graph associated with $G$ w.r.t. a partition $\sPi$ of $V$ is defined as $G/\pi \coloneqq (\pi, E_{\pi})$, where
	\begin{equation*}
		E_\pi = \left\{\{B_i,B_j\} \subseteq \pi \mid \exists \, v_1 \in B_i, v_2 \in B_j, \text{ such that } \{v_1, v_2\} \in E\right\}.
	\end{equation*}
\end{definition}

\begin{figure}[htbp]
	\centering
	\begin{subfigure}{0.30\textwidth}
		\centering
		\begin{tikzpicture}[scale=0.6, baseline={(0,-0.2)}]
			\node[draw, circle, inner sep=1pt] (1) at (0,0) {1};
			\node[draw, circle, inner sep=1pt] (2) at (1,0) {2};
			\node[draw, circle, inner sep=1pt] (3) at (2,0) {3};
			\node[draw, circle, inner sep=1pt] (4) at (3,0) {4};
			\draw (1) -- (2) -- (3) -- (4);
		\end{tikzpicture}
		\caption{the original graph $G$}
		\label{fig:quotient_path_G}
	\end{subfigure}
	\hfill
	\begin{subfigure}{0.30\textwidth}
		\centering
		\begin{tikzpicture}[scale=0.6, baseline={(0,-0.2)}]
			\node[draw, circle, inner sep=1pt] (1) at (0,0) {1};
			\node[draw, circle, inner sep=1pt] (2) at (1,0) {2};
			\node[draw, circle, inner sep=1pt] (3) at (2,0) {3};
			\node[draw, circle, inner sep=1pt] (4) at (3,0) {4};
			\draw (1) -- (2) -- (3) -- (4);
		\end{tikzpicture}
		\caption{$G/\pi,\ \pi = \{\{1\},\{2\},\{3\},\{4\}\}$}
		\label{fig:quotient_path_b}
	\end{subfigure}
	\hfill
	\begin{subfigure}{0.30\textwidth}
		\centering
		\begin{tikzpicture}[scale=0.6, baseline={(0,-0.2)}]
			\node[draw, circle, inner sep=1.5pt] (a) at (0,0) {12};
			\node[draw, circle, inner sep=1pt]   (b) at (1.2,0) {3};
			\node[draw, circle, inner sep=1pt]   (c) at (2.4,0) {4};
			\draw (a) -- (b) -- (c);
		\end{tikzpicture}
		\caption{$G/\pi,\ \pi = \{\{1,2\},\{3\},\{4\}\}$}
		\label{fig:quotient_path_c}
	\end{subfigure}

	\vspace{0.5cm}

	\begin{subfigure}{0.30\textwidth}
		\centering
		\begin{tikzpicture}[scale=0.6, baseline={(0,-0.2)}]
			\node[draw, circle, inner sep=1.5pt] (a) at (0,0.9) {14};
			\node[draw, circle, inner sep=1pt]   (b) at (-0.8,-0.4) {2};
			\node[draw, circle, inner sep=1pt]   (c) at (0.8,-0.4) {3};
			\draw (a) -- (b) -- (c) -- (a);
		\end{tikzpicture}
		\caption{$G/\pi,\ \pi = \{\{1,4\},\{2\},\{3\}\}$}
		\label{fig:quotient_path_d}
	\end{subfigure}
	\hfill
	\begin{subfigure}{0.30\textwidth}
		\centering
		\begin{tikzpicture}[scale=0.6, baseline={(0,-0.2)}]
			\node[draw, circle, inner sep=1.5pt] (a) at (0,0) {12};
			\node[draw, circle, inner sep=1.5pt] (b) at (1.4,0) {34};
			\draw (a) -- (b);
		\end{tikzpicture}
		\caption{$G/\pi,\ \pi = \{\{1,2\},\{3,4\}\}$}
		\label{fig:quotient_path_e}
	\end{subfigure}
	\hfill
	\begin{subfigure}{0.30\textwidth}
		\centering
		\begin{tikzpicture}[scale=0.6, baseline={(0,-0.2)}]
			\node[draw, circle, inner sep=1.5pt] (a) at (0,0) {1234};
		\end{tikzpicture}
		\caption{$G/\pi,\ \pi = \{\{1,2,3,4\}\}$}
		\label{fig:quotient_path_f}
	\end{subfigure}
	\caption{The $4$-path graph $G$ and the five isomorphism types of its quotient graphs $G/\pi$; see Example~\ref{eg:quotient}.}
	\label{fig:quotient_path}
\end{figure}

\begin{example}[Quotient graphs of a $4$-path]
	\label{eg:quotient}
	We use a simple example to illustrate the quotient graph. Let $G$ be the $4$-path graph on $V = \{1,2,3,4\}$ with edges $\{1,2\}, \{2,3\}, \{3,4\}$ (Figure~\ref{fig:quotient_path_G}). For a partition $\pi$ of $V$, the quotient graph $G/\pi$ is obtained by merging the
	vertices of $G$ within each block of $\pi$ into a single vertex, and joining two
	resulting vertices whenever some edge of $G$ runs between the corresponding blocks.

	The $15$ partitions of $V$ yield five quotient graphs up to isomorphism, shown in Figures~\ref{fig:quotient_path_b}--\ref{fig:quotient_path_f}: the path $G$ itself, reproduced by the finest partition $\pi = \{\{1\},\{2\},\{3\},\{4\}\}$ (Figure~\ref{fig:quotient_path_b}); the three-vertex path, obtained from merging any pair of vertices other than $\{1,4\}$ (Figure~\ref{fig:quotient_path_c}); the triangle, obtained from only merging $\{1,4\}$ (Figure~\ref{fig:quotient_path_d}); the single edge, obtained from any two-block partition since $G$ is connected (Figure~\ref{fig:quotient_path_e}); and the single vertex, from the coarsest partition $\pi = \{\{1,2,3,4\}\}$ (Figure~\ref{fig:quotient_path_f}).
\end{example}

\section{A New Algorithm of Computing \pdfu{}-Statistics}
\label{sec:alg_u_v}

We are ready to present our new algorithm for exactly computing $U$-statistics, along with an analysis of the time complexity. Our algorithm has been implemented in an open source Python package \package{} (and its \href{https://github.com/cxy0714/U-Statistics-R}{R} interface). We begin by introducing the \emph{multiplicative-decomposable} (MD) structure of the $U$/$V$-statistic kernel in Section~\ref{sec:decomposition}, followed by the new algorithm and complexity analysis for $V$- and $U$-statistics in Section~\ref{sec:alg_v} and Section~\ref{sec:alg_u}.


\subsection{The Multiplicative-Decomposable Kernels}
\label{sec:decomposition}

The \emph{multiplicative-decomposable} (MD) structure is the key assumption we impose on the kernel of the $U$-statistic to be computed. This structure is essential for developing more efficient algorithms for exactly computing $U$/$V$-statistics than trivially summing over all $m$-out-of-$n$ permutations/$[n]^{m}$. Similar concepts are commonly used to develop numerical algorithms for solving partial differential equations.

\begin{definition}[Multiplicative Decomposition]\label{def:mul_decomp}
	Let $\samplespace$ be a non-empty set, $m$ a positive integer, and $h: \samplespace^m \to \bbR$. Suppose $\calA \coloneqq (A_k)_{k=1}^K$, with $A_k \in [m]^{*} \,\, \forall k \in [K]$ and $\bigcup_{k=1}^K \takesetp{A_k} = [m]$, and let $\calH \coloneqq (h_k)_{k=1}^K$ be a tuple of functions where each $h_k$ is defined on $\samplespace^{|A_k|}$. $\calH \times \calA$, consisting of ordered pairs of the form $(h_{k}, A_{k})$ for $k \in [K]$, is called a multiplicative decomposition of the function $h$, if for any $\bm{x} \in \samplespace^m$,
	\begin{align*}
		h(\bm{x}) = \prod_{k=1}^K h_k\big(\bm{x}[A_k]\big).
	\end{align*}
\end{definition}
In the sequel, we refer to the function tuple $\calH$ as the factors and the index tuple $\calA$ as the decomposition signature. For convenience, we say that a kernel $h$ allows a (non-trivial) MD structure $\calH \times \calA$ if $K > 1$. $U$-statistic kernels with MD structures are common in statistics. We will provide three examples in Section~\ref{sec:examples} and Appendix~\ref{app:examples}.

To utilize the \Einsum{} operation to compute $U$/$V$-statistics, we need to store pieces of the underlying statistics in tensor format. We term this step as \emph{tensorization}. Let $h$ be the kernel of some $U$/$V$-statistic permitting a MD structure $\calH \times \calA$, where $\calH \coloneqq (h_k)_{k=1}^K$ and $\calA \coloneqq (A_k)^K_{k=1}$. Given a sample $\bm{X} \in \samplespace^n$, the tensorization step constructs a tuple of tensors $\calT \coloneqq (T_k)_{k=1}^K$ as follows:
\begin{align*}
	T_k (\alpha_k) = h_k (\bm{X}[\alpha_k]), \quad \forall \, \alpha_k \in [n]^{|A_k|}, \, \forall \, k \in [K].
\end{align*}
Then following Definition~\ref{def:v_stat} and Definition~\ref{def:mul_decomp},
\begin{align}\label{eq:tensorization_v}
	\vstat(h, \bm{X}) = \sum_{\alpha \in [n]^m} h(\bm{X}[\alpha]) \equiv \sum_{\alpha \in [n]^m} \prod_{k=1}^K h_k(\bm{X}[\alpha[A_k]]) \equiv \sum_{\alpha \in [n]^m} \prod_{k=1}^K T_k(\alpha[A_k]),
\end{align}
which implies that the $V$-statistic $\vstat(h, \bm{X})$ can be equivalently expressed by using the tuple of tensors $\calT$ and the decomposition signature $\calA$. As will be seen in Section~\ref{sec:alg_u}, we can also use $\calT$ and $\calA$ to represent $U$-statistics. Hence $\vstat(h, \bm{X})$ and $\ustat(h, \bm{X})$ can be equivalently denoted as $\ustat(\calT, \calA)$ and $\vstat(\calT, \calA)$, respectively. We summarize the tensorization step in Algorithm~\ref{alg:tensorization} below, which constitutes the first step of our new algorithm for exactly computing a $V$-statistic.

\begin{algorithm}
	\caption{Tensorization of Kernel Function}
	\label{alg:tensorization}
	\begin{algorithmic}[1]
		\Require
		\Statex A kernel $h$ with decomposition components $\calH \coloneqq (h_k: \bbX^{m_k} \to \reals)_{k=1}^K$
		\Statex A sample $\bm{X} = (X_1, \dots, X_n) \in \samplespace^n$ from a non-empty set $\samplespace$
		\Ensure A tuple of tensors $\calT \coloneqq (T_k: [n]^{m_k} \to \reals)_{k=1}^K$, where $T_k$ is the tensorized $h_k$

		\Function{Tensorization}{$\calH, \bm{X}$}
		\For{$k \in [K]$}
		\Comment{Iterate over each factor of kernel $h$}
		\For{$\alpha \in [n]^{m_k}$}
		\Comment{Iterate over all $m_k$-dimensional indices}
		\State $T_k(\alpha) \gets h_k(\bm{X}[\alpha])$
		\Comment{Evaluate $h_k$ on $\bm{X} [\alpha]$}
		\EndFor
		\EndFor
		\State \Return $\calT \coloneqq (T_k)_{k=1}^K$
		\EndFunction
	\end{algorithmic}
\end{algorithm}

\begin{remark}\label{rem:complexity_tensorization}
	We now discuss the computational cost of Algorithm~\ref{alg:tensorization}. Throughout this paper, we assume that evaluating each factor $h_k$ of the kernel on a single tuple of samples has cost independent of the sample size $n$ and the order $m$ of the $U$-/$V$-statistic. Hence, evaluating $h_k$ over all possible combinations of the sample $\{X_i\}_{i=1}^n$ requires $O(n^{|A_k|})$ time, where $|A_k|$ is the number of arguments of $h_k$.
	Therefore, the time complexity of Algorithm~\ref{alg:tensorization} is
	\begin{equation*}
		O\Big(\sum_{k=1}^K n^{|A_k|}\Big)
		= O\Big(K n^{m_{\max}}\Big),
	\end{equation*}
	where $m_{\max} = \max_{k\in[K]} |A_k|$.
	In particular, when $K$ is fixed, this complexity is $O(n^{m_{\max}})$.
\end{remark}
The MD structure plays an essential role in reducing the complexity of exactly computing $V$-statistics. Without a (nontrivial) MD structure, the time complexity of computing an $m$-th order $V$-statistic is generally $\less (n^m)$ as we need to evaluate a function of $m$ arguments on all tuples of $[n]^m$. But if a particular MD structure holds, the time complexity may scale with $\less (n^{m_{\mathrm{max}}})$ where $m_{\mathrm{max}} \coloneqq \max_{k \in [K]} |A_k|$ since we just need to evaluate each  $h_k$ on all tuples of $[n]^{|A_k|}$ for $k \in [K]$. As $m_{\mathrm{max}}$ can be smaller than $m$, it is possible to improve the time complexity dramatically in practice. More details on the complexity analysis and intuition can be found in Section~\ref{sec:alg_v} below.

\setcounter{example}{1}
\begin{subexample}
	\label{eg:running-a}
	Recalling Example~\ref{eg:running}. The \Einsum{} operation can be seen as a $V$-statistic whose kernel function $h$ has a MD structure $\calH \times \calA$ with $\calH = (h_1, h_2, h_3)$, signature $\calA = ( (1,2),(2,3),(3,4))$ and $n$ sample points $\bm{X} \in \samplespace^{n}$. Tensorization of $\calH$ outputs tensors $\calT = ( T_1, T_2, T_3)$ such that:
	\begin{equation*}
		T_i (j,k) = h_i (\bm{X}_{j}, \bm{X}_{k}),  \, (j,k) \in [n]^2,  \, i = 1,2,3.
	\end{equation*}
	The $V$-statistic and $U$-statistic with kernel $h$ are respectively:
	\begin{align*}
		\vstat(h, \bm{X}) = \vstat (\calT, \calA) = \Einsum{}(\calT, \calA), \quad \ustat(h, \bm{X}) = \ustat(\calT, \calA).
	\end{align*}
\end{subexample}

\subsection{The Algorithm of Computing \pdfv{}-Statistics}
\label{sec:alg_v}

In this section, we show that $V$-statistics can be computed via the \Einsum{} operation. Let $\vstat(h, \bm{X})$ be an $m$-th order $V$-statistic with kernel $h$ and sample $\bm{X}$ of size $n$. Suppose that the kernel $h$ satisfies MD structure $(\calH \coloneqq (h_k)_{k=1}^K) \times (\calA \coloneqq (A_k)_{k=1}^K)$. As we have seen, $\vstat(h, \bm{X})$ has an equivalent tensorized form $\vstat(\calT, \calA)$, where $\calT$ is the output of Algorithm~\ref{alg:tensorization}.

By definition, $\calA$ constitutes an \Einsum{} notation and $\calT$ can be represented by $\calA$ (see Definition~\ref{def:tensor_contraction}), then we have
\begin{equation*}
	\tensorcontraction(\calT, \calA) = \sum_{\alpha \in [n]^m} \prod_{k=1}^K T_k(\alpha[A_k]).
\end{equation*}
Comparing the above equation with \eqref{eq:tensorization_v}, we immediately have
\begin{align*}
	\vstat(h, \bm{X}) = \tensorcontraction(\calT, \calA),
\end{align*}
i.e. computing a $V$-statistic is the same as performing an \Einsum{} operation. Algorithm~\ref{alg:v_stats} describes how $V$-statistics can be computed by the \Einsum{} operation.

\begin{algorithm}
	\caption{Tensorization Algorithm of $V$-Statistics}
	\label{alg:v_stats}
	\begin{algorithmic}[1]
		\Require
		\Statex A kernel function $h: \samplespace^m \to \reals$ permitting a MD structure $\calH \times \calA$
		\Statex A sample $\bm{X} \in \samplespace^n$ from a non-empty set $\samplespace$
		\Ensure $V$-statistic $\vstat(h, \bm{X})$

		\Function{Vstats}{$\calH, \calA; \bm{X}$}
		\State $\calT \gets$ \Call{Tensorization}{$\calH, \bm{X}$}
		\Comment{see Algorithm~\ref{alg:tensorization}}
		\State $v \gets \tensorcontraction(\calT, \calA)$\footnotemark
		\State \Return $v$
		\EndFunction
	\end{algorithmic}
\end{algorithm}
\footnotetext{\label{fn:einsum}In our implementation of \package{}, we directly call \numpyeinsum{} or \torcheinsum{}, with the summation ordering optimized by \opteinsum{}. }

The time complexity of computing a $V$-statistic can be characterized by that of the
\Einsum{} operation. The essential idea is to establish a correspondence between the
\Einsum{} operation and the vertex elimination operation
(Definition~\ref{def:vertex_elimination}) over a graph associated with the \Einsum{}
notation, so that the time complexity is in turn governed by the treewidth
(Definition~\ref{def:treewidth_elimination}) of that graph. \citet{markov2008simulating} established a treewidth-based complexity bound for tensor network contraction, for which the decomposition graph, as defined in Definition~\ref{def:graph_decomp}, is a line graph of some simple graph (a graph with no loops and with at most one edge between each pair of vertices). We extend their result to general \Einsum{} operations, for which the decomposition graph can be an arbitrary simple graph.

To facilitate presentation, we first introduce the notion of ``decomposition graph'' associated with a given \Einsum{} notation, which captures all the information needed to establish the time complexity of computing a particular $V$-statistic.

\begin{definition}[Decomposition Graph]\label{def:graph_decomp}
	Let $\calA$ be an \Einsum{} notation with no output. The decomposition graph associated with $\calA$ (or equivalently the decomposition signature $\calA$) is a graph $\graphform{\calA} = (V_{\calA}, E_{\calA})$ where
	\begin{itemize}
		\item $V_{\calA} = \indexset(\calA)$,
		\item $E_{\calA} = \big\{ \{i, j\} \subseteq V_{\calA} \mid \exists \, A \in \calA; \text{ such that } i, j \in \takesetp{A} \big\}$.
	\end{itemize}
\end{definition}

We use several concrete examples to illustrate the above definition.

\setcounter{example}{4}

\begin{example}[\Einsum{} notations and its decomposition graphs]
	Figure~\ref{fig:graphs_examples} below illustrates several examples of decomposition graphs associated with different decomposition signatures. The underlying idea is to treat each index in $\calA$ as a vertex and to add an edge between two vertices whenever the corresponding indices co-occur in the same tuple. Repeated tuples do not create multiple edges (see Figures~\ref{fig:graph_hoif_4_2} and~\ref{fig:graph_hoif_4_5}). Moreover, different decomposition signatures can yield the same decomposition graph (see Figures~\ref{fig:graph_motif_4} and~\ref{fig:graph_same_motif_4}).
	\begin{figure}[htbp]
		\centering
		\begin{subfigure}{0.22\textwidth}
			\centering
			\begin{tikzpicture}[scale=0.6, baseline={(0,-0.2)}]
				\node[draw, circle, inner sep=1pt] (1) at (0,0) {1};
				\node[draw, circle, inner sep=1pt] (2) at (1.2,0) {2};
				\node[draw, circle, inner sep=1pt] (3) at (2.4,0) {3};
				\node[draw, circle, inner sep=1pt] (4) at (3.6,0) {4};
				\draw (1) -- (2) -- (3) -- (4);
			\end{tikzpicture}
			\caption{$\calA = ((1,2),(2,3),(3,4))$}
			\label{fig:graph_hoif_4_1}
		\end{subfigure}
		\hfill
		\begin{subfigure}{0.22\textwidth}
			\centering
			\begin{tikzpicture}[scale=0.6, baseline={(0,-0.2)}]
				\node[draw, circle, inner sep=1pt] (1) at (0,0.6) {1};
				\node[draw, circle, inner sep=1pt] (2) at (-0.8,-0.6) {2};
				\node[draw, circle, inner sep=1pt] (3) at (0.8,-0.6) {3};
				\draw (1) -- (2);
				\draw (1) -- (3);
			\end{tikzpicture}
			\caption{$\calA = ((1,2), (2,1), (1,3))$}
			\label{fig:graph_hoif_4_2}
		\end{subfigure}
		\hfill
		\begin{subfigure}{0.22\textwidth}
			\centering
			\begin{tikzpicture}[scale=0.6, baseline={(0,-0.2)}]
				\node[draw, circle, inner sep=1pt] (1) at (90:1) {1};
				\node[draw, circle, inner sep=1pt] (2) at (210:1) {2};
				\node[draw, circle, inner sep=1pt] (3) at (330:1) {3};
				\draw (1) -- (2) -- (3) -- (1);
			\end{tikzpicture}
			\caption{$\calA = ((1,2), (2,3), (3,1))$}
			\label{fig:graph_hoif_4_3_motif_3}
		\end{subfigure}
		\hfill
		\begin{subfigure}{0.22\textwidth}
			\centering
			\begin{tikzpicture}[scale=0.6, baseline={(0,-0.2)}]
				\node[draw, circle, inner sep=1pt] (1) at (0,0.6) {1};
				\node[draw, circle, inner sep=1pt] (2) at (-0.8,-0.6) {2};
				\node[draw, circle, inner sep=1pt] (3) at (0.8,-0.6) {3};
				\draw (1) -- (2);
				\draw (2) -- (3);
			\end{tikzpicture}
			\caption{$\calA = ((1,2), (2,3), (3,2))$}
			\label{fig:graph_hoif_4_4}
		\end{subfigure}

		\vspace{0.5cm}

		\begin{subfigure}{0.22\textwidth}
			\centering
			\begin{tikzpicture}[scale=0.6, baseline={(0,-0.2)}]
				\node[draw, circle, inner sep=1pt] (1) at (0,0) {1};
				\node[draw, circle, inner sep=1pt] (2) at (1.2,0) {2};
				\draw (1) -- (2);
			\end{tikzpicture}
			\caption{
				$\calA = ((1,2), (2,1), (1,2))$}
			\label{fig:graph_hoif_4_5}
		\end{subfigure}
		\hfill
		\begin{subfigure}{0.22\textwidth}
			\centering
			\begin{tikzpicture}[scale=0.5, baseline={(0,-0.2)}]
				\node[draw, circle, inner sep=1pt] (1) at (0,1) {1};
				\node[draw, circle, inner sep=1pt] (2) at (-1,0) {2};
				\node[draw, circle, inner sep=1pt] (3) at (1,0) {3};
				\node[draw, circle, inner sep=1pt] (4) at (0,-1) {4};
				\draw (1) -- (2) -- (3) -- (4) -- (1);
				\draw (1) -- (3);
				\draw (2) -- (4);
			\end{tikzpicture}
			\caption{
				$\calA = ((1,2),(1,3),(1,4),$ \\
				$(2,3),(2,4),(3,4))$}
			\label{fig:graph_motif_4}
		\end{subfigure}
		\hfill
		\begin{subfigure}{0.22\textwidth}
			\centering
			\begin{tikzpicture}[scale=0.6, baseline={(0,-0.2)}]
				\node[draw, circle, inner sep=1pt] (1) at (0,0.8) {1};
				\node[draw, circle, inner sep=1pt] (2) at (-0.8,0) {2};
				\node[draw, circle, inner sep=1pt] (3) at (0.8,0) {3};
				\node[draw, circle, inner sep=1pt] (4) at (0,-0.8) {4};
				\draw (1) -- (2) ;
				\draw (3) -- (4) ;
			\end{tikzpicture}
			\caption{$\calA = ((1,2),(3,4))$}
			\label{fig:graph_dcov_1}
		\end{subfigure}
		\hfill
		\begin{subfigure}{0.22\textwidth}
			\centering
			\begin{tikzpicture}[scale=0.5, baseline={(0,-0.2)}]
				\node[draw, circle, inner sep=1pt] (1) at (0,1) {1};
				\node[draw, circle, inner sep=1pt] (2) at (-1,0) {2};
				\node[draw, circle, inner sep=1pt] (3) at (1,0) {3};
				\node[draw, circle, inner sep=1pt] (4) at (0,-1) {4};
				\draw (1) -- (2) -- (3) -- (4) -- (1);
				\draw (1) -- (3);
				\draw (2) -- (4);
			\end{tikzpicture}
			\caption{ $\calA = ((1,2,3),(1,3,4),$ \\
				$(1,2,4),(2,3,4))$}
			\label{fig:graph_same_motif_4}
		\end{subfigure}
		\caption{Examples of decomposition graphs of different decomposition signatures.}
		\label{fig:graphs_examples}
	\end{figure}
\end{example}

We are ready to present the following ``optimistic'' estimate of the time complexity of the \Einsum{} operation. The proof of the proposition below is deferred to Appendix~\ref{app:complexity_einsum}.

\begin{proposition}
	\label{pro:complexity_einsum}
	Let $\calA$ be an \Einsum{} notation with no output, and $G_\calA$ be the corresponding decomposition graph of $\calA$. Then there exists an algorithm with a particular \Einsum{} ordering such that for any set of tensors $\calT$ that can be represented by the \Einsum{} notation $\calA$, the time complexity of $\tensorcontraction(\calT, \calA)$ is $\less(|\calA|n^{\twp{G_\calA}+1})$.
\end{proposition}

\begin{remark}\label{rem:complexity_einsum_proof}
	In this remark, we briefly explain the algorithm mentioned in Proposition~\ref{pro:complexity_einsum}. The algorithm of Proposition~\ref{pro:complexity_einsum} sums over the indices one by one following some chosen ordering. To sum over an index, it takes all the tensors that involve that index, multiplies them together, and sums the products over that index, which yields a new intermediate tensor. Say this intermediate tensor has dimension $d$, the number of distinct indices appearing alongside the summed index; then the step takes at most $|\calA|\, n^{d+1}$ time (its $n^d$ entries, each a sum over $n$ values with at most $|\calA|$ multiplications per term). For example, for $\sum_{i_1,i_2,i_3,i_4} T_1(i_1,i_2)\, T_2(i_2,i_3)\, T_3(i_3,i_4)$ under the ordering $(2,1,3,4)$, summing over $i_2$ first forms the $n\times n$ tensor $\sum_{i_2} T_1(i_1,i_2)\, T_2(i_2,i_3) = U_1(i_1,i_3)$, which costs $2n^3$ time; the remaining steps produce only vectors, so under this ordering the time is $\less(n^3)$. The ordering $(1,2,3,4)$, by contrast, keeps every intermediate a vector, so its time is $\less(n^2)$, and indeed it is the optimal ordering; for an optimal ordering the time is $\less(n^{\twp{\graphform{\calA}}+1})$, corresponding to the treewidth of the decomposition graph.
\end{remark}

\begin{remark}
	\label{rem:ordering}
	The practical time complexity depends on the summation ordering actually used. We say that the time complexity given in the above proposition is only an optimistic estimate, because although it can be achieved with the optimal \Einsum{} ordering corresponding to the treewidth (Definition~\ref{def:treewidth_elimination}), this particular ordering is usually difficult to find in practice. \citet{arnborg1987complexity} have shown that finding the ordering that gives rise to the treewidth is an NP-complete problem. Therefore, although it is ideal to evaluate $\tensorcontraction(\calT, \calA)$ in Algorithm~\ref{alg:v_stats} using this optimal ordering, suboptimal orderings can also be used in practice. However, to the best of our knowledge, there is currently no general theoretical guarantee on how well the orderings returned by such heuristic strategies approximate the optimal ordering. In addition, the inefficiency caused by using a suboptimal ordering is problem-dependent and may vary across different MD structures.

	In the implementation of our algorithm, we directly use the Python library \opteinsum{} \citep{daniel2018opt} to find an ordering, because such a preexisting library has been optimized and stabilized over the years. \opteinsum{} provides strategies to find suboptimal orderings, such as greedy search, and strategies to find optimal orderings, such as dynamic programming and exhaustive search. The complexity achieved by the ordering found by the optimal strategy is apparently no larger than that achieved by the greedy strategy, but the time required to find the optimal ordering is nonetheless longer. In practice, we recommend that one run a program provided in our package \package{} for computing $U$-statistics (which is also available in \opteinsum{} for \Einsum{}) before the actual computation to estimate their time and space complexities, and then decide which ordering strategy should be used.

	As an illustration, we report in Table~\ref{tab:HOIF_Nl} (in Appendix~\ref{app:HOIF}) the time to search for the summation ordering under two strategies mentioned above (dynamic-programming and greedy strategy), the estimated number of floating-point operations at the two sample sizes $n\in\{1000,10000\}$, and the actual total compute time, for the HOIF estimator with $m\in\{4,5,\dots,10\}$. In this particular example, from Table~\ref{tab:HOIF_Nl}, one can conclude that when the order is below $9$, it does not make too much difference using either of the two strategies, but when the order is $10$, the exact summation ordering has an obvious advantage for both sample sizes ($1000$ and $10000$).
\end{remark}

Now we present the result of the computational complexity of computing $V$-statistics. According to Remark~\ref{rem:complexity_tensorization}, the complexity of the tensorization step is $O (K n^{m_{\max}})$. Moreover, under the decomposition graph defined in Definition~\ref{def:graph_decomp}, each factor $h_k$ corresponds to a clique $G_k$ of $|A_k|$ vertices in the decomposition graph $G_\calA$ associated with $\calA$. Following Lemma~\ref{lem:property_tw}, $|A_k| = \twp{G_k} +1 \leq \twp{G_\calA} +1$, as $G_k$ is a complete graph of $|A_k|$ vertices and is a subgraph of $G_\calA$. Therefore, we have $m_{\max{}} \leq \tw(G_\calA) +1$. According to Proposition~\ref{pro:complexity_einsum}, the complexity of Algorithm~\ref{alg:tensorization} would not dominate the complexity of computing $V$-statistics. Thus we immediately obtain the following optimistic estimate of the time complexity of exactly computing a $V$-statistic.

\begin{corollary}\label{pro:complexity_v}
	Let $\bbX$ be a non-empty set, $m, n$ be positive integers and $\bm{X} \in \samplespace^n$. Suppose that $h: \bbX^{m} \rightarrow \bbR$ be a kernel that permits a MD structure $\calH \times \calA$ and for each $h_k \in \calH = (h_k)_{k = 1}^{K}$, the time complexity of evaluating $h_k$ is of $O (1)$ and does not depend on $n$. Let $\calT$ be the tuple of tensors after applying the tensorization algorithm (Algorithm~\ref{alg:tensorization}) to $\calH$ and $G_\calA$ be the decomposition graph of $\calA$. Then there exists an algorithm such that its time complexity of computing $\vstat(h,\bm{X})$ exactly is~$\less(|\calA| n^{\twp{G_\calA}+1})$.
\end{corollary}

\subsection{The Algorithm of Computing \pdfu{}-Statistics}
\label{sec:alg_u}

As discussed in Section~\ref{sec:pre_u_v}, unlike $V$-statistics, $U$-statistics cannot be directly computed via the \Einsum{} operation. Therefore, we first decompose a given $U$-statistic into a linear combination of $V$-statistics. Each of the $V$-statistic terms can then be efficiently computed through \Einsum{}, and finally recombined to recover the original $U$-statistic.

To present how to decompose a $U$-statistic into $V$-statistics, several notions need to be introduced. Recall that $\partitionp{m}$ denotes all partitions of $[m]$.

\begin{definition}[V-Set]\label{def:v_set}
	Let $m,n$ be positive integers satisfying $m \leq n$ and $\sPi$ be a partition of the set $[m]$. A V-set of size $n$ w.r.t. partition $\sPi$ is defined as
	\begin{align*}
		\vsetp{n}{\sPi} \coloneqq \left\{ \bar{s}_{m} \in [n]^m \mid \text{where for any $i, j \in [m]$},i, j \in Q \text{ for some $Q \in \sPi$}  \Longrightarrow  s_i = s_j \right\}.
	\end{align*}
	If $\sPi$ happens to be the finest partition of $[m]$, i.e. $\sPi = \{\{1\}, \{2\}, \cdots, \{m\}\}$, the corresponding V-set is denoted by $\vsetp{n}{m}$ and in fact, $\vsetp{n}{m} = [n]^m$. The cardinality of the given partition $\sPi$ is referred to as the order of $\vsetp{n}{\sPi}$.
\end{definition}

\begin{example}\label{exa:v_set}
	Let $m = 4, n = 10,\pi = \{\{1,2\}, \{3\}, \{4\}\}$ and we take $\vset(10 , \pi) \subseteq [10]^4$ as an example to illustrate the concept of a V-set. That is, the first and second coordinates must be equal, since $1$ and $2$ lie in the same block of $\sPi$, while the third and fourth coordinates are unrestricted. All coordinates must contain in $[10]$.

	For example, $(1,1,2,3)$ is in $\vset(10, \pi)$ but  $(1,4,2,3)$ is not. Moreover, $(1,1,1,3)$ is also in $\vset(10, \pi)$ because a V-set does not restrict coordinates that not contain in the same block are different. But $(1,1,2,13)$ is not in $\vset(10, \pi)$ because the fourth coordinate is not in $[10]$, which is greater than $10$.
\end{example}

\begin{definition}[Restricted $V$-Statistic]\label{def:res_v_stat}
	Let $\bbX$ be a non-empty set, $m, n$ be positive integers, $\sPi$ be a partition of set $[m]$ and $h: \bbX^{m} \rightarrow \bbR$ be a function defined on $\bbX^m$. For any tuple $\bm{X} \in \bbX^*$, the $V$-statistic with kernel $h$ restricted by partition $\sPi$ takes the following form
	\begin{align*}
		\vstat[\sPi](h,\bm{X}) \coloneqq \sum_{\alpha \in \vsetp{n}{\sPi}} h(\bm{X}[\alpha]).
	\end{align*}
\end{definition}

Viewing some variables of the original kernel function of a $V$-statistic as one variable in the new kernel according to a partition $\pi$, a restricted $V$-statistic is a lower order $V$-statistic constructed based on the original $V$-statistic. We use the following Example~\ref{exa:res_v} to illustrate this concept.

\rev{
	\begin{example}\label{exa:res_v}
		We adopt the same setting as in Example~\ref{exa:v_set}. Let $h(x_1,x_2,x_3,x_4)$ be a function with $4$ variables which is consistent with $\vset(10, \pi) \subset [10]^4$. According to the partition $\pi = \{\{1,2\}, \{3\}, \{4\}\}$, variables $x_1$ and $x_2$ should be the same in the new kernel in the restricted $V$-statistic and there is no restriction for variables $x_3$ and $x_4$. Thus the new kernel
		\begin{equation*}
			\tilde{h}(x_1,x_2,x_3) = h(x_1,x_1,x_2,x_3),
		\end{equation*}
		and the corresponding restricted $V$-statistic is
		\begin{equation*}
			\vstat[\pi](h, \bm{X}) = \sum_{\alpha \in [10]^3} \tilde{h}(\bm{X}[\alpha]) \equiv \sum_{\alpha \in \vset(10, \pi)} h(\bm{X}[\alpha]),
		\end{equation*}
		given a sample of size $10$.
	\end{example}
}

Having introduced the above concepts, we state the following lemma that describes how to decompose a $U$-statistic into a linear combination of $V$-statistics.

\begin{lemma}[Decomposition of $U$-Statistic into $V$-Statistics]
	\label{lem:decomp_u_to_v}
	Let $\bbX$ be a non-empty set and $h: \bbX^{m} \rightarrow \bbR$ be a kernel function. Then for any tuple $\bm{X} \in \bbX^{n}$ with $n \geq m$,
	\begin{align*}
		\ustat(h,\bm{X}) = \sum_{\sPi \in \partitionp{m}}
		\mu_\sPi \vstat[\sPi](h, \bm{X}),
	\end{align*}
	where $\mu_\sPi = (-1)^{(m-|\sPi|)} \prod_{C \in \sPi} (|C| - 1)!$.
\end{lemma}

The proof of the above lemma can be found in Appendix~\ref{app:decomp_u_stats} and leverages \mobius{} inversion.

We now discuss how to extend the tensorization algorithm from $V$-statistics (Algorithm~\ref{alg:v_stats}) to $U$-statistics. Let $\ustatp{h}{\bm{X}}$ be an $m$-th order $U$-statistic with kernel function $h$ and sample $\bm{X} \in \samplespace^n$, where $h$ permits a MD structure $\calH \times \calA$. \rev{Seemingly, one has to apply Algorithm~\ref{alg:v_stats} to each resulting $V$-statistics in the decomposition given in Lemma~\ref{lem:decomp_u_to_v} and then repeatedly apply Algorithm~\ref{alg:tensorization}. But in fact, one only need to apply Algorithm~\ref{alg:tensorization} once since all the resulting $V$-statistics can be computed with the same tensor tuple $\calT = \textsc{Tensorizaion} (\calH, \bm{X})$.}

\rev{To compute the restricted $V$-statistic $\vstat[\pi] (h, \bm{X})$ utilizing the same tensor tuple $\calT$ for different partition $\pi$, we first introduce the corresponding induced \Einsum{} notation $\calA_\pi$ constructed from $\calA$ based on $\pi$. $\calA_{\pi}$ is obtained by replacing all indices that belong to the same block of the partition $\pi$ with the same index, and this construction is further explained in Example~\ref{exa:A_pi} below. Algorithm~\ref{alg:induced_expression} provides an algorithm to obtain $\calA_\pi$. }

\begin{example}\label{exa:A_pi}
	We take $\calA = ((1,2), (2,3), (3,4))$ and $\pi = \{\{1,3\}, \{2,4\}\}$, which is a partition of indices of $\calA$, as an example. As $1$ and $3$, $2$ and $4$ are in the same block respectively, the indices at the corresponding locations must be the same: informally, $3$ can be replaced by $1$ and $4$ can be replaced by $2$. Thus $\calA_\pi = ((1,2), (2,1), (1,2))$, consistent with Definition~\ref{def:tensor_expression}. 
\end{example}

We then have the following identity that relates the restricted $V$-statistic by the partition $\sPi$ and the \Einsum{} operation under the induced decomposition signature $\calA_{\pi}$:
\begin{equation}
	\label{restricted identity}
	\vstat[\sPi](h, \bm{X}) =  \tensorcontraction(\calT, \calA_\sPi).
\end{equation}
One can immediately realize that \eqref{restricted identity} holds because, by Definitions~\ref{def:v_set} and \ref{def:res_v_stat}, $\vstat[\pi](h, \bm{X})$ is a summation over all tuples of $[n]^m$, under the restriction that indices belonging to the same subset of the partition $\pi$ \rev{are required to take the same value}.

\begin{algorithm}
	\caption{Induced \Einsum{} Notation via Partitions}
	\label{alg:induced_expression}
	\begin{algorithmic}[1]
		\Require
		\Statex An \Einsum{} notation $\calA = (A_k)_{k=1}^K$ with index set $[m]$ and a partition $\sPi \in \partitionp{m}$ of size $J$
		\Ensure The \Einsum{} notation $\calA_\sPi$ induced by $\sPi$ on $\calA$

		\Function{InducedNotation}{$\calA; \sPi$}
		\State Let $(Q_j)_{j=1}^J$ be \rev{an arbitrary} permutation of $\sPi$\footnotemark
		\Comment{Label subsets by $j \in [J]$}
		\State Define $\psi \colon [m] \to [J]$ as the subset assignment map:
		\Function{$\psi$}{$i$}
		\State \Return $j$ if $i \in Q_j$
		\Comment{Map all indices in $Q_j$  to $j$}
		\EndFunction
		\State Initialize $\calA_\sPi \gets \calA$
		\For{$k \in [K]$}
		\For{$i \in [|A_k|]$}
		\State $\calA_\sPi[k][i] \gets \psi(A_k[i])$
		\Comment{Replace index with its subset label}
		\EndFor
		\EndFor
		\State \Return $\calA_\sPi$
		\EndFunction
	\end{algorithmic}
\end{algorithm}

According to Lemma~\ref{lem:decomp_u_to_v}, when computing an $m$-th order $U$-statistic, the number of $V$-statistics that we need to compute is the same as the total number of partitions of $[m]$, or equivalently the Bell number $\mathsf{B}_m$. Bell number $\mathsf{B}_m$ grows super-exponentially with $m$. As a result, when the order $m$ increases, the number of $V$-statistics that needs to be computed grows rapidly. Fortunately, we discover a sparsification trick to reduce the number of $V$-statistics in the decomposition when possible. To this end, we first introduce the following sparsified tensor list $\tilde{\calT}$, which gives the same $U$-statistic as the original tensor list $\calT$.

\begin{lemma}\label{lem:sparsification_trick}
	Let $h$ be a kernel function that permits a MD structure $(\calH = (h_k)_{k=1}^K) \times (\calA = (A_k)_{k=1}^K)$. A tensor tuple $\calT = (T_k)_{k = 1}^{K}$ is constructed from $\calH$ as in Algorithm~\ref{alg:tensorization}. Let $\tilde{\calT} \coloneqq (\tilde{T}_k)_{k=1}^K$ be constructed as follows:
	\begin{equation}
		\label{tilde_T}
		\begin{split}
			\tilde{T}_k (\alpha) =
			\begin{cases}
				h_k (\bm{X}[\alpha]) & \text{if } \forall \, i,j \in [|A_k|]: i \neq j, \, \alpha [i] \neq \alpha [j] \\
				0                    & \text{otherwise},
			\end{cases}
			\, \forall \, \alpha \in [n]^{|A_k|}, \, \forall \, k \in [K].
		\end{split}
	\end{equation}
	Then $\ustat(\calT, \calA) = \ustat(\tilde{\calT}, \calA)$.
\end{lemma}

\footnotetext{\rev{Here this permutation can be an arbitrary permutation of $\sPi$ as it is just used to construct an one-to-one mapping from $\pi$ to the index set $[|\pi|]$.}}

The proof of Lemma~\ref{lem:sparsification_trick} is straightforward as none of the indices $\alpha$ such that $\tilde{T}_{k} (\alpha) \neq T_{k} (\alpha)$ (when the second case of \eqref{tilde_T} holds) contributes to the value of the underlying $U$-statistic, because by definition $U$-statistic sums over distinct indices. Lemma~\ref{lem:sparsification_trick} suggests that some \Einsum{} operations can be skipped when computing $U$-statistics.

\begin{remark}\label{rem:sparsification_terms_revised}
	After applying the sparsification trick, the only remaining terms required to compute the $U$-statistic $\ustat(\widetilde{\calT}, \calA)$ are the restricted $V$-statistics $\vstat[\pi](\widetilde{\calT}, \calA)$ associated with partitions in the set
	\begin{equation}
		\label{eq:quotient_set}
		\partition_m^\calA \coloneqq \left\{ \pi \in \partitionp{m} \,\middle|\, \forall \, Q \in \pi,\ \forall A \in \calA,\ |Q \cap \takesetp{A}| < 2 \right\}.
	\end{equation}
	That is, only the \Einsum{} operations with notation $\calA_\pi$, for which any $Q \in \pi$ contains $\leq 1$ index from $\takesetp{A}$ for any $A \in \calA$, need to be evaluated when computing $\ustat(h, \calA)$.  To be more concrete, under a given partition $\pi$, if some indices are in the same component of the decomposition signature $\calA$, the corresponding \Einsum{} operation can be skipped because the result must be zero. Example~\ref{eg:sparsification} will illustrate this sparsification trick using a toy example. In Table~\ref{tab:HOIF} of Section~\ref{sec:example_hoif}, we use higher-order influence functions (HOIFs) mentioned in the beginning of this paper as a realistic example to also exhibit the effectiveness of this sparsification trick in reducing the number of $V$-statistics needed to compute.
\end{remark}

\setcounter{example}{1}
\begin{subexample}
	\label{eg:sparsification}
	We now illustrate the sparsification trick on the running example. Consider the $4$-th order $U$-statistic whose kernel $h = h_1(x_1,x_2)h_2(x_2,x_3)h_3(x_3,x_4)$ has decomposition signature $\calA = ((1,2),(2,3),(3,4))$, as in Example~\ref{eg:running-a}, with tensor tuple $\calT = (T_1,T_2,T_3)$. We also construct $\tilde{\calT} = \{\widetilde{T}_1, \widetilde{T}_2, \widetilde{T}_3\}$ as in Lemma~\ref{lem:sparsification_trick}, for example,
	\begin{equation}\label{eq:exp_sparsification_T1}
		\widetilde{T}_1 = \begin{cases}
			h_1(X_i,X_j), & \text{if } i \neq j, \\
			0,            & \text{if } i =j.
		\end{cases}
	\end{equation}
	Compared with $T_1$, the diagonals of $\widetilde{T}_1$ are zero and the other entries are the same, and similarly for $\widetilde{T}_2, \widetilde{T}_3$. Thus following Lemma~\ref{lem:sparsification_trick}, we have $\ustat(\tilde{\calT},\calA) = \ustat(\calT,\calA)$.

	Using $\tilde{\calT}$ rather than $\calT$ to compute the $U$-statistic  allows us to ignore the computation of any induced $V$-statistic whose corresponding partition is not in  \eqref{eq:quotient_set}. We take $\pi = \{\{1,2\}, \{3\}, \{4\}\}$ as an example to illustrate this point. There are block $B :=\{1,2\} \in \pi$ and $A := (1,2) \in \calA$ such that $|B \cap \takesetp{A}| = |\{1,2\}| = 2 \geq 2$. Then we have
	\begin{equation*}
		\vstat(\tilde{\calT},\calA_\pi) = \sum_{i_1,i_2,i_3 =1} \tilde{T}_1(i_1,i_1)\tilde{T}_2(i_1,i_2)\tilde{T}_3(i_2,i_3) \equiv 0,
	\end{equation*}
	because $\tilde{T}_1(i_1,i_1) \equiv 0$ by \eqref{eq:exp_sparsification_T1}. Thus we can ignore computing this V-statistics $\vstat(\tilde{\calT},\calA_\pi)$ in practice.

	Applying this trick to every partition of $[4]$, the $U$-statistic decomposes into
	\begin{align*}
		\ustat(\widetilde{ \calT }, \calA) & = \vstat[\{1\}\{2\}\{3\}\{4\}](\widetilde{ \calT }, \calA) - \cancel{\vstat[\{1,2\}\{3\}\{4\}](\widetilde{ \calT }, \calA)}
		- \vstat[\{1,3\}\{2\}\{4\}](\widetilde{ \calT }, \calA)                                                                                                          \\
		                                   & - \vstat[\{1,4\}\{2\}\{3\}](\widetilde{ \calT }, \calA) - \cancel{\vstat[\{2,3\}\{1\}\{4\}](\widetilde{ \calT }, \calA)}
		- \vstat[\{2,4\}\{1\}\{3\}](\widetilde{ \calT }, \calA)
		\\
		                                   &
		- \cancel{\vstat[\{3,4\}\{1\}\{2\}](\widetilde{ \calT }, \calA)}
		+ \cancel{\vstat[\{1,2\}\{3,4\}](\widetilde{ \calT }, \calA)}
		+ \vstat[\{1,3\}\{2,4\}](\widetilde{ \calT }, \calA)
		\\
		                                   &
		+ \cancel{\vstat[\{1,4\}\{2,3\}](\widetilde{ \calT }, \calA)}
		+ \cancel{2\vstat[\{1,2,3\}\{4\}](\widetilde{ \calT }, \calA)}
		+ \cancel{2\vstat[\{1,2,4\}\{3\}](\widetilde{ \calT }, \calA)}                                                                                                   \\
		                                   &
		+ \cancel{2\vstat[\{1,3,4\}\{2\}](\widetilde{ \calT }, \calA)}
		+ \cancel{2\vstat[\{2,3,4\}\{1\}](\widetilde{ \calT }, \calA)}
		- \cancel{\vstat[\{1,2,3,4\}](\widetilde{ \calT }, \calA)},
	\end{align*}
	where every canceled term is zero for exactly the reason above. Of the $15$ non-redundant $V$-statistics, only $5$ survive the sparsification trick, which are indexed by $\{1\}\{2\}\{3\}\{4\}$, $\{1,3\}\{2\}\{4\}$, $\{1,4\}\{2\}\{3\}$, $\{2,4\}\{1\}\{3\}$, and $\{1,3\}\{2,4\}$.
\end{subexample}

We are now ready to present the new algorithm for computing $U$-statistics in Algorithm~\ref{alg:u_stats} below.

\begin{algorithm}[H]
	\caption{Computation of $U$-Statistics via Tensorization}
	\label{alg:u_stats}
	\begin{algorithmic}[1]
		\Require
		\Statex A kernel function $h: \samplespace^m \to \reals$ satisfying MD $\calT \times \calA$
		\Statex A sample $\bm{X} = (X_1, \dots, X_n) \in \samplespace^n$
		\Ensure The $U$-statistic $\ustat(h, \bm{X})$

		\Function{Ustats}{$\calT, \calA; \bm{X}$}
		\State Let $\widetilde{\calT}$ be constructed as in  Lemma~\ref{lem:sparsification_trick} and initialize $u \gets 0$
		\For{$\sPi \in \partitionp{m}$}
		\If{$\forall Q \in \pi, \forall A \in \calA; |Q \cap \takesetp{A}| < 2$}
		\State $\mu_\sPi \gets (-1)^{m-|\sPi|} \prod_{C \in \sPi} (|C|-1)!$
		\Comment{See Lemma~\ref{lem:decomp_u_to_v}}
		\State $\calA_\sPi \gets$ \Call{InducedNotation}{$\calA; \sPi$}
		\Comment{See Algorithm~\ref{alg:induced_expression}}
		\State $v_{\sPi} \gets \tensorcontraction(\widetilde{\calT}, \calA_\sPi)$ \Comment{See Footnote~\ref{fn:einsum}}
		\State $u \gets u + \mu_\sPi v_{\sPi}$
		\EndIf
		\EndFor
		\State \Return $u$
		\EndFunction
	\end{algorithmic}
\end{algorithm}

Given a partition $\pi$ and the decomposition signature $\calA$, the decomposition graph corresponding to induced \Einsum{} notation $\calA_\pi$ can be characterized by the quotient graph $G_\calA / \pi$ (see Definition \ref{def:quotient_graph}) of the decomposition graph $G_\calA$ induced by $\pi$. Armed with the above observation, we can bound the time complexity of exactly computing $U$-statistics by combining Corollary~\ref{pro:complexity_v}, Lemma~\ref{lem:decomp_u_to_v}, and Lemma~\ref{lem:sparsification_trick}, leading to the following result.

\begin{corollary}
	\label{cor:U}
	Let $G_\calA$ be the decomposition graph of $\calA$.
	Under the same conditions of Corollary~\ref{pro:complexity_v}, there exists an algorithm such that the time complexity of computing the $U$-statistic $\ustatp{h}{\bm{X}}$ is $\less(|\calA| \sum_{l = 1}^{M} N_{l} n^{l + 1})$, where
	\begin{equation*}
		M \coloneqq \max\{\twp{G_\calA /\pi} \mid \pi \in \partitionp{m}^\calA \}, \quad N_{l} \coloneqq |\{\pi \in \Pi_{m}^{\calA} \mid \twp{G_{\calA} / \pi} = l\}|, \; \forall \, l
		\in [M].
	\end{equation*}
	and $\partitionp{m}^\calA$ is defined in \eqref{eq:quotient_set}.
\end{corollary}

\begin{remark}
	\label{rem:U-complexity}
	In the above result, quantities $N_{l}$ and $M$ are introduced. $N_{l}$ is the number of partitions in $\Pi_{m}^{\calA}$ such that the treewidth of the decomposition graph $G_{\calA}$ quotient by $\pi$ is $l$. Although abstract, $N_{l}$ and $M$ can in principle be computed numerically for concrete examples. For instance, in Section~\ref{sec:example_hoif}, we report the values of $M$ for HOIF estimators between orders $2$ to $12$; see the last column of Table~\ref{tab:HOIF} in Appendix~\ref{app:HOIF}. The counts $N_{l}$ can likewise be obtained with \opteinsum{}: the strategy that finds the optimal ordering (see Remark~\ref{rem:ordering}) returns the exact treewidth of each quotient graph, while the strategy that finds a suboptimal ordering returns an upper bound that is easier to compute. Both are reported with a numerical experiment for the HOIF estimator with $m\in\{4,5,\dots,10\}$ in Table~\ref{tab:HOIF_Nl} of Appendix~\ref{app:HOIF}, and a small example is given in Example~\ref{eg:U_complexity} below.
\end{remark}

\setcounter{example}{1}
\begin{subexample}
	\label{eg:U_complexity}
	In Example~\ref{eg:sparsification}, the $4$th-order $U$-statistic under study has decomposition signature $\calA = ((1,2),(2,3),(3,4))$, whose decomposition graph $\graphform{\calA}$ is precisely the $4$-path $G$ discussed in Example~\ref{eg:quotient}. After applying the sparsification trick, only the five restricted $V$-statistics indexed by
	\begin{equation*}
		\partitionp{4}^{\calA} = \big\{\, \{\{1\}\{2\}\{3\}\{4\}\},\; \{\{1,3\}\{2\}\{4\}\},\; \{\{2,4\}\{1\}\{3\}\},\; \{\{1,4\}\{2\}\{3\}\},\; \{\{1,3\}\{2,4\}\} \,\big\}
	\end{equation*}
	remain. Their decomposition graphs are exactly the quotient graphs $\graphform{\calA}/\pi$ enumerated in Example~\ref{eg:quotient}, so their treewidths can be read off directly:
	\begin{itemize}
		\item $\pi = \{1\}\{2\}\{3\}\{4\}$ reproduces the $4$-path itself (Figure~\ref{fig:quotient_path_b}), with $\twp{\graphform{\calA}/\pi} = 1$;
		\item $\pi = \{1,3\}\{2\}\{4\}$ and $\pi = \{2,4\}\{1\}\{3\}$ both give the three-vertex path (Figure~\ref{fig:quotient_path_c}), each with treewidth $1$;
		\item $\pi = \{1,3\}\{2,4\}$ gives the single edge (Figure~\ref{fig:quotient_path_e}), with treewidth $1$;
		\item $\pi = \{1,4\}\{2\}\{3\}$ gives the triangle (Figure~\ref{fig:quotient_path_d}), with treewidth $2$.
	\end{itemize}
	Apparently, four of the quotient graphs have treewidth $1$ and one has treewidth $2$, so in the notation of Corollary~\ref{cor:U} we have $M = 2$, $N_1 = 4$, and $N_2 = 1$. Since $|\calA| = 3$, Corollary~\ref{cor:U} yields the time complexity
	\begin{equation*}
		\less\Big(|\calA| \sum_{l=1}^{M} N_l\, n^{l+1}\Big) = \less\big(3\,(N_1 n^{2} + N_2 n^{3})\big) = \less\big(3 n^{3} + 12 n^{2}\big),
	\end{equation*}
\end{subexample}

\begin{remark}\label{rem:ours_vs_bruth}
	Fix the order $m$ and the kernel of a $U$-statistic, and let the sample size $n$ grow. In this regime, the complexity of our algorithm is governed by the largest treewidth among the quotient graphs $\graphform{\calA}/\pi$ induced by the decomposition graph with the MD structure. More precisely, the complexity is $\Theta(n^{M+1})$, where
	\begin{equation*}
		M = \max\{\twp{\graphform{\calA}/\pi}:\pi\in\partitionp{m}^{\calA}\}.
	\end{equation*}
	By contrast, the direct brute-force evaluation of this $m$th-order $U$-statistic sums over all $m$-tuples of distinct sample indices, incurring a cost of order $\Theta(n^m)$. Since applying a quotient operation does not increase the number of vertices, Lemma~\ref{lem:property_tw}~(1) implies that $M+1\le m$, with equality only when some quotient graph is the complete graph $K_m$. Consequently, unless the decomposition graph $\graphform{\calA}$ is itself $K_m$, our algorithm achieves a lower computational complexity than the brute-force approach.
\end{remark}

\begin{remark}
	\label{rem:tw}
	In this paper, we have used the treewidth of the decomposition graph (or its quotient as in Corollary~\ref{cor:U}) to characterize the optimistic time-complexity of exactly computing $U$/$V$-statistics, and the \Einsum{} operation. Although treewidth is a widely used concept for establishing the time-complexity of algorithms on graphs, it is generally nontrivial to compute or derive a tight bound. To our knowledge, Theorem~4.7 in \cite{kneis2009bound} proved a bound for treewidth linear in edge size. Whether this bound can be further improved remains an open problem (see Appendix~\ref{app:open}).
\end{remark}

\begin{remark}
	\label{rem:space_complexity}
	We now discuss the space complexity of our algorithm. Throughout, we measure memory usage by the number of floating-point entries stored in memory. The space complexity of computing the $U$-statistic is:
	\begin{equation*}
		O\big(\max\{K\, n^{m_{\max}},\, n^{M}\}\big),
		\qquad
		m_{\max} = \max_{k \in [K]} |A_k|, \quad
		M = \max\big\{\twp{\graphform{\calA}/\pi} \mid \pi\in\partitionp{m}^{\calA}\big\}.
	\end{equation*}
	This bound assumes optimal summation orderings and may underestimate the practical peak memory, which depends on the ordering actually chosen. A more detailed discussion is deferred to Appendix~\ref{app:space_complexity}.  We also report the measured peak memory under different summation-ordering strategies at sample sizes $n\in\{1000,10000\}$, for the HOIF estimator with $m\in\{4,5,\dots,10\}$ in Appendix~\ref{app:HOIF} (Table~\ref{tab:HOIF_Nl}).
\end{remark}

Before moving on to applications of our algorithm in the next section, we summarize our proposed algorithm below. Given a $U$-statistic kernel $h$ and a sample $\bm{X}$:
\begin{enumerate}[label = (\arabic*)]
	\item tensorize $h$ into $\calT$ (Algorithm~\ref{alg:tensorization}) and construct $\tilde{\calT}$ as in Lemma~\ref{lem:sparsification_trick};

	\item for each possible partition $\pi$ in $\Pi_{m}$, obtain the corresponding \Einsum{} notation $\calA_{\pi}$ (Algorithm~\ref{alg:induced_expression}) and evaluate the corresponding $V$-statistic using the \Einsum{} operation;

	\item add all evaluated $V$-statistics together sequentially as in Algorithm~\ref{alg:u_stats}.
\end{enumerate}

\begin{remark}
	\label{rem:he}
	As alluded to in the Introduction, \citet{he2021asymptotically} proposed an algorithm to compute a specific class of $U$-statistics useful for high-dimensional testing problems. In Appendix~\ref{app:examples_complexity}, we show that our algorithm subsumes theirs as a special case; and the framework for analyzing the runtime complexity developed here can also be used to characterize the runtime complexity of computing the class of $U$-statistics therein.
\end{remark}

\section{Examples and Numerical Results}
\label{sec:examples}

In this section, we evaluate the performance of our new algorithm for computing $U$/$V$-statistics in three concrete statistical applications\footnote{The replication code can be found in \href{https://github.com/cxy0714/U-Statistics-Experiments}{this link}.}.
We have made our algorithm available for public use, which is available to download from the GitHub repository (by clicking \package{} or its \href{https://github.com/cxy0714/U-Statistics-R}{R} interface) or the Python package index \href{https://pypi.org/}{\texttt{PyPi}}.  For the first application (Section~\ref{sec:example_hoif}) on Higher-Order Influence Functions (HOIFs) \citep{robins2008higher}, we also provide a standalone R package \href{https://cran.r-project.org/web/packages/HOIF/index.html}{\texttt{HOIF}} on CRAN, based on \package{}, given the growing interest in using HOIFs in practice in the causal inference community \citep{wager2024causal, zheng2025perturbed, cinelli2025challenges}.

\subsection{Application 1: Higher-Order Influence Functions}
\label{sec:example_hoif}

As mentioned in the Introduction, HOIFs \citep{robins2008higher} can help construct rate-optimal estimators for important parameters such as the treatment-specific mean under ignorability under certain complexity-reducing assumptions on the data generating mechanism. However, estimators constructed via HOIFs are high-order $U$-statistics, which are time-consuming to compute in practice. Fortunately, the HOIF estimator of ATE is a sum of $U$-statistics, each of which has kernel satisfying the MD structure as in Definition~\ref{def:mul_decomp}.
In the setting of an unconfounded observational study, we often have access to a sample tuple $\bm{X} \equiv (X_{i} \equiv (Z_{i}, A_{i}, Y_{i}))_{i = 1}^{n}$ of size $n$, where $Z \in \bbR^{d}$ denotes the baseline covariates, $A \in \{0, 1\}$ denotes a binary treatment variable and $Y$ denotes the outcome of interest. According to \citet{robins2008higher}, with a $k$-dimensional dictionary $\phi: \bbR^{d} \rightarrow \bbR^{k}$, an $m$-th order influence function estimator of the treatment-specific mean reads as
\begin{align*}
	\hat{\tau}_{m} = \sum_{j = 2}^{m} (-1)^{j} \ustat (h^{(j)}, \bm{X}),
\end{align*}
where for every $j = 2, \cdots, m$, given $\bar{i}_{j} \in [n]^j$,
\begin{align*}
	h^{(j)} (\bm{X}[\bar{i}_{j}]) = A_{i_{1}} \phi (Z_{i_{1}})^{\top} \left\{ \prod_{s = 3}^{j} \left( A_{i_{s}} \phi (Z_{i_{s}}) \phi (Z_{i_{s}})^{\top} - \mathbb{I} \right) \right\} \phi (Z_{i_{2}}) Y_{i_{2}}
\end{align*}
and we assume for simplicity that the population mean of $A \phi (Z) \phi (Z)^{\top}$ is $\mathbb{I}$.

$\hat{\tau}_{m}$ can be rewritten as a sum of $j$-th order $U$-statistics with kernel $h^{(j)}$, each of which further admits an expansion into a summation of $U$-statistics whose kernels satisfy the MD structure, for $j = 2, \cdots, m$.
To see this, without loss of generality we consider $j = m$ and rewrite $\ustat (h^{(m)}, \bm{X})$ as
\begin{equation*}
	\ustatp{h^{(m)}}{\bm{X}} = \sum_{j=0}^{m-2} \binom{m-2}{j} \frac{1}{\binom{n}{j} j!}  \ustatp{h_j^{(m)}}{\bm{X}}.
\end{equation*}

where for any $\bar{i}_{j} \in [n]^j$,
\begin{align}
	\label{eq:hoif_kernel}
	h_j^{(m)} (\bm{X}[\bar{i}_{j}])
	 & \coloneqq A_{i_1} \phi (Z_{i_1})^{\top} \left\{ \prod_{s = 2}^{j-1} A_{i_s} \phi (Z_{i_s}) A_{i_s} \phi (Z_{i_s})^{\top} \right\} \phi (Z_{i_j}) Y_{i_j} \notag \\
	 & =
	\begin{cases}
		f(X_{i_1}, X_{i_2}),                                                                                      & j = 2,         \\[1ex]
		g(X_{i_1}, X_{i_2}) \, g(X_{i_2}, X_{i_3}) \cdots g(X_{i_{j-2}}, X_{i_{j-1}}) \, f(X_{i_{j-1}}, X_{i_j}), & 3 \le j \le m,
	\end{cases}
\end{align}
here the two bivariate factors are defined as $g(X_1, X_2) \coloneqq A_1 \phi (Z_1)^{\top} A_2 \phi (Z_2), f(X_1, X_2) \coloneqq A_1 \phi (Z_1)^{\top} \phi (Z_2) Y_2$.

Then $h_j^{(m)}$ admits an MD structure in the sense of Definition~\ref{def:mul_decomp}, of order $j$ with $K = j-1$ factors: the factor tuple is $\calH = (h_k)_{k=1}^{K}$ with $h_k = g$ for $k = 1, \dots, j-2$ and $h_{j-1} = f$, each defined on $\samplespace^{2}$ and output a scalar, and the decomposition signature is $\calA = (A_k)_{k=1}^{K}$ with $A_k = (k, k+1)$, i.e. $((1,2),(2,3),\cdots,(j-1,j))$, which is denoted by $\calA_{\mathsf{HOIF},j}$, so the decomposition graph of $\calA_{\mathsf{HOIF},j}$ is a chain graph with $e = j - 1$ edges, the computational complexity up to order $m=12$ is analyzed in Observation~\ref{obs:HOIF} in Appendix~\ref{app:HOIF}.  Using this example, we also demonstrate the effectiveness of the sparsification trick (see Lemma~\ref{lem:sparsification_trick}) in Table~\ref{tab:HOIF} of Appendix~\ref{app:HOIF}, and further report, for $m\in\{4,5,\dots,10\}$ at sample sizes $n\in\{1000,10000\}$, the time to find the summation ordering under different ordering strategies, the resulting counts $N_l$ of Corollary~\ref{cor:U}, and the peak memory, in Table~\ref{tab:HOIF_Nl} of Appendix~\ref{app:HOIF}.

In Table~\ref{tab:HOIF_runtime}, we report the average runtime over 10 runs when computing the $U$-statistic $\ustatp{h_j^{(m)}}{\bm{X}}$ where kernel function $h_j^{(m)}$ is as in \eqref{eq:hoif_kernel} with $m=7$ and $j \in \{2, 3, \cdots, 7\}$ and sample size $n \in \{1000, 2000, 4000, 8000, 10000\}$, respectively using CPU parallel computation and using a single GPU. We would like to highlight that even when $m = 7$, $n = 10000$, it only takes about 5 minutes using our new package \package{}, with CPU parallel computation. If one instead has access to a single GPU, the runtime even reduces to about 8 seconds, a substantial speedup! We remark in passing that all the runtime reported in Table~\ref{tab:HOIF_runtime} excludes the time of kernel evaluation and tensorization.
\begin{table}[ht]
	\centering
	\caption{Average runtime (in seconds) using CPU (Intel Xeon ICX Platinum 8358 CPUs [2.6GHz, 64 total cores] with memory of 512 GB) with parallel computation (upper cell) and using a single GPU (NVIDIA RTX 5090, 32GB) with parallel computation (lower cell), evaluated across varying sample sizes and orders of HOIFs.}
	\label{tab:HOIF_runtime}
	\begin{tabular}{cccccc}
		\toprule
		\textbf{Order} & \textbf{1000}                   & \textbf{2000} & \textbf{4000} & \textbf{8000} & \textbf{10000} \\
		\midrule
		2              & \makecell{0.00217 \\ 0.00070}
		               & \makecell{0.00515 \\ 0.00142}
		               & \makecell{0.01965 \\ 0.00492}
		               & \makecell{0.03466 \\ 0.02779}
		               & \makecell{0.04811 \\ 0.04398}                                                                    \\
		\hline
		3              & \makecell{0.00388 \\ 0.00124}
		               & \makecell{0.01548 \\ 0.00274}
		               & \makecell{0.06968 \\ 0.01337}
		               & \makecell{0.27130 \\ 0.06413}
		               & \makecell{0.43110 \\ 0.08977}                                                                    \\
		\hline
		4              & \makecell{0.01963 \\ 0.00242}
		               & \makecell{0.07295 \\ 0.00475}
		               & \makecell{0.32076 \\ 0.02375}
		               & \makecell{1.81642 \\ 0.10690}
		               & \makecell{2.33967 \\ 0.17264}                                                                    \\
		\hline
		5              & \makecell{0.09356 \\ 0.00567}
		               & \makecell{0.36376 \\ 0.00922}
		               & \makecell{1.91751 \\ 0.04619}
		               & \makecell{9.80295 \\ 0.23752}
		               & \makecell{16.92521 \\ 0.41339}                                                                   \\
		\hline
		6              & \makecell{0.45275 \\ 0.01863}
		               & \makecell{1.52683 \\ 0.02545}
		               & \makecell{7.76946 \\ 0.13763}
		               & \makecell{41.78568 \\ 0.89001}
		               & \makecell{62.54028 \\ 1.66450}                                                                   \\
		\hline
		7              & \makecell{1.73287 \\ 0.07340}
		               & \makecell{6.80758 \\ 0.10322}
		               & \makecell{36.41280 \\ 0.64650}
		               & \makecell{197.05856 \\ 4.49332}
		               & \makecell{288.89672 \\ 8.58869}                                                                  \\
		\bottomrule
	\end{tabular}
\end{table}

\subsection{Application 2: Motif Counts in Networks}
\label{sec:example_motif}

Motif (or subgraph) counts refer to the number of occurrences of a particular subgraph type $\sfr$ within a larger graph $G$, denoted by $\motifrg{\sfr}{G}$. These counts, also known as network moments, are crucial for understanding the structural properties of random graph models often used to model real-world networks. 
Interestingly, motif counts can be equivalently represented as $U$-statistics. Leveraging this observation, we show that various motif counts can be computed using a single \Einsum{} operation. While this approach may not capture fine-grained structural information such as sparsity of the adjacency matrix, it is highly efficient in the context of dense graphs.

Motif counts in a simple graph $G$ with $n$ vertices and the adjacency matrix $A \in \{0,1\}^{n \times n}$ can be represented as $U$-statistics of the following form:
\begin{equation*}
	\motifrg{\sfr}{G} = \frac{1}{|\mathrm{Aut}(\sfr)|} \sum_{ \bar{i}_{m} \in \perm{n,m}} h_{\sfr} (\bar{i}_{m}; A),
\end{equation*}
where the kernel function $h_{\sfr}$ encodes the motif structure, and $|\mathrm{Aut} (\sfr)|$ denotes the number of automorphisms of motif $\sfr$. 
Due to space limitation, we present the concrete forms of $U$-statistics and the theory analysis for all 3-/4-vertex motifs in Appendix~\ref{app:motifs}.

We compare our package \package{} with several state-of-the-art systems for motif counting, including the widely used graph analysis library \igraph{}~\citep{csardi2006igraph}, the parallel motif counting system \peregrine{}~\citep{jamshidi2020peregrine}, and NVIDIA's GPU-based library \cugraph{}~\citep{fender2022rapids}. These systems represent different computational paradigms.
\igraph{} relies on highly optimized enumeration algorithms derived from \citet{wernicke2006fanmod} and runs in a single-threaded setting.
\peregrine{} implements parallel, motif-specific optimizations for fast CPU-based counting.
\cugraph{} provides specialized GPU kernels, but currently supports only a limited set of motifs (e.g., triangles). In contrast, \package{} adopts a unified formulation based on $U$-statistics and \Einsum{}, without requiring implementations specific to motif counting.

We conduct three sets of experiments to evaluate our method: counting (i) all $3$-vertex motifs on CPUs, (ii) all $4$-vertex motifs on CPUs, and (iii) all triangles on GPUs.
We generate random graphs using the classical Erd\H{o}s--R\'enyi model $G(n, p)$, where each of the $\binom{n}{2}$ possible edges is included independently with probability $p$. The experiments for counting all $3$-vertex and $4$-vertex motifs are conducted on a server equipped with Intel Xeon Scalable Cascade Lake 6248 CPUs (2.5\,GHz, 40 cores) and 192\,GB of memory, while the triangle counting experiment with \cugraph{} is performed on a single NVIDIA RTX 4090 GPU (24\,GB).
The runtime comparisons for counting 3-vertex motifs are shown in Table~\ref{tab:motif_counts_3}, with results for 4-vertex motifs reported in Table~\ref{tab:motif_counts_4} in Appendix~\ref{app:motifs}, and GPU-based triangle counting results presented in Table~\ref{tab:triangle_counting_gpu} also in Appendix~\ref{app:motifs}.

The results reveal distinct performance regimes across methods.
On sparse graphs (small $p$), \peregrine{} and  \igraph{} outperform \package{} due to their ability to exploit sparsity through enumeration-based strategies.
However, as the graph becomes denser, their runtime increases significantly, with \igraph{} failing to complete within the time limit even at moderate densities. In contrast, \package{} maintains nearly constant runtime across all values of $p$, resulting in substantial performance advantages in dense regimes.
The results for 4-vertex motifs in Table~\ref{tab:motif_counts_4} exhibit similar trends.
On GPUs, \cugraph{} achieves faster runtimes on sparse graphs due to specialized kernels, but its performance degrades with increasing density and eventually runs out of memory, whereas \package{} remains stable without requiring motif-specific implementations.

\begin{remark}\label{rem:motif_count_sparse}
	To compute the motif count, our algorithm directly applies \Einsum{} operations to the adjacency matrix of the graph. Since the current implementations of \Einsum{} in popular libraries such as \numpy{} and \torch{} do not provide efficient optimization for sparse matrices or tensors, our method cannot automatically benefit from the sparsity of the input graph. Therefore, sparse graphs do not necessarily lead to better performance using our approach. In contrast, existing enumeration-based algorithms implemented in \igraph{}, \peregrine{}, and \cugraph{} can explicitly exploit sparse graph structures, making them more efficient in sparse regimes. Future developments of efficient sparse \Einsum{} implementations may further improve the performance of our approach on sparse graphs. We leave this direction to future work.
\end{remark}

\begin{table}[ht]
	\centering
	\caption{
		Runtime comparison of exact all $3$-vertex motif counts using \package{}, \peregrine{} and \igraph{} on Erd\H{o}s--R\'enyi graphs $G(n, p)$ with $n = 20000$.
		Experiments were run on Intel Xeon Scalable Cascade Lake 6248 CPUs (2.5GHz, 40 total cores) with memory of 192 GB.
		``OOT'' denotes instances that exceeded the time limit of 1800 seconds.
	}
	\label{tab:motif_counts_3}
	\begin{tabular}{cccc}
		\toprule
		\makecell{\textbf{Edge} \\ \textbf{Prob.} $p$} & \makecell{\textbf{\package{}} \\ \textbf{Time (s)}} & \makecell{\textbf{\peregrine{}} \\ \textbf{Time (s)}} & \makecell{\textbf{\igraph{}} \\ \textbf{Time (s)}} \\
		\midrule
		0.0005                                         & 10.3478                                             & 0.0703                                                & 0.1915                                             \\
		0.0010                                         & 10.4452                                             & 0.0740                                                & 0.7259                                             \\
		0.0050                                         & 10.6196                                             & 0.0780                                                & 40.7373                                            \\
		0.0100                                         & 10.3616                                             & 0.1190                                                & 289.5082                                           \\
		0.0500                                         & 9.8357                                              & 1.1782                                                & OOT                                                \\
		0.2000                                         & 10.6010                                             & 17.1063                                               & OOT                                                \\
		0.4000                                         & 10.8445                                             & 67.6727                                               & OOT                                                \\
		0.8000                                         & 8.8623                                              & 158.7377                                              & OOT                                                \\
		\bottomrule
	\end{tabular}
\end{table}

\subsection{Application 3: Distance Covariance Measures}
\label{sec:dependence}

In the last example, we revisit the problem considered in Section~5.1 of \citet{shao2025u}, which concerns the problem of computing the squared distance covariance $\mathsf{dCov}^2$, a recently proposed dependence measure between two random vectors \citep{szekely2007measuring, yao2018testing}. In particular, let $O_i \coloneqq (X_i, Y_i), i \in [n]$ be the observed sample. $\mathsf{dCov}^2 (X, Y)$ can be written as a $4$-th order $U$-statistic: Denoting $\|\cdot\|_2$ as the vector $\ell_2$-norm, then
\begin{align*}
	\mathsf{dCov}^2(X, Y) = \frac{(n - 4)!}{n!} \sum_{\bar{i}_{4} \in \perm{[n],4}} & \Bigg( h_1 (O_{i_1}, O_{i_2}) h_2 (O_{i_3}, O_{i_4}) + h_1 (O_{i_1}, O_{i_2}) h_2 (O_{i_1}, O_{i_2})    \\
	                                                                                & - h_1 (O_{i_1}, O_{i_2}) h_2 (O_{i_1}, O_{i_3}) - h_1 (O_{i_1}, O_{i_2}) h_2 (O_{i_2}, O_{i_4}) \Bigg),
\end{align*}
where $h_1 (O_i, O_j) \coloneqq \|X_i - X_j\|_2$ and $h_2 (O_i, O_j) \coloneqq \|Y_i - Y_j\|_2$ for $i, j \in [n]$.

Quoting from Table~5 of \citet{shao2025u}, it takes ``8099.73 seconds'' to compute $\mathsf{dCov}^2$ among 11 random vectors, together with its confidence interval, when $n = 138$. Each vector corresponds to a different stock sector—such as Health Care or Industrials—and its components represent the monthly log-returns of individual stocks within that sector, using historical S\&P 500 data. \citet{shao2025u} resort to computing a randomized incomplete $U$-statistic instead. In the same Table~5, \citet{shao2025u} demonstrated that there is indeed a significant speedup when using randomized incomplete $U$-statistics. 

In Table~\ref{tab:dcov}, we display the average runtime over 10 runs of exactly computing $\mathsf{dCov}^2$ using our package \package{}, or computing $\mathsf{dCov}^2$ exactly or approximately with the randomized incomplete $U$-statistic using the \texttt{MATLAB} code attached in the supplementary materials of \citet{shao2025u}. We rerun the code of \citet{shao2025u} together with \package{} on the same computing platform to ensure a fair comparison and we only focus on point estimation. In particular, the algorithm proposed in \citet{shao2025u} will randomly choose $O (n^\alpha)$ permutations, with $\alpha > 0$ a tuning parameter determined by the available computation budget. By choosing $\alpha < 4$, the running time will be reduced compared to $O (n^{4})$. We vary the tuning parameter $\alpha$ within the set $\{1.5, 2, 2.5\}$. Our package \package{} again exhibits quite impressive runtime performance: \package{} without parallel computation takes slightly shorter time than the randomized incomplete $U$-statistic when $\alpha = 2$, whereas \package{} with parallel computation takes less than half of the time of the case when $\alpha = 1.5$!

\begin{table}[htbp]
	\centering
	\caption{
		Runtime (in seconds) comparison of various methods for computing $\mathsf{dCov}^{2}$.
		Experiments were run on Intel Xeon ICX Platinum 8358 CPUs (2.6GHz, 64 total cores) with memory of 512 GB.
	}
	\label{tab:dcov}
	\begin{tabular}{cc|cccc}
		\toprule
		\multicolumn{2}{c|}{\package{}}                  & \multicolumn{4}{c}{\citet{shao2025u}'s \texttt{MATLAB} code}                                         \\
		\cmidrule(lr){1-2} \cmidrule(lr){3-6}
		\makecell{\textbf{No Parallel}}                  &
		\makecell{\textbf{Parallel}}                     &
		\makecell{\textbf{Randomized} \\ $\alpha = 1.5$} &
		\makecell{\textbf{Randomized} \\ $\alpha = 2.0$} &
		\makecell{\textbf{Randomized} \\ $\alpha = 2.5$} &
		\makecell{\textbf{Complete}}                                                                                                                            \\
		\midrule
		4.0928                                           & 0.1847                                                       & 0.4211 & 4.5744 & 53.0265 & 3395.4001 \\
		\bottomrule
	\end{tabular}
\end{table}

\section{Concluding Remarks}
\label{sec:conclusions}

In this paper, we have conducted a delicate and systematic study on exactly computing and the complexity of exactly computing higher-order $U$-statistics (and also $V$-statistics). Explicit connections are established between the complexity of computing higher-order $U$/$V$-statistics and concepts from tensor computation and graph theory. An accompanying Python package \package{} and its \href{https://github.com/cxy0714/U-Statistics-R}{R} interface are provided to help practitioners compute $U$/$V$-statistics more efficiently in applications. 
It will also be interesting to test \package{} on other problems involving $U$-statistics \citep{shen2025engression}.

One limitation of our work is that we prioritize time complexity over space complexity. As a result, when we need to store a very high-order tensor, the sample size $n$ cannot be too large; this is why we have only evaluated the performance of \package{} on computing HOIF estimators of order $m \leq 7$ when $n = 10000$. 
Therefore, an important future direction is to explore how to balance time and space complexities when both $m$ and $n$ are even larger.



\phantomsection\label{supplementary-material}
\bigskip

\bibliographystyle{apalike}
\bibliography{Master}

\newpage

\begin{center}

	{\large\bf SUPPLEMENTARY MATERIAL}

\end{center}

\appendix

\numberwithin{example}{section}

\section{More Examples on Einsum Notation and Einsum Operation}
\label{app:examples}

In this section, we provide two additional examples to illustrate the definitions of the \Einsum{} notation and the \Einsum{} operation, supplementing Example~\ref{eg:running}.

\begin{example}
	\label{eg:matrix multiplication}
	We take matrix multiplication as an example to illustrate the \Einsum{} operation just defined. Let $T_1, T_2$ be two input matrices, which are also 2nd-order tensors. Therefore $\calT = (T_1, T_2)$. Suppose we are interested in multiplying $T_1$ and $T_2$, denoted as $T^{\rm out}$: for any $i, j \in [n]$,

	\begin{align*}
		T^{\rm out} (i, j) = \sum\limits_{k = 1}^n T_{1} (i, k) T_{2} (k, j).
	\end{align*}

	The matrix multiplication corresponds to the \Einsum{} operation with \Einsum{} notation $\calN = (\calA; B)$, with input tuple list $\calA = ((1,2),(2,3))$ and output tuple $B = (1,3)$ and $|B| = 2$. The index set is simply $I = [3]$. This means that the operation sums over the dimensions of $T_1$ and $T_2$ corresponding to index $2$ and reserves indices $1$ and $3$. $T^{\rm out} (\alpha)$ for any $\alpha = (i, j) \in  [n]^{2}$ can be written in the \Einsum{} operation form $\tensorcontraction(\calT, \calN)$ because by definition:
	\begin{align*}
		\tensorcontraction(\calT, \calN) & \equiv \sum_{\substack{\alpha'  \in [n]^3 \\ \alpha' [1, 3] = \alpha}} \prod_{k = 1}^{2} T_k (\alpha' [A_k]) = \sum_{\substack{\alpha'
				                                                                                                                                                  \in [n]^3 \\ \alpha' [1, 3] = \alpha}} T_1 (\alpha' [A_1]) T_2 (\alpha' [A_2])                                            \\
		                                 & = \sum_{\substack{\alpha'  \in [n]^3 \\ \alpha' [1, 3] = \alpha}} T_1 (\alpha' [1, 2]) T_2 (\alpha' [2, 3]) = \sum_{\substack{\alpha'  \in [n]^3 \\ \alpha' [1, 3] = \alpha}} T_1 (\alpha' [1], \alpha' [2]) T_2 (\alpha' [2], \alpha'[3]) \\
		                                 & = \sum_{\alpha' [2] = 1}^{n} T_{1} (\alpha [1], \alpha' [2]) T_{2} (\alpha' [2], \alpha [3]) = \sum_{\alpha' [2] = 1}^{n} T_{1} (i, \alpha' [2]) T_{2} (\alpha' [2], j),
	\end{align*}
	so the claim holds by identifying $\alpha' [2] = k$.
\end{example}

\allowdisplaybreaks
\begin{example}
	\label{eg:tensor}
	Here we consider an example involving higher-order tensors. Given three tensors $\calT = (T_1, T_2, T_3)$, each of size $n$, where $T_1$ and $T_2$ are 3rd-order tensors while $T_3$ is a matrix, suppose that we are interested in computing the following 3rd-order tensor: for any $(i, j, k) \in [n]^{3}$
	\begin{align*}
		T^{\rm out} (i, j, k) = \sum_{a = 1}^{n} \sum_{b = 1}^{n} \sum_{c = 1}^{n} T_{1} (i, a, b) T_{2} (a, c, j) T_{3} (b, k).
	\end{align*}
	$T^{\rm out} (i, j, k)$ is in fact a partial trace, a common operation taken in quantum mechanics. The corresponding \Einsum{} notation is $\calN = (\calA; B)$ with $\calA = (A_1 = (1, 2, 3), A_2 = (2, 4, 5), A_3 = (3, 6))$ with $B = (1, 5, 6)$ and the index set $I = [6]$. $T^{\rm out} (\alpha)$ for any $\alpha = (i, j, k) \in [n]^{3}$ can be written in the \Einsum{} operation form $\tensorcontraction(\calT, \calN)$ because by definition:
	\begin{align*}
		\tensorcontraction (\calT, \calN) & \equiv \sum_{\substack{\alpha' \in [n]^4 \\ \alpha' [1, 5, 6] = \alpha}} T_1 (\alpha' [A_1]) T_2 (\alpha' [A_2]) T_3 (\alpha' [A_3])                                                                          \\
		                                  & = \sum_{\substack{\alpha' \in [n]^4 \\ \alpha' [1, 5, 6] = \alpha}} T_1 (\alpha' [1, 2, 3]) T_2 (\alpha' [2, 4, 5]) T_3 (\alpha' [3, 6])                                                                      \\
		                                  & = \sum_{\alpha' [2] = 1}^{n} \sum_{\alpha' [3] = 1}^{n} \sum_{\alpha' [4] = 1}^{n} T_{1} (\alpha [1], \alpha' [2], \alpha' [3]) T_{2} (\alpha' [2], \alpha' [4], \alpha [5]) T_{3} (\alpha' [3], \alpha [6]).
	\end{align*}
	Thus the claim holds by identifying $\alpha' [2] = a$, $\alpha' [3] = b$, and $\alpha' [4] = c$.
\end{example}

\section{A Discussion on the Space Complexity}
\label{app:space_complexity}
The overall space complexity of computing a $U$-statistic can be divided into two components: the tensorization step and the evaluation of the induced $V$-statistics.

For the tensorization step (Algorithm~\ref{alg:tensorization}), let $(h_k)_{k=1}^K, (A_k)_{k=1}^K$ be the MD structure of the kernel. The space complexity is the total number of entries of the generated tensors $\calT = (T_k)_{k=1}^K$, namely
\begin{equation*}
	\sum_{k=1}^{K} n^{|A_k|} = O\big(K\, n^{m_{\max}}\big), \qquad m_{\max} = \max_{k \in [K]} |A_k|,
\end{equation*}
which equals the time complexity of this step.

For computing a single induced $V$-statistic $\Einsum(\tilde{\calT}, \calA_\pi)$, where $\tilde{\calT}$ denotes the collection of tensors constructed as in Lemma~\ref{lem:sparsification_trick} and $\calA_\pi$, $\pi \in \Pi_{m}^\calA$, is some induced decomposition signature, the space complexity is determined by the size of the largest intermediate tensor generated during the \Einsum{} computation, which is of order $O(n^{\tw(G_{\calA_\pi})})$. This is obtained in the proof of Proposition~\ref{pro:complexity_einsum} (Appendix~\ref{app:complexity_einsum}). Since the induced $V$-statistics are evaluated one at a time and their working memory can be reused, the overall space complexity of computing all of them is $\less(n^{M})$, where $M = \max\{\twp{\graphform{\calA}/\pi} \mid \pi\in\partitionp{m}^{\calA}\}$.

Finally, the space complexity of computing the $U$-statistic is the larger of the two, $O (\max\{K\, n^{m_{\max}}, n^{M}\})$. For the running example of Example~\ref{eg:U_complexity}, this gives peak memory $\less(3n^{2})$.

Our theoretical estimate is optimistic and may underestimate the practical memory consumption: the working-memory terms are attained only by an optimal summation ordering, so memory is subject to the same ordering-finding difficulty as the time complexity (Remark~\ref{rem:ordering}), and a suboptimal ordering may enlarge the largest intermediate tensor and hence the peak memory.

To illustrate the practical peak memory, we run a numerical experiment on the HOIF-type $U$-statistics with $m\in\{4,5,\dots,10\}$ at $n\in\{1000,10000\}$, and report the peak memory under different summation-ordering strategies, together with the runtimes, in Table~\ref{tab:HOIF_Nl} of Appendix~\ref{app:HOIF}.

\section{Complexity Analysis: $U$-statistics for high-dimensional testing}
\label{app:examples_complexity}

In \cite{he2021asymptotically} and \cite{lai2023block}, higher-order $U$-statistics are employed to test diagonal and block-diagonal structures of high-dimensional covariance matrices. A detailed computational analysis, including an iterative algorithm and complexity bounds, is provided in \cite{he2021asymptotically}.

Consider the sample $\bm{X} = (X_i)_{i=1}^n$ with $X_i \in \mathbb{R}^p$. The $U$-statistic used to construct the test (omitting the normalizing constant) takes the form
\begin{align}
	\mathcal{U}(a) = \sum_{(j_1,j_2) \in \mathbf{I}} \sum_{\bar{i}_{2a} \in \perm{[n],2a}} \prod_{k=1}^{a} (X_{i_{2k-1},j_1} X_{i_{2k-1},j_2} - X_{i_{2k-1},j_1} X_{i_{2k},j_2}),
	\label{equ:he_example_original}
\end{align}
where $\mathbf{I} \subseteq [p]^2$ depends on the specific testing goal. For diagonal testing, \cite{he2021asymptotically} uses $\mathbf{I} = [p]^2 \setminus \bigcup_{j=1}^p \{(j,j)\}$.

When the mean of each $X_i$ is known to be zero, the statistic simplifies to
\begin{align*}
	\tilde{\mathcal{U}}(a) = \sum_{(j_1,j_2) \in \mathbf{I}} \sum_{\bar{i}_{a} \in \perm{[n],a}} \prod_{k=1}^{a} (X_{i_{k},j_1} X_{i_{k},j_2}).
\end{align*}
For this simplified form, \cite{he2021asymptotically} presents an iterative algorithm (which is a special case of our algorithm) and establishes a computational complexity of $O(p^2 n)$ independent of the order $a$.

In terms of the original test statistic $\mathcal{U} (a)$ in \eqref{equ:he_example_original}, however, \citet{he2021asymptotically} only provides an algorithm and complexity analysis for $a \leq 2$. They remark that:
\begin{quote}
	\emph{When $a \geq 3$, a similar iterative method can be applied, but a closed-form expression may be difficult to derive directly.}
\end{quote}

The difficulty in obtaining a closed-form expression stems precisely from the decomposition challenge addressed in Lemma~\ref{lem:decomp_u_to_v} of our paper, which expresses a general $U$-statistic as a linear combination of $V$-statistics. To circumvent this issue, \cite{he2021asymptotically} instead considers a modified test statistic $\mathcal{U}_c(a)$ resembling $\tilde{\mathcal{U}}(a)$ and shows that it exhibits asymptotic properties similar to $\mathcal{U}(a)$.

Our algorithm \package{} can directly compute $\mathcal{U}(a)$, and our complexity analysis developed in Section~\ref{sec:alg_u} yields the same runtime complexity $O(p^2 n)$. To see this, we first expand the product using the binomial theorem:
\begin{align*}
	\mathcal{U}(a)
	 & = \sum_{(j_1,j_2) \in \mathbf{I}} \sum_{\bar{i}_{2a} \in \perm{[n],2a}} \sum_{S \subseteq [a]} (-1)^{|S|}
	\prod_{k \notin S} (X_{i_{2k-1},j_1} X_{i_{2k-1},j_2})
	\prod_{k \in S} (X_{i_{2k-1},j_1} X_{i_{2k},j_2})                                                            \\
	 & = \sum_{(j_1,j_2) \in \mathbf{I}} \sum_{S \subseteq [a]} (-1)^{|S|} \, \ustat(h_{S,j_1,j_2}, \bm{X}),
\end{align*}

where
\begin{align*}
	h_{S,j_1,j_2}(\bar{X}_{2a})
	= \prod_{k \notin S} (X_{i_{2k-1},j_1} X_{i_{2k-1},j_2})
	\prod_{k \in S} (X_{i_{2k-1},j_1})
	\prod_{k \in S} (X_{i_{2k},j_2}).
\end{align*}

Thus, $h_{S,j_1,j_2}$ admits a MD structure (Definition~\ref{def:mul_decomp}) with
\begin{align*}
	\calH_{S,j_1,j_2} & = \bigl(
	\underbrace{h_1(X_i) = X_{i,j_1} X_{i,j_2}}_{\text{repeated } a - |S| \text{ times}},\
	\underbrace{h_2(X_i) = X_{i,j_1}}_{\text{repeated } |S| \text{ times}},\
	\underbrace{h_3(X_i) = X_{i,j_2}}_{\text{repeated } |S| \text{ times}}
	\bigr),                                                                                     \\
	\calA_S           & = \bigl( (2k-1)_{i \notin S}, (2k-1)_{i \in S} , (2k)_{i \in S} \bigr).
\end{align*}

The corresponding decomposition graph $\graphform{\calA_S}$ (Figure~\ref{fig:decomp_graph_AS}) consists solely of isolated vertices (i.e., it is edgeless). Consequently, all quotient graphs are also edgeless and have treewidth 0.

By Corollary~\ref{cor:U}, each $\ustat(h_{S,j_1,j_2}, \bm{X})$ can be computed in $O(n)$ time for fixed $S$ and $(j_1,j_2)$. Since $|\mathbf{I}| = O(p^2)$ and there are $2^a$ possible subsets $S \subseteq [a]$, the total complexity of computing $\mathcal{U}(a)$ is $O(2^a p^2 n)$. In the regime where $a$ is small relative to $n$ and $p$ (the setting considered in \cite{he2021asymptotically}), this reduces to $O(p^2 n)$.

\begin{figure}[h!]
	\centering
	\begin{tikzpicture}[scale=0.8]
		\node[circle, draw, minimum size=0.8cm] (v1) at (0,0) {};
		\node[circle, draw, minimum size=0.8cm] (v2) at (1.6,0) {};
		\node at (3.2, 0) {$\cdots$};
		\node[circle, draw, minimum size=0.8cm] (v3) at (4.8,0) {};

		\draw [decorate, decoration={brace, amplitude=6pt, mirror}]
		(-0.5, -0.6) -- (5.3, -0.6)
		node[midway, below=8pt] {\small $a + |S|$ vertices};
	\end{tikzpicture}
	\caption{The decomposition graph $\graphform{\calA_{S}}$, which consists of $a + |S|$ isolated vertices.}
	\label{fig:decomp_graph_AS}
\end{figure}

\section{Proof of Proposition~\ref{pro:complexity_einsum}}
\label{app:complexity_einsum}

Our proposed algorithm for computing a tensor $\calT$ is based on iterating the \Einsum{} operation w.r.t. the corresponding \Einsum{} notation $\calA$, summing over one index at a time. We refer readers to Definition~\ref{def:tensor_expression} in the main text for the meaning of the index set of an \Einsum{} notation.

On a high level, the proof proceeds by establishing a correspondence between the \Einsum{} notation and the associated decomposition graph (Definition~\ref{def:graph_decomp}). With the latter, we can leverage concepts from graph theory and combinatorics to obtain a time complexity upper bound of computing $\calT$. In particular, we will demonstrate that the summation over each index in the \Einsum{} notation corresponds to eliminating one vertex of the decomposition graph following Definition~\ref{def:vertex_elimination}. This characterization can help demonstrate that the treewidth of (Definition~\ref{def:treewidth_elimination}) the decomposition graph gives an optimistic estimate of the time complexity of the \Einsum{} operation.

To ease exposition, we decide to present the proof in a slow pace. The roadmap of the proof is as follows. We first construct a sequence of \Einsum{} notations, described in Definition~\ref{def:einsum_sequence} and Lemma~\ref{lem:einsum_graph_seq}, to mimic the computation procedure presented in Section~\ref{sec:alg_v}. Next, in Lemma~\ref{lem:complexity_one_index}, we obtain the computational complexity of summing over an index, which corresponds to the degree of the associated vertex in the decomposition graph. Finally, we complete the proof by showing that the optimal ordering of summation over indices is equivalent to identifying the best vertex elimination order of the decomposition graph. This is why the treewidth of the decomposition graph can be used to give an optimistic estimate of the time complexity.

\begin{definition}[\Einsum{} Notation Sequence]
	\label{def:einsum_sequence}
	Let $\calA = (A_1, \cdots, A_K)$ be an \Einsum{} notation (recall from Definition~\ref{def:tensor_expression}) with no output. Let $I = [m]$ be the index set of $\calA$ and $\sigma$ be a permutation over $I$. The \Einsum{} notation sequence of $\calA$ under $\sigma$ is the sequence $\calA^{\sigma}_0, \calA^{\sigma}_1, \ldots, \calA^{\sigma}_m$ defined recursively as follows:
	\begin{itemize}
		\item $\calA^{\sigma}_0 = \calA$.
		\item For each $i \in I$, let $s_{\sigma[i]}$ be a permutation of the following set
		      \begin{equation*}
			      \bigcup_{\substack{ A \in \calA^{\sigma}_{i-1} \\ \sigma[i] \in A}} (\takesetp{A} \setminus \{\sigma[i]\}).
		      \end{equation*}.
		      Then $\calA^{\sigma}_i$ is obtained from $\calA^{\sigma}_{i-1}$ by:
		      \begin{itemize}
			      \item Removing tuples in $\calA_{i - 1}^{\sigma}$ containing $\sigma[i]$.
			      \item Inserting $s_{\sigma[i]}$ as the last tuple into the tuple list if $s_j \ne \emptyset$.
			      \item Denoting the new tuple list as $\calA^{\sigma}_{i}$.
		      \end{itemize}
	\end{itemize}
\end{definition}

We use the following concrete example to further illustrate the above definition.

\begin{example}
	Let $\calA = ((1,2), (1,3), (1,4), (2,3), (2,4), (3,4))$ and $\sigma = (1,2,3,4)$. Then the \Einsum{} notation sequence of $(\calA; \emptyset)$ under $\sigma$ is comprised of:
	\begin{align*}
		\calA^{\sigma}_{0} & = ((1,2), (1,3), (1,4), (2,3), (2,4), (3,4)), \\
		\calA^{\sigma}_{1} & = ((2,3), (2,4), (3,4),(2,3,4)),              \\
		\calA^{\sigma}_{2} & = ((3,4), (3,4)),                             \\
		\calA^{\sigma}_{3} & = ((4)),                                      \\
		\calA^{\sigma}_{4} & = (()).
	\end{align*}
\end{example}

The following lemma establishes a one-to-one correspondence between the \Einsum{} notation sequence under a particular summation ordering $\sigma$ defined in Definition~\ref{def:einsum_sequence} and vertex elimination given in Definition~\ref{def:vertex_elimination}.

\begin{lemma}
	\label{lem:einsum_graph_seq}
	Let $\calA$ be an \Einsum{} notation with index set $I = [m]$ and $\sigma$ be a permutation of $I$. Let the sequence $\calA^{\sigma}_0, \calA^{\sigma}_1, \ldots, \calA^{\sigma}_m$ be the \Einsum{} notation sequence w.r.t. $\calA$ under permutation $\sigma$ (see Definition~\ref{def:einsum_sequence}) and $\graphform{\calA^{\sigma}_0}, \graphform{\calA^{\sigma}_1}, \ldots, \graphform{\calA^{\sigma}_m}$ be the corresponding decomposition graphs. Then for each $i \in [m]$, it holds that
	\begin{align*}
		\graphform{\calA^{\sigma}_i} = \eliminatevg{\sigma[i]}{\graphform{\calA^{\sigma}_{i-1}}}.
	\end{align*}
\end{lemma}

\begin{proof}
	Recalling from Definition~\ref{def:vertex_elimination} about the vertex elimination operation over a graph and Definition~\ref{def:graph_decomp} about the decomposition graph associated with a tuple list $\calA$ and tet $s_{\sigma[i]}$ be a permutation of the following set
	\begin{equation*}
		\bigcup_{\substack{ A \in \calA^{\sigma}_{i-1} \\ \sigma[i] \in A}} (\takesetp{A} \setminus \{\sigma[i]\}),
	\end{equation*}
	we then recognize that removing all tuples containing $\sigma[i]$ from the tuple list $\calA^{\sigma}_{i - 1}$ corresponds to removing the vertex $\sigma[i]$ and all edges incident to $\sigma[i]$ in the decomposition graph $\graphform{\calA^{\sigma}_{i-1}}$ associated with $\calA^{\sigma}_{i-1}$ (some edges connected neighbors of $\sigma[i]$ may also be removed but will be re-added latter), while inserting the tuple $s_{\sigma[i]}$ into $\calA^{\sigma}_{i-1}$ corresponds connecting the neighbors of $\sigma[i]$ into a clique, which completes the proof.
\end{proof}

\begin{lemma}
	\label{lem:complexity_one_index}
	Let $\calN = (\calA; B)$ be an \Einsum{} notation. Let $I = \indexset(\calN)$ and $m \coloneqq |I|$. Suppose that there is at least one common index $i$ in each $A_k$ and $B$ is a permutation of $I \backslash \{i\}$. Then for any tensors $\calT$ that can be represented by $\calN$ with each tensor of size $n$, the time complexity of computing $\tensorcontraction(\calT, \calN)$ is $\appequal(|\calA|n^{d})$.
\end{lemma}

\begin{proof}
	We immediately have $|B| = d - 1$, which implies that the result of $\tensorcontraction(\calT, \calN)$ has $n^{d-1}$ entries. For each entry, as one index is summed over, it takes $n-1$ summation operations and $(|\calA|-1)n$ multiplication operations. Therefore, the total time complexity of computing $\tensorcontraction(\calT, \calN)$ is $\big((n-1)n^{d-1} + (|\calA|-1)n^d\big)$, which is $\appequal(|\calA|n^{d})$.
\end{proof}

We need to introduce another notation before presenting the final proof: given two lists of tuples $\calA$ and $\calB$, we write $\calA \sqsubseteq \calB$ to mean that $\calA$ is a sublist of $\calB$.

\begin{proof}[Proof of Proposition~\ref{pro:complexity_einsum}]
	Let $m \coloneqq |\verticesp{G_\calA}|$. Let $\sigma$ be a permutation of $[m]$, and $\graphform{\calA^{\sigma}_0}, \graphform{\calA^{\sigma}_1}, \ldots, \graphform{\calA^{\sigma}_m}$ be the corresponding decomposition graphs. We define the $\calA^{\sigma}_0, \calA^{\sigma}_1, \ldots, \calA^{\sigma}_m$ be the \Einsum{} notation sequence under $\sigma$ as in Definition~\ref{def:einsum_sequence}.

	We first show that there is an algorithm such that the time complexity of computing $\tensorcontraction(\calT, \calA)$ is $\less(n^{C_\sigma + 1}) $ where
	\begin{equation*}
		C_\sigma \coloneqq \max_{i \in [m]} \degp{\sigma[i]}{\graphform{\calA^{\sigma}_{i-1}}}.
	\end{equation*}

	To see this, we use \Einsum{} notation to characterize this procedure. First, for $i \in [m] \cup {0}$, we define the tensor tuple $\calT^\sigma_i$ corresponding to the \Einsum{} notation sequence $\calA^\sigma_i$ as follows: $\calT^{\sigma}_{0} \coloneqq \calT$ and for each $i \in [m]$:
	\begin{itemize}
		\item $\sigma[i]$ is the index to be eliminated at the $i$-th step;
		\item Let $\calA^{\sigma, \mathrm{in}}_i \sqsubseteq \calA^{\sigma}_{i-1}$ be the sublist of $\calA^{\sigma}_{i-1}$ that contains $\sigma[i]$, i.e.
		      \begin{align*}
			      A \in \calA^{\sigma, \mathrm{in}}_i \iff \sigma[i] \in A,
		      \end{align*}
		      and similarly let $\calT^{\sigma, \mathrm{in}}_i \sqsubseteq \calT^{\sigma}_{i-1}$ be the tensor tuple corresponding to $\calA^{\sigma, \mathrm{in}}_i $;

		\item Let $B^{\sigma, \mathrm{out}}_i$ be a permutation of the set $\bigcup_{\substack{A \in \calA^{\sigma}_{i-1} \\ \sigma[i] \in A}} (\takesetp{A} \setminus \{\sigma[i]\})$ as the output tuple and $T^{\sigma, \mathrm{out}}_i \coloneqq \tensorcontraction(\calT^{\sigma, \mathrm{in}}_i, (\calA^{\sigma, \mathrm{in}}_i, B^{\sigma, \mathrm{out}}_i))$;

		\item Update the tensor tuple $\calT^{\sigma}_i$ from $\calT^{\sigma}_{i-1}$ by
		      \begin{itemize}
			      \item removing all tensors in $\calT^{\sigma, \mathrm{in}}_i$ from $\calT^{\sigma}_{i-1}$, and then
			      \item inserting $T^{\sigma, \mathrm{out}}_i$ to the end of $\calT^{\sigma}_{i-1}$ to form $\calT^{\sigma}_{i}$.
		      \end{itemize}
	\end{itemize}

	According to the distributive property of multiplication over summation, it holds that for any $i,j \in [m]$,
	\begin{align*}
		\tensorcontraction(\calT^\sigma_i, \calA^\sigma_i) \equiv  \tensorcontraction(\calT^\sigma_j, \calA^\sigma_j) \equiv \tensorcontraction(\calT, \calA).
	\end{align*}

	Thus the time complexity of computing $\tensorcontraction(\calT, \calA)$ just comes from transferring from one step to the next step, which is $\tensorcontraction(\calT^{\sigma, \mathrm{in}}_i, (\calA^{\sigma, \mathrm{in}}_i, B^{\sigma, \mathrm{out}}_i))$. Following Lemma~\ref{lem:complexity_one_index}, we obtain that the time complexity of computing $\tensorcontraction(\calT^{\sigma, \mathrm{in}}_i, (\calA^{\sigma, \mathrm{in}}_i, B^{\sigma, \mathrm{out}}_i))$ is $\appequal(|\calA^{\sigma, \mathrm{in}}_i| n^{\degp{x}{\graphform{\calA^{\sigma}_{i-1}}}})$, which is also $\appequal(|\calA|n^{\degp{\sigma[i]}{\graphform{\calA^{\sigma}_{i-1}}}})$ as $|\calA^{\sigma, \mathrm{in}}_i| \leq |\calA|$. By taking maximum over $i$, the overall time complexity is thus $\less(|\calA|n^{C_\sigma + 1})$.

	Finally, following Lemma~\ref{lem:einsum_graph_seq}, we have
	\begin{align*}
		\graphform{\calA^{\sigma}_i} = \eliminatevg{\sigma[i]}{\graphform{\calA^{\sigma}_{i-1}}},
	\end{align*}
	and by finding the permutation $\sigma^{\ast}$ minimizing $ C_{\sigma}$ and recalling Definition~\ref{def:treewidth_elimination}, we can conclude that there exists an algorithm such that the time complexity of $\tensorcontraction(\calT, \calA)$ is $\less(|\calA|n^{\twp{G_\calA}+1})$, which completes the proof.
\end{proof}

\section{Proof of Lemma~\ref{lem:decomp_u_to_v}}
\label{app:decomp_u_stats}

Here we prove an extended version of Lemma~\ref{lem:decomp_u_to_v}, as shown in Lemma~\ref{lem:ext_decomp_u_to_v} below. Lemma~\ref{lem:decomp_u_to_v} then becomes a simple corollary of Lemma~\ref{lem:ext_decomp_u_to_v}. We begin by introducing the following definitions.

\begin{definition}[Refinement of Partitions]\label{def:refinement_of_partition}
	Let $\sPi$ and $\rho$ be two partitions of a non-empty set $S$. $\sPi$ is said to be finer than $\rho$, if for any $A \in \sPi$, there exists $B \in \rho$ such that $A \subseteq B$, denoted by $\sPi \pleq \rho$. Additionally, $\sPi$ is strictly finer than $\rho$ if $\sPi \neq \sPi$, denoted by $\rho \pless \sPi$.
\end{definition}

\begin{remark}
	Refinement defines a partial order on partitions of a set and in fact forms a \emph{partition lattice}, a classical result in lattice theory \citep{birkhoff1940lattice}.
\end{remark}

Similar to restricted $V$-statistics defined in Definition~\ref{def:res_v_stat}, we define U-set and restricted $U$-statistics as follows.

\begin{definition}[U-Set]\label{def:u_set}
	Let $m,n$ be positive integers satisfying $m \leq n$ and $\sPi$ be a partition of the set $[m]$. A U-set of size $n$ w.r.t. the partition $\sPi$ is defined as
	\begin{align*}
		\usetnm{n}{\sPi} \coloneqq \left\{ \overline{s}_m \in [n]^m \mid \text{where for any $i, j \in [m]$}, i, j \in Q \text{ for some $Q \in \sPi$} \iff   s_i = s_j \right\}.
	\end{align*}
	If $\sPi$ happens to be the finest partition of $[m]$, i.e. $\sPi = \{\{1\}, \{2\}, \cdots, \{m\}\}$, the corresponding U-set is denoted by $\uset(n,m)$ and in fact, $\uset(n,m) \equiv \perm{n,m}$. The cardinality of the given partition $\sPi$ is referred to as the order of $\usetnm{n}{\sPi}$.
\end{definition}

\begin{definition}[Restricted $U$-Statistic]\label{def:res_u_stats}
	Let $\bbX$ be a non-empty set, $m, n$ be positive integers satisfying $m \leq n$, $\sPi$ be a partition of set $[m]$ and $h: \bbX^{m} \rightarrow \bbR$ be a function defined on $\bbX^m$. For any tuple $\bm{X} \in \bbX^n: n \geq m$, the $U$-statistic with kernel $h$ restricted by partition $\sPi$ takes the following form
	\begin{align*}
		\ustat[\pi](h, \bm{X}) \coloneqq \sum_{\alpha \in \uset(n,\pi)} h(\bm{X}[\alpha]).
	\end{align*}
\end{definition}

Armed with Definitions~\ref{def:refinement_of_partition} -- \ref{def:res_u_stats}, we are ready to present the following extended form of Lemma~\ref{lem:decomp_u_to_v}.

\begin{lemma}[Extended Form of Lemma~\ref{lem:decomp_u_to_v}]
	\label{lem:ext_decomp_u_to_v}
	Let
	\begin{itemize}
		\item $\samplespace$ be a non-empty set,
		\item $m,n$ be integers satisfying $m \leq n$,
		\item $h$ be a function defined on $\samplespace^m$,
		\item $\bm{X} \in \samplespace^n$.
	\end{itemize}
	Then for any $\sPi \in \partitionp{m}$, it holds that
	\begin{align*}
		\ustat[\sPi](h,\bm{X})
		= \sum_{\substack{ \rho \in \partitionp{m} \\ \sPi \pleq \rho }}
		\mu(\sPi, \rho)  \vstat[\rho](h, \bm{X}),
	\end{align*}
	where the coefficients $\mu(\sPi, \rho)$ are given by
	\begin{align*}
		\mu(\sPi, \rho)
		= (-1)^{|\sPi|-|\rho|}
		\prod_{C \in \rho} (\kappa_{\sPi, C} - 1)!, \text{ and } \kappa_{\sPi, C} \coloneqq |\{ Q \in \sPi \mid Q \subseteq C \}|.
	\end{align*}
	In particular, if $\sPi$ happens to be $\{\{k\} \mid k \in [m]\}$, the above result reduces to Lemma~\ref{lem:decomp_u_to_v}.
\end{lemma}

The proof of Lemma~\ref{lem:ext_decomp_u_to_v} can be found in Appendix~\ref{app:proof_ext_decomp_u_to_v}.

\subsection{Preparatory Materials}\label{app:preparatory_materials}

Before presenting the proof in the next section, we introduce some useful preparatory materials, including: (1) the decomposition of a $V$-statistic into $U$-statistics and (2) a dual form of \mobius{} inversion formula \citep{stanley2011enumerative}. We first state the following technical result of decomposing a $V$-statistic into a linear combination of $U$-statistics.
\begin{lemma}[Decomposition of $V$-Statistic into $U$-Statistics]
	\label{lem:decomp_v_to_u}
	Let
	\begin{itemize}
		\item $\samplespace$ be a non-empty set,
		\item $m,n$ be positive integers satisfying $m \leq n$,
		\item $h$ be a function defined on $\bbX^m$,
		\item $\bm{X} \in \samplespace^n$.
	\end{itemize}
	Then for any $\sPi \in \partitionp{m}$, it holds that
	\begin{align*}
		\vstat[\sPi](h, \bm{X}) = \sum_{\substack{
				                          \rho \in \partitionp{m} \\ \sPi \pleq \rho
			                          }} \ustat[\rho](h,\bm{X}),
	\end{align*}
	where
	\begin{align*}
		\ustat[\rho](h,\bm{X}) = \sum_{\alpha \in \usetnm{n}{\rho}} h(\bm{X}[\alpha]).
	\end{align*}
\end{lemma}

\begin{proof}
	It is equivalent to prove that $\{\usetnm{n}{\rho} \mid \rho \in \partitionp{m}: \sPi \pleq \rho\}$ forms a partition (i.e. disjoint union) of $\vsetp{n}{\sPi}$, i.e.
	\begin{align*}
		\vsetp{n}{\sPi} = \bigsqcup_{\substack{\rho \in \partitionp{m} \\ \sPi \pleq \rho} }\usetnm{n}{\rho}.
	\end{align*}

	{\bf Part 1: Disjointness}

	We first prove that for any $\pi_1, \pi_2 \in \{\rho \in \partitionp{m} \mid \sPi \pleq \rho \}: \pi_1 \neq \pi_2$, it holds that
	\begin{align*}
		\usetnm{n}{\pi_1} \cap \usetnm{n}{\pi_2} = \emptyset.
	\end{align*}

	Since $\pi_1 \neq \pi_2$, there must exist $i_1, i_2 \in [m]$ such that either
	\begin{itemize}
		\item $\exists \, Q_1 \in \pi_1$, such that $i_1, i_2 \in Q_1$, and
		\item $\forall \, Q_2 \in \pi_2$, $i_1, i_2 \notin Q_2$,
	\end{itemize}
	or
	\begin{itemize}
		\item $\forall \, Q_1 \in \pi_1$, $i_1, i_2 \notin Q_1$,
		\item $\exists \, Q_2 \in \pi_2$, such that $i_1, i_2 \in Q_2$.
	\end{itemize}
	Without loss of generality, we assume that the former holds. Following Definition~\ref{def:u_set}, we have for any $\alpha_1 \in \usetnm{n}{\pi_1}$, $\alpha_1[i_1] = \alpha_1[i_2]$, but for any $\alpha_2 \in \usetnm{n}{\pi_2}$, $\alpha_1[i_1] \neq \alpha_1[i_2]$, which implies that $\usetnm{n}{\pi_1} \cap \usetnm{n}{\pi_2} = \emptyset$.

		{\bf Part 2: Coverage}

	We are left to prove that for any $\alpha \in \vsetp{n}{\sPi}$, there exists $\rho \in \partitionp{m}: \sPi \pleq \rho$, such that $\alpha \in \usetnm{n}{\rho}$. We first construct a candidate $\rho$ as follows. For each $j \in \takesetp{\alpha}$, let
	\begin{align*}
		Q_j \coloneqq \{i \in [m] \mid \alpha[i] = j\},
	\end{align*}
	and then let $\rho \coloneqq \{Q_j \mid j \in \takesetp{\alpha}\}$. It is straightforward to verify that $\rho$ is a partition of $[m]$.

	Then we prove that $\rho$ is the partition satisfying the desiderata. Following Definition~\ref{def:u_set}, we have $\alpha \in \usetnm{n}{\rho}$. Following Definition~\ref{def:v_set}, we have for any $Q \in \sPi$ and any $i_1, i_2 \in Q$, $\alpha[i_1] = \alpha[i_2]$, which implies that for any $Q \in \pi$, $Q \subseteq B_{j_Q} \in \rho$, for some $j_Q$ such that
	\begin{align*}
		\alpha[j_Q] = \alpha[i], \quad \forall \, i \in Q,
	\end{align*}
	We then obtain that $\sPi \pleq \rho$ following Definition~\ref{def:refinement_of_partition}, which complete the proof of coverage.

	Combining Part 1 and Part 2, we complete the proof of this lemma.
\end{proof}

Next, we introduce the notion of \mobius{} inversion formula, which is used to reverse the decomposition from a $V$-statistic into $U$-statistics in Lemma~\ref{lem:decomp_v_to_u} to obtain the decomposition from a $U$-statistic into $V$-statistics.

\begin{definition}[\mobius{} Function, \cite{stanley2011enumerative}]\label{def:mobius_func}
	Let $(P, \preceq)$ be a finite poset and, a function $\mu: P^2 \to \reals$ is said to be a \mobius{} function on $(P, \preceq)$, if
	\begin{align*}
		\mu(s,s) = & \ 1, \quad \forall \ s \in P, \text{ and}                                                           \\
		\mu(s,u) = & - \sum_{\substack{t \in P \\ s \preceq t \prec u}} \mu(s,t), \quad \forall \ s, u \in P: s \prec u.
	\end{align*}
\end{definition}

\begin{lemma}[\mobius{} Function on Partition Lattices, \cite{stanley2011enumerative}]\label{lem:mobius_func_partition}
	Let $S$ be a non-empty set. Then the \mobius{} function on $(\partitionp{S}, \pleq)$ takes the following form
	\begin{align*}
		\mu(\sPi, \rho) = (-1)^{|\sPi|-|\rho|} \prod_{C \in \rho} (\kappa_{\sPi, C} - 1)!, \quad \forall \, \sPi,\rho \in \partitionp{S}: \sPi \pleq \rho \text{ and } \kappa_{\sPi, C} = |\{ A \in \sPi \mid A \subseteq C \}|.
	\end{align*}
\end{lemma}

\begin{lemma}[\mobius{} Inversion Formula, Dual Form,  \cite{stanley2011enumerative}]\label{lem:mobius_formula_dual}
	Let $(P, \preceq)$ be a finite poset and $f, g$ be functions defined on $P$. Then
	\begin{align*}
		g(s) = \sum_{\substack{t \in P \\ s \preceq t}}f(t), \quad \forall \, s \in P,
	\end{align*}
	if and only if
	\begin{align*}
		f(s) = \sum_{\substack{t \in P \\ s \preceq t}} \mu(s,t)g(t), \quad \forall \, s \in P,
	\end{align*}
	where $\mu$ is the \mobius{} function on $(P, \preceq)$.
\end{lemma}

\subsection{Proof of Lemma~\ref{lem:ext_decomp_u_to_v}}
\label{app:proof_ext_decomp_u_to_v}

In this section, we prove Lemma~\ref{lem:ext_decomp_u_to_v}. The argument is straightforward given the preparations in the previous section.

\begin{proof}[Proof of Lemma~\ref{lem:ext_decomp_u_to_v}]
	If we fix $h$ and $\bm{X}$, $f(\sPi) \coloneqq \ustat[\sPi](h,\bm{X})$ and $g(\rho) \coloneqq \vstat[\rho](h,\bm{X})$ are functions defined on $\partitionp{m}$.

	Since $\partitionp{m}$ is finite and the refinement defined in Definition~\ref{def:refinement_of_partition} defines a partial order on $\partitionp{m}$, we can apply Lemma~\ref{lem:mobius_formula_dual} to functions $f$ and $g$.

	Combining the \mobius{} function on $\partitionp{m}$ as stated in Lemma~\ref{lem:mobius_func_partition}, the dual form of \mobius{} inversion formula given in Lemma~\ref{lem:mobius_formula_dual} and the relationship between $f$ and $g$ as stated in Lemma~\ref{lem:decomp_v_to_u}, we immediately complete the proof.
\end{proof}

\section{An Open Problem}
\label{app:open}

In this section, we present an open problem that we ourselves find intriguing. As we indicated in Remark~\ref{rem:tw} at the end of Section~\ref{sec:alg_u_v}, the treewidth of the decomposition graph associated with the \Einsum{} notation of the underlying $V$-statistic relates to its time complexity. But the treewidth of a graph is not easy to compute either analytically \citep{harvey2018treewidth} or numerically \citep{arnborg1987complexity}. It will be interesting to get a sharp bound on treewidth using the number of vertices or edges of the graph. As mentioned in Remark~\ref{rem:tw}, \citet{kneis2009bound} has shown that the treewidth is bounded at most linearly by the edge size, or more precisely bounded by $e / 5.769 + O (\log n)$ where $n \coloneqq |V|$ denotes the number of vertices in $G$. However, we do not know whether this bound can be further improved. Below we provide the maximum tree-width that is possible for a graph with edge size no greater than $15$, to hopefully serve as a starting point of fully addressing this problem.

\begin{observation}
	\label{pro:treewidth_table}
	Let
	\begin{equation*}
		\mathrm{t}(e) \coloneqq \max \left\{ \twp{G} \;\middle|\; G \text{ is a graph where } |E(G)| = e \right\}.
	\end{equation*}
	We have computed the exact values of $\mathrm{t} (e)$ for all $1 \le e \le 15$ as follows:
	\begin{equation*}
		\begin{split}
			\mathrm{t} (e) = \left\{ \begin{array}{ll}
				                         1, & e \in \{1, 2\},               \\
				                         2, & e \in \{3, 4, 5\},            \\
				                         3, & e \in \{6, 7, 8, 9\},         \\
				                         4, & e \in \{10, 11, 12, 13, 14\}, \\
				                         5, & e = 15.
			                         \end{array} \right.
		\end{split}
	\end{equation*}
\end{observation}
The proof can be found in Appendix~\ref{app:treewidtg_table}. However, extending the above result to $e > 15$ is currently beyond our reach.

\section{Proof of Observation~\ref{pro:treewidth_table}}
\label{app:treewidtg_table}

We begin by recalling some well-known bounds on treewidth. We then restrict our attention to simple graphs without isolated vertices and derive recurrence relations in several basic cases. The final proof proceeds by carefully verifying the conditions of Lemma~\ref{lem:property_of_tt} to be stated below.

\begin{definition}[Degeneracy]
	\label{def:degeracy}
	The \emph{degeneracy} of a graph $G$, denoted by $\degeneracyg{G}$, is defined as the maximum value of the minimum degree $\delta(H)$ over all subgraphs $H$ of $G$.
\end{definition}

We first present some results on inequalities involving treewidth. The first two inequalities are from \citet{bodlaender2011treewidth}, while the last one follows directly from Definition~\ref{def:treewidth_elimination}.

\begin{lemma}
	\label{lem:simple_bounds_on_treewidth}
	Let $G = (V, E)$ be a graph. Then
	\begin{enumerate}
		\item $ \degeneracyg{G} \leq \twp{G},$

		\item for any vertex $v \in V$, it holds that $\twp{G} \leq \max \{\degp{v}{G}, \twp{\eliminatevg{v}{G}} \}$.
	\end{enumerate}
\end{lemma}

We then present results concerning the treewidth of graphs subject to constraints on the number of edges or vertices.

\begin{lemma}
	\label{lem:property_of_tt}
	Let $\graphsetne{n}{e}$ be the set of all simple graphs with $n$ vertices, $e$ edges, and no isolated vertices, and let $\graphsetne{n}{e} = \emptyset$ if no such graphs exist. Let $\mathrm{tt}(n, e)$ denote the maximum treewidth of the graphs in $\graphsetne{n}{e}$:
	\begin{equation*}
		\mathrm{tt}(n, e) \coloneqq \max\left\{ \twp{G}\;\middle|\; G \in \graphsetne{n}{e} \right\}.
	\end{equation*}
	The following properties hold:
	\begin{enumerate}
		\item \textbf{Support set:} For any fixed $e \in \positivenaturals $, $\graphsetne{n}{e}$ is nonempty if and only if
		      \begin{equation*}
			      n \in \mathrm{n}(e) \coloneqq \left\{ n \in \positivenaturals \, \middle \vert\,  \binom{n}{2} \geq e,  n \leq 2e \right\}.
		      \end{equation*}
		\item \textbf{Representation of $\mathrm{t}(e): $} $ \mathrm{t}(e) = \max_{n \in \mathrm{n}(e)} \mathrm{tt} (n, e).$
		\item \textbf{Monotonicity in edge counts:} For any fixed $n: n \in \mathrm{n}(e+1) $, the function $\mathrm{tt}(n, e)$ is non-decreasing w.r.t. $e$.
		\item \textbf{Treewidth of complete graph:} $\mathrm{tt}(n, \binom{n}{2}) =  n - 1$, $\mathrm{tt}(n, \binom{n}{2} - 1) = n - 2$.
		\item \textbf{Recurrence relations in simple cases:}  recall that $\frac{2e}{n}$ is in fact the average degree of the graph and when $e > 2$ and $n \in \mathrm{n} (e)$:
		      \begin{itemize}
			      \item[(1)] If $1 \leq \frac{2e}{n} < 2$, then $ \mathrm{tt}(n,e) \leq \max\{ 1, \mathrm{t}(e-1) \}$.
			      \item[(2)] If $1 \leq \frac{2e}{n} < 3$, then $ \mathrm{tt}(n,e)  \leq \max\{ 2, \mathrm{t}(e-1) \} $.
			      \item[(3)] If $1 \leq \frac{2e}{n} < 4$, then $ \mathrm{tt}(n,e)  \leq \max\{ 3, \mathrm{tt}(n - 1 ,e), \mathrm{t}(e-1) \}$.
			      \item[(4)] $ \mathrm{tt}(7,14) \leq  \max\{ 4,  \mathrm{t}(13) \}$.
		      \end{itemize}
	\end{enumerate}
\end{lemma}

\begin{proof}
	We prove the lemma in the following steps.
	\begin{enumerate}
		\item Since a simple graph with $n$ vertices and no isolated vertices can have at most $\binom{n}{2}$ edges, we must have $e \leq \binom{n}{2}$. Moreover, as isolated vertices are not allowed, every vertex must have degree at least one. In particular, the average degree satisfies $\frac{2e}{n} \geq 1$, which implies $n \leq 2e$.

		\item The definition of $\mathrm{t}(e)$ does not exclude graphs with isolated vertices. However, for any graph $G$ with $e$ edges and isolated vertices, eliminating those vertices (with degree 0) does not affect the number of edges or the treewidth. Therefore, it suffices to take the maximum over all graphs in $\graphsetne{n}{e}$ for $n \in \mathrm{n}(e)$, i.e. taking the maximum over $\mathrm{tt}(n,e)$ for $n \in \mathrm{n}(e)$.

		\item We observe that any graph $G \in \graphsetne{n}{e}$ can be extended to a graph $G^\prime \in \graphsetne{n}{e+1}$ by adding an edge, since $n \in \mathrm{n}(e+1)$. By the subgraph monotonicity of treewidth stated in Lemma~\ref{lem:property_tw}, it follows that $\twp{G} \leq \twp{G^\prime}$. Therefore, we have $\mathrm{t}(e) \leq \mathrm{t}(e+1)$.

		\item First, consider the case where $e = \binom{n}{2}$. The only graph (up to isomorphism) with $n$ vertices and $\binom{n}{2}$ edges is the complete graph $K_n$. Since the degeneracy of $K_n$ is $n - 1$, we have $\tw(K_n) \geq n - 1$. To show the upper bound, note that eliminating any vertex in $K_n$ removes a vertex of degree $n - 1$ and yields $K_{n-1}$. By the second bound in Lemma~\ref{lem:simple_bounds_on_treewidth}, we get $\tw(K_n) \leq \max\{n - 1, \tw(K_{n-1})\}$. Repeating this process down to $K_2$, we have $\tw(K_n) \leq \max\{n - 1, \dots, \tw(K_2)\} = n - 1$, since $\tw(K_2) = 1$ is easy to verify. Therefore, $\mathrm{tt}(n, \binom{n}{2}) = n - 1$.

		      Next, consider $e = \binom{n}{2} - 1$. There is still only one graph in this case (up to isomorphism), obtained by deleting one edge from $K_n$. In this graph, two vertices have degree $n - 2$ and the rest have degree $n - 1$. Since it contains $K_{n-1}$ as a subgraph, we have a lower bound $\tw \geq n - 2$ by the degeneracy. For the upper bound, eliminating one of the vertices with degree $n - 2$ yields $K_{n-2}$. Thus, $\tw \leq \max\{n - 2, \tw(K_{n-2})\} = n - 2$. Therefore, $\mathrm{tt}(n, \binom{n}{2} - 1) = n - 2$.

		\item Let us introduce another notation: let $\graphset_e \coloneqq \bigcup_{n \in \mathrm{n}(e)} \graphsetne{n}{e}$.

		      \textbf{Case (1): $1 \leq \frac{2e}{n} < 2$}, the minimum degree in any graph in $\graphsetne{n}{e}$ must be exactly $1$. Eliminating such a vertex removes exactly one edge and yields a graph in $\graphset_{e - 1}$. Therefore, by the second bound in Lemma~\ref{lem:simple_bounds_on_treewidth}, we have $\mathrm{tt}(n,e) \leq \max\{1, \mathrm{t}(e - 1)\}$.

		      \textbf{Case (2): $1 \leq \frac{2e}{n} < 3$}, the minimum degree in any graph from $\graphsetne{n}{e}$ must be either $1$ or $2$. If the minimum degree is $2$, eliminating such a vertex removes two edges (if its neighbors are adjacent) or one edge (if they are not, in which case we need to add one edge to connect them), and in both cases, one vertex is removed. The resulting graph lies in either $\graphsetne{n-1}{e-1}$ or $\graphsetne{n-1}{e-2}$. On the other hand, if the minimum degree is $1$, then as in the previous case, eliminating that vertex removes one edge and results in a graph in $\graphset_{e-1}$.     Therefore, by the vertex elimination bound (Lemma~\ref{lem:simple_bounds_on_treewidth}), we have $\mathrm{tt}(n,e) \leq \max\{2, \mathrm{t}(e - 1), \mathrm{tt}(n-1, e - 1), \mathrm{tt}(n-1, e - 2)\}$. Since $\mathrm{tt}(n-1, e - 2) \leq \mathrm{tt}(n-1, e - 1) \leq \mathrm{t}(e - 1)$, we conclude $\mathrm{tt}(n,e) \leq \max\{2, \mathrm{t}(e - 1)\}$.

		      \textbf{Case (3): $1 \leq \frac{2e}{n} < 4$}. The minimum degree in any graph from $\graphsetne{n}{e}$ must be either $1$, $2$, or $3$. If the minimum degree is $3$, then eliminating such a vertex removes between $0$ and $3$ edges, depending on the connectivity of its neighbors, but in all cases it only removes exactly one vertex. Therefore, the resulting graph lies in $\graphsetne{n-1}{e}$, $\graphsetne{n-1}{e-1}$, $\graphsetne{n-1}{e-2}$, or $\graphsetne{n-1}{e-3}$. Combining this with the previous cases for minimum degree $1$ and $2$, we obtain the recurrence relation: $\mathrm{tt}(n,e) \leq \max\{3, \mathrm{tt}(n-1,e), \mathrm{tt}(n-1,e-1), \mathrm{tt}(n-1,e-2), \mathrm{tt}(n-1,e-3), \mathrm{t}(e-1)\} \leq \max\{3, \mathrm{tt}(n-1,e), \mathrm{t}(e-1)\}$.

		      \textbf{Case (4): $n = 7$, $e = 14$}. First we observe that $\frac{2e}{n} = 4$, so the minimum degree in any graph $G = (V, E)$ from $\graphsetne{n}{e}$ must be $1$, $2$, $3$, or $4$.

		      We focus on the case where the minimum degree is $4$, and argue that in this case, there must exist a vertex $v_1$ of degree exactly $4$ such that eliminating $v_1$ does not increase the number of edges.

		      In general, eliminating a vertex of degree $4$ may remove 4 edges and add up to $\binom{4}{2} = 6$ fill-in edges among its neighbors, so the net change in edge counts, denoted by $\Delta_e$, ranges from $-4$ to $+2$. We now show that in our setting, increasing the number of edges (i.e., $\Delta_e > 0$) is impossible.

		      Let $v_1$ be a vertex of degree $4$, and denote its neighborhood as $\{v_2, v_3, v_4, v_5\} = \neighborvg{v_1}{G}$. Partition the vertex set as $V = \{v_1\} \cup \{v_2, v_3, v_4, v_5\} \cup \{v_6, v_7\}$.

		      Suppose that eliminating $v_1$ leads to $\Delta_e = 2$. This implies that $\neighborvg{v_1}{G}$ contains no edges, so 6 fill-in edges are added, making the neighborhood a clique. Now we count the maximum possible number of edges under this configuration:
		      the 4 edges incident to $v_1$,
		      no edges inside $\{v_2, v_3, v_4, v_5\}$,
		      up to 1 edge inside $\{v_6, v_7\}$, and
		      at most $2 \cdot 4 = 8$ edges between $\{v_6, v_7\}$ and $\{v_2, v_3, v_4, v_5\}$. Hence, the total number of edges is at most $4 + 0 + 8 + 1 = 13 < 14 = e$, contradicting the assumption. Therefore, $\Delta_e = 2$ is impossible.

		      Now consider the case where eliminating $v_1$ leads to $\Delta_e = 1$. Then $\neighborvg{v_1}{G}$ must contain exactly one edge, say connecting $v_2$ and $v_3$, and the edges between $\{v_6, v_7\}$ and $\{v_2, v_3, v_4, v_5\}$ and inside $\{v_6, v_7\}$ all appear. This allows a maximum of $4 + 1 + 8 + 1 = 14$ edges. However, in this setting, $v_4$ and $v_5$ are not adjacent to each other or to $v_2$ or $v_3$, and their neighbors are only in $\{v_1, v_6, v_7\}$. Thus, each of them has degree at most $3$, contradicting the assumption that the minimum degree is $4$.

		      Combining the above arguments, we conclude that eliminating $v_1$ does not increase the number of edges. Hence, if the minimum degree is $4$, then there exists a vertex of degree $4$ such that eliminating this vertex yields a graph in $\graphsetne{n-1}{e'}$ with $e' \leq e$.

		      Now combining this with previous arguments for when the minimum degree is $1$, $2$, or $3$, we obtain the recurrence relation: $ \mathrm{tt}(n, e) \leq \max\{4, \mathrm{tt}(n - 1, e), \mathrm{tt}(n - 1, e - 1), \mathrm{tt}(n - 1, e - 2), \mathrm{tt}(n - 1, e - 3), \mathrm{tt}(n - 1, e - 4), \mathrm{t}(e - 1)\} \leq \max \{ 4, \mathrm{tt}(n - 1, e),\mathrm{t}(e - 1) \}$.

		      Since $\mathrm{tt}(n - 1, e) = \mathrm{tt}(6, 14) = 4$ by Property~4 shown earlier, we conclude: $\mathrm{tt}(7, 14) \leq \max\{4, \mathrm{t}(13)\}$.

	\end{enumerate}
\end{proof}

Armed with the above results on treewidth, we prove Observation~\ref{pro:treewidth_table}.
\begin{proof}[Proof of Observation~\ref{pro:treewidth_table}]

	We verify the values of $\mathrm{t}(e)$ for each $1 \le e \le 15$ by analyzing the sets $\graphsetne{n}{e}$ and showing that the lower and upper bounds coincide. Let $\mathrm{u}(e)$ (resp. $\mathrm{l}(e)$) denote the target upper (resp. lower) bound of $\mathrm{t}(e)$ for each $e$. Our goal is to prove that $\mathrm{t}(e) = \mathrm{l}(e) = \mathrm{u} (e)$ for all $e \in \{1, 2, \dots, 15\}$.

	\textbf{Lower bound.} For each $e$, a specific graph $G_e$ is provided in Figure~\ref{fig:treewidth_table}. Each $G_e$ is a simple graph with $e$ edges and contains a clique $\completegraphp{\mathrm{l}(e)+1}$ as a subgraph. By the definition of degeneracy and Lemma~\ref{lem:simple_bounds_on_treewidth}, we have:
	\begin{equation*}
		\mathrm{l}(e) \leq \degeneracyg{G_e} \leq \twp{G_e} \leq \mathrm{t}(e).
	\end{equation*}

	\textbf{Upper bound.} We now show that $\mathrm{t}(e) \leq \mathrm{u}(e)$ by bounding $\mathrm{tt}(n,e)$ for all $n \in \mathrm{n}(e)$, using the properties in Lemma~\ref{lem:property_of_tt}.

	From Property~4, we obtain:
	\begin{align*}
		\mathrm{tt}(2,1) & = 1, \mathrm{tt}(3,3) = 2, \mathrm{tt}(4,6) = 3, \\ \mathrm{tt}(5,10) & = 4, \mathrm{tt}(6,15) = 5, \mathrm{tt}(5,9) = 3, \mathrm{tt}(6,14) = 4.
	\end{align*}

	\textbf{Case $e = 1, 2$:}

	We have $\mathrm{n}(1) = \{2\}$ and $\mathrm{n}(2) = \{3, 4\}$, so:
	\begin{equation*}
		\mathrm{t}(1) = \mathrm{tt}(2,1) = 1.
	\end{equation*}

	For $e = 2$, both $n = 3$ and $n = 4$ satisfy $1 \leq \frac{2e}{n} < 2$. By Property~5:
	\begin{align*}
		\mathrm{tt}(3,2) & \leq \max\{1, \mathrm{t}(1)\} = 1, \\
		\mathrm{tt}(4,2) & \leq \max\{1, \mathrm{t}(1)\} = 1,
	\end{align*}
	and hence,
	\begin{equation*}
		\mathrm{t}(2) = 1.
	\end{equation*}

	\textbf{Case $e = 3, 4, 5$:}

	The support sets are:
	\begin{align*}
		\mathrm{n}(3) & = \{3, 4, 5, 6\},      \\
		\mathrm{n}(4) & = \{4, 5, 6, 7, 8\},   \\
		\mathrm{n}(5) & = \{4, 5, \dots, 10\}.
	\end{align*}
	Observe that
	\begin{equation*}
		\mathrm{n}(3) \setminus \{3\} \subseteq \mathrm{n}(4) \subseteq \mathrm{n}(5).
	\end{equation*}

	By monotonicity (Part 3 in Lemma~\ref{lem:property_of_tt}), we have:
	\begin{equation*}
		\max_{n \in \mathrm{n}(3) \setminus \{3\}} \mathrm{tt}(n,3) \leq \mathrm{t}(4) \leq \mathrm{t}(5).
	\end{equation*}

	Since $\mathrm{tt}(3,3) = 2$ and the lower bound is $2$, it follows that:
	\begin{equation*}
		2 \leq \mathrm{t}(3) \leq \mathrm{t}(4) \leq \mathrm{t}(5).
	\end{equation*}

	Now, for each $e \in \{3, 4, 5\}$ and all $n \in \mathrm{n}(e)$, we have $1 \leq \frac{2e}{n} < 3$. By Property~5:
	\begin{equation*}
		\mathrm{t}(5) \leq \max\{2, \mathrm{t}(4)\} \leq \max\{2, \mathrm{t}(3)\} \leq \max\{2, \mathrm{t}(2)\} = 2,
	\end{equation*}
	and thus:
	\begin{equation*}
		\mathrm{t}(3) = \mathrm{t}(4) = \mathrm{t}(5) = 2.
	\end{equation*}

	\textbf{Case $e = 6, 7, 8, 9$:}

	The support sets are:
	\begin{align*}
		\mathrm{n}(6) & = \{4, 5, \dots, 12\}, \\
		\mathrm{n}(7) & = \{5, 6, \dots, 14\}, \\
		\mathrm{n}(8) & = \{5, 6, \dots, 16\}, \\
		\mathrm{n}(9) & = \{5, 6, \dots, 18\}.
	\end{align*}

	We see that $\mathrm{tt}(4,6) = 3$, and that
	\begin{equation*}
		\mathrm{n}(6) \setminus \{4\} \subseteq \mathrm{n}(7) \subseteq \mathrm{n}(8) \subseteq \mathrm{n}(9),
	\end{equation*}
	the lower bound is $3$. Therefore, by monotonicity:
	\begin{equation*}
		3 \leq \mathrm{t}(6) \leq \mathrm{t}(7) \leq \mathrm{t}(8) \leq \mathrm{t}(9).
	\end{equation*}

	More specifically, note that
	\begin{equation*}
		\mathrm{tt}(5,6) \leq \mathrm{tt}(5,7) \leq \mathrm{tt}(5,8) \leq \mathrm{tt}(5,9) = 3.
	\end{equation*}

	For all $e \in \{6, 7, 8, 9\}$ and $n \in \mathrm{n}(e) \setminus \{4, 5\}$, we have $1 \leq \frac{2e}{n} < 4$, so by Property~5:
	\begin{equation*}
		\mathrm{tt}(n, e) \leq \max\{3, \mathrm{t}(e - 1), \mathrm{tt}(5, e)\} \leq \max\{3, \mathrm{t}(e - 1)\}.
	\end{equation*}

	Hence,
	\begin{equation*}
		\mathrm{t}(9) \leq \max\{3, \mathrm{t}(8)\} \leq \max\{3, \mathrm{t}(7)\} \leq \max\{3, \mathrm{t}(6)\} \leq \max\{3, \mathrm{t}(5)\} = 3,
	\end{equation*}
	and therefore,
	\begin{equation*}
		\mathrm{t}(6) = \mathrm{t}(7) = \mathrm{t}(8) = \mathrm{t}(9) = 3.
	\end{equation*}

	\textbf{Case $e = 10, 11, 12, 13, 14$:}

	The support sets are:
	\begin{align*}
		\mathrm{n}(10) & = \{5, 6, \dots, 20\},  \\
		\mathrm{n}(11) & = \{6, 7, \dots, 22\},  \\
		\mathrm{n}(12) & = \{6, 7, \dots, 24\},  \\
		\mathrm{n}(13) & = \{6, 7, \dots, 26\},  \\
		\mathrm{n}(14) & = \{6, 7,  \dots, 28\}.
	\end{align*}

	We know that $\mathrm{tt}(5,10) = 4$, the lower bound is $4$, and
	\begin{equation*}
		\mathrm{n}(10) \setminus \{5\} \subseteq \mathrm{n}(11) \subseteq \mathrm{n}(12) \subseteq \mathrm{n}(13) \subseteq \mathrm{n}(14),
	\end{equation*}

	Thus, by monotonicity:
	\begin{equation*}
		4 \leq \mathrm{t}(10) \leq \mathrm{t}(11) \leq \mathrm{t}(12) \leq \mathrm{t}(13) \leq \mathrm{t}(14).
	\end{equation*}

	From Property~5 of Lemma~\ref{lem:property_of_tt} (Case (4)), we obtain:
	\begin{equation*}
		\mathrm{tt}(7,14) \leq \max\{4, \mathrm{t}(13)\} .
	\end{equation*}

	And by monotonicity:
	\begin{equation*}
		\mathrm{tt}(6,10) \leq \mathrm{tt}(6,11) \leq \mathrm{tt}(6,12) \leq \mathrm{tt}(6,13) \leq \mathrm{tt}(6,14) = 4,
	\end{equation*}

	For $e = 14$ and $n \in \mathrm{n}(14) \setminus \{5, 6, 7\}$, we have $1 \leq \frac{2e}{n} < 4$, so:
	\begin{equation*}
		\mathrm{tt}(n,14) \leq \max\{3, \mathrm{tt}(7,14), \mathrm{t}(13)\} \leq \max\{4, \mathrm{t}(13)\},
	\end{equation*}
	and hence $\mathrm{t}(14) \leq \mathrm{t}(13)$.

	For $e \in \{10,11,12,13\}$ and $n \in \mathrm{n}(e) \setminus \{5,6\}$, Property~5 gives:
	\begin{equation*}
		\mathrm{tt}(n,e) \leq \max\{3, \mathrm{t}(e-1), \mathrm{tt}(6,e)\} \leq \max\{4, \mathrm{t}(e-1)\}.
	\end{equation*}

	Thus, we obtain:
	\begin{equation*}
		\mathrm{t}(14) \leq \mathrm{t}(13) \leq \max\{4, \mathrm{t}(12)\} \leq \max\{4, \mathrm{t}(11)\} \leq \max\{4, \mathrm{t}(10)\} \leq \max\{4, \mathrm{t}(9)\} = 4,
	\end{equation*}
	and therefore,
	\begin{equation*}
		\mathrm{t}(10) = \mathrm{t}(11) = \mathrm{t}(12) = \mathrm{t}(13) = \mathrm{t}(14) = 4.
	\end{equation*}

	\textbf{Case $e = 15$:}

	The support set is $\mathrm{n}(15) = \{6, 7, \dots, 30\}$.

	We have $\mathrm{tt}(6,15) = 5$, and $\mathrm{tt}(7,15) \leq \mathrm{tt}(7,20) = 5$. For all other $n \in \mathrm{n}(15) \setminus \{6, 7\}$, since $1 \leq \frac{2e}{n} < 4$, they are bounded by $ \max\{3, \mathrm{tt}(7,15), \mathrm{t}(14)\} = 5$, Thus we conclude $ \mathrm{t}(15) = 5$.
\end{proof}

\begin{figure}[htbp]
	\centering
	\begin{subfigure}{0.18\textwidth}
		\centering
		\begin{tikzpicture}[scale=0.8]
			\node[circle, draw] (a) at (0,0) {};
			\node[circle, draw] (b) at (1,0) {};
			\draw (a) -- (b);
		\end{tikzpicture}
		\caption{$G_{1}(\completegraphp{2})$ \\$e=1$, $\mathrm{t}(e)=1$}
		\label{fig:graph_e_1}
	\end{subfigure}
	\hfill
	\begin{subfigure}{0.18\textwidth}
		\centering
		\begin{tikzpicture}[scale=0.8]
			\node[circle, draw] (a) at (0,0) {};
			\node[circle, draw] (b) at (1,0) {};
			\node[circle, draw] (c) at (0.5,1) {};
			\draw (a) -- (b) -- (c);
		\end{tikzpicture}
		\caption{$G_{2}$ \\$e=2$, $\mathrm{t}(e)=1$}
		\label{fig:graph_e_2}
	\end{subfigure}
	\hfill
	\begin{subfigure}{0.18\textwidth}
		\centering
		\begin{tikzpicture}[scale=0.8]
			\node[circle, draw] (a) at (0,0) {};
			\node[circle, draw] (b) at (1,0) {};
			\node[circle, draw] (c) at (0.5,1) {};
			\draw (a) -- (b) -- (c) -- (a);
		\end{tikzpicture}
		\caption{ $G_{3}(\completegraphp{3})$\\$e=3$, $\mathrm{t}(e)=2$}
		\label{fig:graph_e_3}
	\end{subfigure}
	\hfill
	\begin{subfigure}{0.18\textwidth}
		\centering
		\begin{tikzpicture}[scale=0.8]
			\node[circle, draw] (a) at (0,0) {};
			\node[circle, draw] (b) at (1,0) {};
			\node[circle, draw] (c) at (1,1) {};
			\node[circle, draw] (d) at (0,1) {};
			\draw (a) -- (b);
			\draw (b) -- (d);
			\draw (c) -- (d);
			\draw (a) -- (d);
		\end{tikzpicture}
		\caption{$G_{4}$\\$e=4$, $\mathrm{t}(e)=2$}
		\label{fig:graph_e_4}
	\end{subfigure}
	\hfill
	\begin{subfigure}{0.18\textwidth}
		\centering
		\begin{tikzpicture}[scale=0.8]
			\node[circle, draw] (a) at (0,0) {};
			\node[circle, draw] (b) at (1,0) {};
			\node[circle, draw] (c) at (1,1) {};
			\node[circle, draw] (d) at (0,1) {};
			\node[circle, draw] (e) at (2,1) {};
			\draw (a) -- (b);
			\draw (b) -- (d);
			\draw (c) -- (d);
			\draw (a) -- (d);
			\draw (c) -- (e);
		\end{tikzpicture}
		\caption{$G_{5}$\\$e=5$, $\mathrm{t}(e)=1$}
		\label{fig:graph_e_5}
	\end{subfigure}

	\vspace{0.5cm}

	\begin{subfigure}{0.18\textwidth}
		\centering
		\begin{tikzpicture}[scale=0.8]
			\node[circle, draw] (a) at (0,0) {};
			\node[circle, draw] (b) at (1,0) {};
			\node[circle, draw] (c) at (1,1) {};
			\node[circle, draw] (d) at (0,1) {};
			\draw (a) -- (b) -- (c) -- (d) -- (a);
			\draw (a) -- (c);
			\draw (b) -- (d);
		\end{tikzpicture}
		\caption{ $G_{6}(\completegraphp{4})$\\$e=6$, $\mathrm{t}(e)=3$}
		\label{fig:graph_e_6}
	\end{subfigure}
	\hfill
	\begin{subfigure}{0.18\textwidth}
		\centering
		\begin{tikzpicture}[scale=0.8]
			\node[circle, draw] (a) at (0,0) {};
			\node[circle, draw] (b) at (1,0) {};
			\node[circle, draw] (c) at (1,1) {};
			\node[circle, draw] (d) at (0,1) {};
			\node[circle, draw] (e) at (2,1) {};
			\draw (a) -- (b) -- (c) -- (d) -- (a);
			\draw (a) -- (c);
			\draw (b) -- (d);
			\draw (c) -- (e);
		\end{tikzpicture}
		\caption{$G_{7}$\\$e=7$, $\mathrm{t}(e)=3$}
		\label{fig:graph_e_7}
	\end{subfigure}
	\hfill
	\begin{subfigure}{0.18\textwidth}
		\centering
		\begin{tikzpicture}[scale=0.8]

			\node[circle, draw] (a) at (0,0) {};
			\node[circle, draw] (b) at (1,0) {};
			\node[circle, draw] (c) at (1,1) {};
			\node[circle, draw] (d) at (0,1) {};

			\node[circle, draw] (e) at (2,1) {};
			\node[circle, draw] (f) at (2,0) {};

			\draw (a) -- (b);
			\draw (a) -- (c);
			\draw (a) -- (d);
			\draw (b) -- (c);
			\draw (b) -- (d);
			\draw (c) -- (d);

			\draw (c) -- (e);
			\draw (e) -- (f);
		\end{tikzpicture}
		\caption{$G_{8}$\\$e=8$, $\mathrm{t}(e)=4$}
		\label{fig:graph_e_8}
	\end{subfigure}
	\hfill
	\begin{subfigure}{0.18\textwidth}
		\centering
		\begin{tikzpicture}[scale=0.8]
			\node[circle, draw] (a) at (0,0) {};
			\node[circle, draw] (b) at (1,0) {};
			\node[circle, draw] (c) at (2,0) {};
			\node[circle, draw] (d) at (0,1) {};
			\node[circle, draw] (e) at (1,1) {};
			\node[circle, draw] (h) at (2,1) {};
			\node[circle, draw] (g) at (1.5,0.5) {};
			\draw (a) -- (b) -- (d) -- (e);
			\draw (c) -- (h) -- (e);
			\draw (a) -- (d);
			\draw (b) -- (e);
			\draw (a) -- (e);
			\draw (c) -- (g);
		\end{tikzpicture}
		\caption{$G_{9}$\\$e=9$, $\mathrm{t}(e)=3$}
		\label{fig:graph_e_9}
	\end{subfigure}
	\hfill
	\begin{subfigure}{0.18\textwidth}
		\centering
		\begin{tikzpicture}[scale=0.8]
			\node[circle, draw] (a) at (90:1) {};
			\node[circle, draw] (b) at (162:1) {};
			\node[circle, draw] (c) at (234:1) {};
			\node[circle, draw] (d) at (306:1) {};
			\node[circle, draw] (e) at (18:1) {};
			\draw (a) -- (b) -- (c) -- (d) -- (e) -- (a);
			\draw (a) -- (c);
			\draw (a) -- (d);
			\draw (b) -- (d);
			\draw (b) -- (e);
			\draw (c) -- (e);
		\end{tikzpicture}
		\caption{ $G_{10}(\completegraphp{5})$\\$e=10$, $\mathrm{t}(e)=4$}
		\label{fig:graph_e_10}
	\end{subfigure}

	\vspace{0.5cm}

	\begin{subfigure}{0.18\textwidth}
		\centering
		\begin{tikzpicture}[scale=0.6]
			\node[circle, draw] (v1) at (90:1) {};
			\node[circle, draw] (v2) at (18:1) {};
			\node[circle, draw] (v3) at (306:1) {};
			\node[circle, draw] (v4) at (234:1) {};
			\node[circle, draw] (v5) at (162:1) {};

			\draw (v1) -- (v2) -- (v3) -- (v4) -- (v5) -- (v1);
			\draw (v1) -- (v3) -- (v5) -- (v2) -- (v4) -- (v1);

			\node[circle, draw] (v6) at (90:1.8) {};
			\draw (v6) -- (v1);
		\end{tikzpicture}
		\caption{ $G_{11}$\\$e=11$, $\mathrm{t}(e)=4$}
		\label{fig:graph_e_11}
	\end{subfigure}
	\hfill
	\begin{subfigure}{0.18\textwidth}
		\centering
		\begin{tikzpicture}[scale=0.6]
			\node[circle, draw] (v1) at (90:1) {};
			\node[circle, draw] (v2) at (18:1) {};
			\node[circle, draw] (v3) at (306:1) {};
			\node[circle, draw] (v4) at (234:1) {};
			\node[circle, draw] (v5) at (162:1) {};

			\draw (v1) -- (v2) -- (v3) -- (v4) -- (v5) -- (v1);
			\draw (v1) -- (v3) -- (v5) -- (v2) -- (v4) -- (v1);

			\node[circle, draw] (v6) at (90:1.8) {};
			\draw (v6) -- (v1);
			\draw (v6) -- (v2);
		\end{tikzpicture}
		\caption{ $G_{12}$\\$e=12$, $\mathrm{t}(e)=4$}
		\label{fig:graph_e_12}
	\end{subfigure}
	\hfill
	\begin{subfigure}{0.18\textwidth}
		\centering
		\begin{tikzpicture}[scale=0.6]
			\node[circle, draw] (v1) at (90:1) {};
			\node[circle, draw] (v2) at (18:1) {};
			\node[circle, draw] (v3) at (306:1) {};
			\node[circle, draw] (v4) at (234:1) {};
			\node[circle, draw] (v5) at (162:1) {};

			\draw (v1) -- (v2) -- (v3) -- (v4) -- (v5) -- (v1);
			\draw (v1) -- (v3) -- (v5) -- (v2) -- (v4) -- (v1);

			\node[circle, draw] (v6) at (90:1.8) {};
			\node[circle, draw] (v7) at (50:1.8) {};
			\draw (v6) -- (v1);
			\draw (v7) -- (v1);
			\draw (v6) -- (v7);
		\end{tikzpicture}
		\caption{ $G_{13}$\\$e=13$, $\mathrm{t}(e)=4$}
	\end{subfigure}
	\hfill
	\begin{subfigure}{0.18\textwidth}
		\centering
		\begin{tikzpicture}[scale=0.6]
			\node[circle, draw] (v1) at (90:1) {};
			\node[circle, draw] (v2) at (18:1) {};
			\node[circle, draw] (v3) at (306:1) {};
			\node[circle, draw] (v4) at (234:1) {};
			\node[circle, draw] (v5) at (162:1) {};

			\draw (v1) -- (v2) -- (v3) -- (v4) -- (v5) -- (v1);
			\draw (v1) -- (v3) -- (v5) -- (v2) -- (v4) -- (v1);

			\node[circle, draw] (v6) at (90:1.8) {};
			\node[circle, draw] (v7) at (50:1.8) {};
			\node[circle, draw] (v8) at (130:1.8) {};
			\draw (v7) -- (v1);
			\draw (v8) -- (v1);
			\draw (v6) -- (v7);
			\draw (v8) -- (v6);
		\end{tikzpicture}
		\caption{ $G_{14}$\\$e=14$, $\mathrm{t}(e)=4$}
	\end{subfigure}
	\hfill
	\begin{subfigure}{0.18\textwidth}
		\centering
		\begin{tikzpicture}[scale=0.6]
			\node[circle, draw] (v1) at (90:1) {};
			\node[circle, draw] (v2) at (30:1) {};
			\node[circle, draw] (v3) at (330:1) {};
			\node[circle, draw] (v4) at (270:1) {};
			\node[circle, draw] (v5) at (210:1) {};
			\node[circle, draw] (v6) at (150:1) {};

			\draw (v1) -- (v2);
			\draw (v1) -- (v3);
			\draw (v1) -- (v4);
			\draw (v1) -- (v5);
			\draw (v1) -- (v6);
			\draw (v2) -- (v3);
			\draw (v2) -- (v4);
			\draw (v2) -- (v5);
			\draw (v2) -- (v6);
			\draw (v3) -- (v4);
			\draw (v3) -- (v5);
			\draw (v3) -- (v6);
			\draw (v4) -- (v5);
			\draw (v4) -- (v6);
			\draw (v5) -- (v6);
		\end{tikzpicture}
		\caption{ $G_{15}$ ($\completegraphp{6}$)\\$e=15$, $\mathrm{t}(e)=5$}
	\end{subfigure}

	\caption{Graphs $G_1$ to $G_{15}$ with increasing edges and maximum treewidth.}
	\label{fig:treewidth_table}
\end{figure}

\section{Supplementary Results of Section~\ref{sec:example_hoif}}
\label{app:HOIF}

The following table Table~\ref{tab:HOIF} lists the total number of $V$-statistics and the number of floating-point operations (FLOPs) that need to be computed before and after applying the sparsification trick when the sample size is $10^4$. The rightmost column of Table~\ref{tab:HOIF} also records the value of $M$ defined in Corollary~\ref{cor:U}.

\begin{table}[htbp]
	\centering
	\caption{Exact number of floating-point operations (FLOPs), $V$-statistic decompositions, intermediate tensor sizes,
		and the maximum treewidth $M$ across all decomposition graphs
		for HOIF orders from $2$ to $12$, with and without the sparsification trick, where the \Einsum{} ordering is found by applying \opteinsum{} with optimization scheme \texttt{greedy}. Sample size is fixed at $n=10000$.}
	\label{tab:HOIF}
	\begin{tabular}{cccccc}
		\toprule
		\textbf{Order ($m$)} & \makecell{\textbf{V-stat terms} \\ \textbf{(Bell number)}} & \makecell{\textbf{V-stat terms} \\ \textbf{(sparsification)}} & \makecell{\textbf{FLOPs} \\ \textbf{(initial)}} & \makecell{\textbf{FLOPs} \\ \textbf{(sparsification)}} & \makecell{$M$} \\
		\midrule
		2                    & 2                                                          & 1                                                             & $2.00020\times10^8$                             & $2.00000\times10^8$                                    & $1$            \\
		3                    & 5                                                          & 2                                                             & $2.00060\times10^{12}$                          & $2.00020\times10^{12}$                                 & $1$            \\
		4                    & 15                                                         & 5                                                             & $8.00310\times10^{12}$                          & $4.00150\times10^{12}$                                 & $2$            \\
		5                    & 52                                                         & 15                                                            & $5.20133\times10^{13}$                          & $2.60045\times10^{13}$                                 & $2$            \\
		6                    & 203                                                        & 52                                                            & $2.94069\times10^{14}$                          & $1.20020\times10^{14}$                                 & $2$            \\
		7                    & 877                                                        & 203                                                           & $1.78236\times10^{15}$                          & $6.32091\times10^{14}$                                 & $2$            \\
		8                    & 4140                                                       & 877                                                           & $2.31054\times10^{17}$                          & $2.23430\times10^{17}$                                 & $3$            \\
		9                    & 21147                                                      & 4140                                                          & $5.29126\times10^{18}$                          & $3.47959\times10^{18}$                                 & $3$            \\
		10                   & 115975                                                     & 21147                                                         & $2.07990\times10^{21}$                          & $2.04048\times10^{21}$                                 & $3$            \\
		11                   & 678570                                                     & 115975                                                        & $5.50026\times10^{22}$                          & $3.72144\times10^{22}$                                 & $4$            \\
		12                   & 4213597                                                    & 678570                                                        & $1.13401\times10^{24}$                          & $6.62181\times10^{23}$                                 & $4$            \\
		\bottomrule
	\end{tabular}
\end{table}

We further report the empirical results in Table~\ref{tab:HOIF_Nl}, comparing two
\opteinsum{} strategies for choosing the summation ordering: the strategy that finds
the optimal ordering, \texttt{dp} (dynamic programming), and the strategy that finds a
suboptimal ordering, \texttt{greedy}, discussed in Remark~\ref{rem:ordering}. The table
records three things at once: (i)~the ordering search time \texttt{searching time} versus the
total compute time \texttt{total time} and its floating-point
operation count \texttt{FLOPs}, at different sample sizes; (ii)~the estimated counts
$N_l$ of Corollary~\ref{cor:U}, i.e.\ the number of $V$-statistics---after the
sparsification trick---whose \Einsum{} costs $O(n^{l+1})$ under the chosen ordering; and
(iii)~the measured peak memory \texttt{mem}. All values are averaged over $3$ repeats.

For finding the ordering, \texttt{greedy} is consistently much faster than \texttt{dp}, since it avoids the exhaustive enumeration, and the gap becomes larger with the order $m$ as the decomposition graph grows. Only for small $m$, the two different ordering times, (the \texttt{searching time} column in Table~\ref{tab:HOIF_Nl}) are comparable. The total computing time (the \texttt{total time} columns in Table~\ref{tab:HOIF_Nl}), in contrast, reflects the trade-off between searching for an ordering and computing based on that ordering: at $n = 1000$ the two strategies have comparable runtime, whereas at $n=10000$ \texttt{dp} is uniformly faster because it gives better orderings, with fewer \texttt{FLOPs} as reported in the \texttt{FLOPs} column. This is exactly the trade-off anticipated in Remark~\ref{rem:ordering}.

The counts $N_l$ can be read off from the chosen ordering: \texttt{greedy} gives an upper bound, whereas \texttt{dp}, being optimal, gives the exact values. For each order, \texttt{dp} shifts the counts toward smaller $l$, reducing the number of high-cost $V$-statistics relative to \texttt{greedy}, which is precisely the effect discussed in
Remark~\ref{rem:U-complexity}.

Finally, the \texttt{mem} column confirms the analysis of
Remark~\ref{rem:space_complexity} and Appendix~\ref{app:space_complexity}. 
The tensorization step stores the data matrices at a cost $O(n^2)$, and the \Einsum{} step adds the largest intermediate tensor, whose order is set by the maximal treewidth of the chosen ordering. For $m=4$ to $7$ both strategies reach treewidth $2$, so the memory stays at the $O(n^2)$ order: about $0.03$--$0.05$\,GB at $n=1000$ and $2.3$--$4.0$\,GB at
$n=10000$. For $m=8,9$ the treewidth rises to $3$, giving $O(n^3)$ memory of about
$12$\,GB at $n=1000$, while $n=10000$ exceeds the $192$\,GB limit. At $m=10$ the two
strategies diverge: \texttt{dp} stays at treewidth $3$ and still runs at $n=1000$
($15.9$\,GB), whereas \texttt{greedy} reaches treewidth $4$, so even at $n=1000$ its $O(n^4)$ memory is over the limit.

\begin{table}[htbp]
	\centering
	\caption{A comparison of different strategies of searching for summation ordering for HOIF-type $U$-statistics of orders $m=4,\dots,10$, run on Intel
		Xeon Cascade Lake 6248 CPUs ($2.5$\,GHz, $40$ cores) with $192$\,GB of memory. For each
		order we report the two most practical \opteinsum{} strategies for choosing the
		summation ordering, \texttt{greedy} and \texttt{dp}. \texttt{searching time} is the mean time
		to find the summation orderings of all induced \Einsum{}, and the counts $N_l$ record
		the number of $V$-statistics whose \Einsum{} has $O(n^{l+1})$ runtime under the chosen
		orderings; both are independent of the sample size $n$. For each sample size
		$n\in\{1000,10000\}$ we report \texttt{FLOPs} (total floating-point operations of the
		chosen orderings), \texttt{total time} (mean time to compute the $U$-statistic), and
		\texttt{mem} (measured peak memory increase during the computation). ``---'' marks a
		run killed for exceeding the memory limit.}
	\label{tab:HOIF_Nl}
	\scriptsize
	\setlength{\tabcolsep}{3pt}
	\begin{tabular}{clrrrrrrrrrrr}
		\toprule
		    &                           &                                              & \multicolumn{4}{c}{$N_l$ counts} & \multicolumn{3}{c}{$n=1000$} & \multicolumn{3}{c}{$n=10000$}                                                                                                                 \\
		\cmidrule(lr){4-7}\cmidrule(lr){8-10}\cmidrule(lr){11-13}
		$m$ & Opt.                      & \makecell{\makecell{searching \\ time}\\(s)}
		    & \makecell{$N_1$\\($n^2$)} & \makecell{$N_2$\\($n^3$)}                    & \makecell{$N_3$\\($n^4$)}        & \makecell{$N_4$\\($n^5$)}
		    & FLOPs                     & \makecell{\makecell{total \\ time}\\(s)}     & \makecell{mem\\(GB)}
		    & FLOPs                     & \makecell{\makecell{total \\ time}\\(s)}     & \makecell{mem\\(GB)}                                                                                                                                                                                            \\
		\midrule
		\multirow{2}{*}{$4$}
		    & \texttt{greedy}           & $4.14\times10^{-4}$                          & $3$                              & $2$                          & $0$                           & $0$  & $4.02\times10^{9}$  & $2.82\times10^{-2}$ & $0.03$ & $4.00\times10^{12}$ & $5.28\times10^{0}$ & $2.30$ \\
		    & \texttt{dp}               & $6.94\times10^{-4}$                          & $4$                              & $1$                          & $0$                           & $0$  & $2.02\times10^{9}$  & $2.90\times10^{-2}$ & $0.04$ & $2.00\times10^{12}$ & $4.34\times10^{0}$ & $2.30$ \\
		\addlinespace
		\multirow{2}{*}{$5$}
		    & \texttt{greedy}           & $1.32\times10^{-3}$                          & $5$                              & $10$                         & $0$                           & $0$  & $2.61\times10^{10}$ & $1.28\times10^{-1}$ & $0.05$ & $2.60\times10^{13}$ & $2.50\times10^{1}$ & $2.80$ \\
		    & \texttt{dp}               & $2.84\times10^{-3}$                          & $9$                              & $6$                          & $0$                           & $0$  & $1.41\times10^{10}$ & $1.10\times10^{-1}$ & $0.05$ & $1.40\times10^{13}$ & $2.09\times10^{1}$ & $2.70$ \\
		\addlinespace
		\multirow{2}{*}{$6$}
		    & \texttt{greedy}           & $5.64\times10^{-3}$                          & $10$                             & $42$                         & $0$                           & $0$  & $1.20\times10^{11}$ & $5.05\times10^{-1}$ & $0.05$ & $1.20\times10^{14}$ & $1.06\times10^{2}$ & $3.20$ \\
		    & \texttt{dp}               & $1.58\times10^{-2}$                          & $20$                             & $32$                         & $0$                           & $0$  & $8.43\times10^{10}$ & $4.57\times10^{-1}$ & $0.04$ & $8.40\times10^{13}$ & $9.78\times10^{1}$ & $3.20$ \\
		\addlinespace
		\multirow{2}{*}{$7$}
		    & \texttt{greedy}           & $2.56\times10^{-2}$                          & $25$                             & $178$                        & $0$                           & $0$  & $6.33\times10^{11}$ & $2.29\times10^{0}$  & $0.03$ & $6.32\times10^{14}$ & $4.88\times10^{2}$ & $4.00$ \\
		    & \texttt{dp}               & $1.02\times10^{-1}$                          & $48$                             & $155$                        & $0$                           & $0$  & $4.83\times10^{11}$ & $2.15\times10^{0}$  & $0.03$ & $4.82\times10^{14}$ & $4.82\times10^{2}$ & $3.60$ \\
		\addlinespace
		\multirow{2}{*}{$8$}
		    & \texttt{greedy}           & $1.77\times10^{-1}$                          & $56$                             & $810$                        & $11$                          & $0$  & $2.54\times10^{13}$ & $7.86\times10^{2}$  & $12.0$ & $2.23\times10^{17}$ & ---                & ---    \\
		    & \texttt{dp}               & $1.09\times10^{0}$                           & $113$                            & $753$                        & $11$                          & $0$  & $2.48\times10^{13}$ & $4.16\times10^{2}$  & $12.0$ & $2.23\times10^{17}$ & ---                & ---    \\
		\addlinespace
		\multirow{2}{*}{$9$}
		    & \texttt{greedy}           & $7.11\times10^{-1}$                          & $149$                            & $3818$                       & $173$                         & $0$  & $3.66\times10^{14}$ & $2.40\times10^{2}$  & $12.0$ & $3.48\times10^{18}$ & ---                & ---    \\
		    & \texttt{dp}               & $7.22\times10^{0}$                           & $282$                            & $3685$                       & $173$                         & $0$  & $3.63\times10^{14}$ & $2.58\times10^{2}$  & $12.1$ & $3.48\times10^{18}$ & ---                & ---    \\
		\addlinespace
		\multirow{2}{*}{$10$}
		    & \texttt{greedy}           & $4.13\times10^{0}$                           & $343$                            & $18860$                      & $1934$                        & $10$ & $2.42\times10^{16}$ & ---                 & ---    & $2.04\times10^{21}$ & ---                & ---    \\
		    & \texttt{dp}               & $7.41\times10^{1}$                           & $689$                            & $18514$                      & $1944$                        & $0$  & $4.16\times10^{15}$ & $2.53\times10^{3}$  & $15.9$ & $4.06\times10^{19}$ & ---                & ---    \\
		\bottomrule
	\end{tabular}
\end{table}

We can derive the sharp treewidth bounds presented in Table~\ref{tab:HOIF} using similar techniques to the above in Appendix~\ref{app:treewidtg_table}. For each $m \in \{2, 3, \cdots, 12\}$, the corresponding decomposition signature takes the following form: $\calA_{m} = ((1, 2), \cdots, (m - 1, m))$. Hence by Definition~\ref{def:graph_decomp}, the corresponding decomposition graph $G_{\calA_{m}}$ is the graph with $V(G_{\calA_{m}}) = [m]$, $E(G_{\calA_{m}}) = \{ \{i, i+1\} \}_{i=1}^{m-1}$, whose structure is illustrated in Figure~\ref{fig:hoif_graph_m}.

The result is formally stated in the following observation.
\begin{observation}
	\label{obs:HOIF}
	For $m = 2,3,\cdots, 12$, let $\calA_{m} = ((1,2),\cdots,(m-1,m))$, and $G_{\calA_m}$ be the corresponding decomposition graph. Similar to Corollary~\ref{cor:U}, $M(m)$ is defined as
	\begin{align*}
		M(m) \coloneqq \max \{ \twp{G_{\calA_m} /\pi} \mid \pi \in \partitionp{m}^{\calA_m}\}.
	\end{align*}
	Then the following holds:
	\begin{equation}
		\label{HOIF_M}
		\begin{split}
			M (m) = \left\{ \begin{array}{ll}
				                1, & m \in \{2, 3\},       \\
				                2, & m \in \{4, 5, 6, 7\}, \\
				                3, & m \in \{8, 9, 10\},   \\
				                4, & m \in \{11, 12\}.
			                \end{array} \right.
		\end{split}
	\end{equation}
\end{observation}

\begin{proof}
	Let $\graphset^{\mathsf{HOIF}}_{m} \coloneqq \{ G_{\calA_m} /\pi \mid \pi \in \partitionp{m}^{\calA_m} \}$ then $M(m) \equiv \max \{ \twp{G} \mid G \in \graphset^{\mathsf{HOIF}}_{m} \}$ and it holds that
	\begin{equation*}
		| E(G) | \leq m-1, \; \forall \, G \in \graphset^{\mathsf{HOIF}}_{m}.
	\end{equation*}
	A direct application of Observation~\ref{pro:treewidth_table} gives us the following upper bounds: $M(m) \le \mathrm{t}(m-1)$, summarized in Table~\ref{tab:HOIF_upper_bound}. The desired result then follows from a characterization of the lower bound of $M (m)$.

	\begin{table}[ht]
		\centering
		\caption{Upper bounds of $M(m)$ for HOIFs: via $M(m) \leq \mathrm{t}(m-1)$. }
		\label{tab:HOIF_upper_bound}
		\begin{tabular}{ccc}
			\toprule
			$m$         & $e$        & $\mathrm{t}(e)$ \\
			\midrule
			2, 3        & 1          & 1               \\
			4, 5, 6     & 3, 4, 5    & 2               \\
			7, 8, 9, 10 & 6, 7, 8, 9 & 3               \\
			11, 12      & 10, 11     & 4               \\
			\bottomrule
		\end{tabular}
	\end{table}

	We observe that except for $m=7$, the upper bounds match the claimed results in \eqref{HOIF_M}. The lower bounds of $M(m)$ can be obtained by enlisting some concrete examples, which are given in Table~\ref{tab:HOIF_lower_bound}.

	\begin{table}[h]
		\centering
		\caption{Lowers bound of $M(m)$ for HOIFs: via finding the exemplified graph.}
		\label{tab:HOIF_lower_bound}
		\begin{tabular}{c c c c}
			\toprule
			$m$ & $\pi$                                                             & Graph $\graphform{\calA_{m}} /\pi$              & $\twp{\graphform{\calA_{m}} /\pi}$ \\
			\midrule
			2   & \{\{1\}, \{2\}\}                                                  & Figure~\ref{fig:graph_e_1}: $G_1$               & 1                                  \\
			3   & \{\{1\}, \{2\}, \{3\}\}                                           & Figure~\ref{fig:graph_e_2}: $G_2$               & 1                                  \\
			4   & \{\{1, 4\}, \{2\}, \{3\}\}                                        & Figure~\ref{fig:graph_e_3}: $G_3$               & 2                                  \\
			5   & \{\{1, 4\}, \{2\}, \{3\}, \{5\}\}                                 & Figure~\ref{fig:graph_e_4}: $G_4$               & 2                                  \\
			6   & \{\{1, 4\}, \{2\}, \{3\}, \{5\}, \{6\}\}                          & Figure~\ref{fig:graph_e_5}: $G_5$               & 2                                  \\
			7   & \{\{1, 4\}, \{2\}, \{3\}, \{5\}, \{6\}, \{7\}\}                   & Figure~\ref{fig:graph_e_6_hoif}: $G_6^{\prime}$ & 2                                  \\
			8   & \{\{1, 4\}, \{2, 6\}, \{3, 7\}, \{5, 8\}\}                        & Figure~\ref{fig:graph_e_7}: $G_7$               & 3                                  \\
			9   & \{\{1, 4\}, \{2, 6\}, \{3, 7\}, \{5, 8\}, \{9\}\}                 & Figure~\ref{fig:graph_e_8}: $G_8$               & 3                                  \\
			10  & \{\{1, 4\}, \{2, 6\}, \{3, 7\}, \{5, 8\}, \{9\}, \{10\}\}         & Figure~\ref{fig:graph_e_9}: $G_9$               & 3                                  \\
			11  & \{\{1, 6, 11\}, \{2, 9\}, \{3, 7\}, \{4, 10\}, \{5, 8\}\}         & Figure~\ref{fig:graph_e_10}: $G_{10}$           & 4                                  \\
			12  & \{\{1, 6, 11\}, \{2, 9\}, \{3, 7\}, \{4, 10\}, \{5, 8\}, \{12\}\} & Figure~\ref{fig:graph_e_11}: $G_{11}$           & 4                                  \\
			\bottomrule
		\end{tabular}
	\end{table}

	\begin{figure}[h!]
		\centering
		\begin{subfigure}{0.45\textwidth}
			\centering
			\begin{tikzpicture}[scale=0.8]
				\node[circle, draw] (p1) at (0,0) {};
				\node[circle, draw] (p2) at (1,0) {};
				\node[circle, draw] (p3) at (2,0) {};
				\node[circle, draw] (p4) at (3,0) {};
				\draw (p1) -- (p2) -- (p3);
				\draw (p3) -- (p4);
			\end{tikzpicture}
			\caption{$\graphform{\calA_m}$ with $m=4$.}
			\label{fig:hoif_graph_m}
		\end{subfigure}
		\hfill
		\begin{subfigure}{0.45\textwidth}
			\centering
			\begin{tikzpicture}[scale=0.8]
				\node[circle, draw] (a) at (0,0) {};
				\node[circle, draw] (b) at (1,0) {};
				\node[circle, draw] (c) at (1,1) {};
				\node[circle, draw] (d) at (0,1) {};
				\node[circle, draw] (e) at (2,1) {};
				\node[circle, draw] (f) at (2,0) {};
				\draw (a) -- (b);
				\draw (b) -- (d);
				\draw (c) -- (d);
				\draw (a) -- (d);
				\draw (c) -- (e);
				\draw (e) -- (f);
			\end{tikzpicture}
			\caption{$G^{\prime}_6:$ the example in Table~\ref{tab:HOIF_lower_bound} when $m=7$.}
			\label{fig:graph_e_6_hoif}
		\end{subfigure}
		\caption{Two decomposition graphs used in Appendix~\ref{app:HOIF}.}
	\end{figure}

	Combining Table~\ref{tab:HOIF_upper_bound} and Table~\ref{tab:HOIF_lower_bound}, we can conclude that, except for $m = 7$, the upper bound (given in Table~\ref{tab:HOIF_upper_bound}) and the lower bound (given in Table~\ref{tab:HOIF_lower_bound}) of $M (m)$ agree. For $m = 7$, we have $M (7) \in \{2, 3\}$. In the sequel, we will show that $M (7) < 3$ so the claim \eqref{HOIF_M} in Observation~\ref{obs:HOIF} holds.

	To this end, we will first establish a lemma (Lemma~\ref{lem:HOIF_table_lem_1}) stating that a simple graph with at most $6$ edges has treewidth $3$ if and only if it is $\completegraphp{4}$ (i.e. the complete graph with $4$ vertices). Then, another lemma (Lemma~\ref{lem:HOIF_table_lem_2}) will show that the set of all quotient graphs of $\graphform{\calA_7}$ does not contain $\completegraphp{4}$. The proof is then complete.
\end{proof}

We first state and prove Lemma~\ref{lem:HOIF_table_lem_1}.
\begin{lemma}
	\label{lem:HOIF_table_lem_1}
	Recall the definitions of $\mathrm{tt}(n, e)$ and $\graphsetne{n}{e}$ in Lemma~\ref{lem:property_of_tt}. When $1\le e \le  6$, $\mathrm{tt}(n, e) = 3$ if and only if $n = 4, e = 6$, and otherwise $\mathrm{tt}(n, e) < 3$. Moreover, $\graphsetne{4}{6} = \{\completegraphp{4}\}$ up to isomorphism.
\end{lemma}

\begin{proof}
	We adopt the same notation as in Lemma~\ref{lem:property_of_tt}.

	First, by Observation~\ref{pro:treewidth_table}, we know $\mathrm{t} (e) < 3$ when $1 \le e \le 5$, and $\mathrm{t}(6) = 3$. Thus we only need to consider the case $e = 6$. Note that the support set of $\graphsetne{n}{6}$ is $\mathrm{n}(6) = \{4,5,\dots,12\}$.

	For $n = 4$, we have $\binom{4}{2} = 6$, and by Property~4 in Lemma~\ref{lem:property_of_tt}, $\mathrm{tt}(4, 6) = 3$ directly follows.

	For any $n \in \mathrm{n}(6) \setminus \{4\}$, we have $1 \le \frac{2 \cdot 6}{n} < 3$.  Therefore, $\mathrm{tt}(n, 6) \le \max\{ 2, \mathrm{t}(5) \} < 3$. This proves the first claim.

	For the second claim, note that any simple graph with $4$ vertices can have at most $\binom{4}{2} = 6$ edges, and the only such graph (up to isomorphism) is $\completegraphp{4}$.
\end{proof}

Lastly, we state and prove Lemma~\ref{lem:HOIF_table_lem_2}.

\begin{lemma}
	\label{lem:HOIF_table_lem_2}
	Let $\graphform{\calA_7}$ be the decomposition graph of decomposition signature $\calA_7 = ((i,i+1))_{i=1}^6$. Then no quotient graph of $\graphform{\calA_7}$ is $\completegraphp{4}$ or contains $\completegraphp{4}$ as a subgraph.
\end{lemma}

\begin{proof}
	Consider a partition $\pi = \{Q_1, Q_2, \dots, Q_K\} \in \partitionp{7}$ of the vertex set of $\graphform{\calA_7}$, as defined in Definition~\ref{def:quotient_graph}.
	The quotient graph $\graphform{\calA_7} / \pi$ has $|\pi| = K \le 7$ vertices, with each vertex corresponding to a subset $Q_k$ for $k \in [K]$ of the partition $\pi$. An edge exists between $Q_i$ and $Q_j$ ($i \neq j$) in the quotient graph if and only if there is at least one edge in $\graphform{\calA_7}$ connecting a vertex in $Q_i$ to a vertex in $Q_j$.

	We observe that $|\edgesp{\graphform{\calA_7} / \pi}| \le |\edgesp{\graphform{\calA_7}}| = 6$ for all $\pi \in \partitionp{7}$, and the inequality is strict whenever any subset $Q_k$ contains adjacent vertices in $\graphform{\calA_7}$.

	Since $|\edges(\graphform{\calA_7})| = 6$ and $|\edges(\completegraphp{4})| = \binom{4}{2} = 6$, no quotient graph of $\graphform{\calA_7}$ can contain $\completegraphp{4}$ as a proper subgraph (i.e. a subgraph that is not the original graph).
	Therefore, it suffices to consider the case where $\completegraphp{4}$ is isomorphic to a quotient graph of $\graphform{\calA_7}$.
	For simplicity, we write ``$=$'' to denote graph isomorphism.

	Suppose that there exists a partition $\pi^{*} = \{ Q_1, Q_2, Q_3, Q_4 \} \in \Pi_7$ such that $\graphform{\calA_7} / \pi^{*} = \completegraphp{4}$.
	Then, $\pi^{*}$ must satisfy the following conditions:

	\begin{enumerate}
		\item No subset $Q_i$ contains adjacent numbers, i.e., $(k, k+1)$ for $k \in \{1,2,\dots,6\}$ and $i \in [4]$;
		\item For any $i \neq j$, there exist $k \in Q_i$ and $k' \in Q_j$ such that $|k - k'| = 1$.
	\end{enumerate}

	The first condition ensures that $|\edgesp{\graphform{\calA_7} / \pi}| = 6$.
	The second condition ensures that every pair of subsets is connected by at least one edge in the quotient graph.
	Since $\graphform{\calA_7}$ has only $6$ edges, each edge must correspond uniquely to a distinct pair of subsets, covering all $6$ edges of $\completegraphp{4}$.

	Define a map $c: [7] \to [4]$ by
	\begin{equation*}
		c(k) = i \quad \text{for all } k \in Q_i.
	\end{equation*}
	By above analysis, all 6 edges in $\graphform{\calA_7} / \pi^*$ must be
	\begin{equation*}
		\{ Q_{c(1)}, Q_{c(2)} \}, \
		\{ Q_{c(2)}, Q_{c(3)} \}, \
		\{ Q_{c(3)}, Q_{c(4)} \}, \
		\{ Q_{c(4)}, Q_{c(5)} \}, \
		\{ Q_{c(5)}, Q_{c(6)} \}, \
		\{ Q_{c(6)}, Q_{c(7)} \},
	\end{equation*}

	which forms an Euler trail (i.e., a sequence of edges in which consecutive edges share a common endpoint, visiting every edge of the graph exactly once without repetition.) in the quotient graph $\graphform{\calA_7} / \pi^* = \completegraphp{4}$.

	However, by Corollary~4.1 in \citet{bondy1976graph},
	a connected undirected graph admits an Euler trail if and only if it has exactly $0$ or $2$ vertices of odd degree.
	Since $\completegraphp{4}$ has $4$ vertices, each of degree $3$, all vertices have odd degrees, and thus no Euler trail exists.
	This contradiction shows that no partition $\pi^*$ can produce $\completegraphp{4}$ as a quotient graph of $\graphform{\calA_7}$.
\end{proof}

\section{Supplementary Results of Section~\ref{sec:example_motif}}
\label{app:motifs}

In this section, we first present the $U$-statistic representations for all 3-vertex and 4-vertex motifs mentioned in Section~\ref{sec:example_motif}. Letting $B \equiv 1 - A$, the corresponding $U$-statistics for motif counts for specific $3$-vertex motifs and $4$-vertex motifs (see Figure~\ref{fig:motif-3-4} for which motifs $\sfr_1$ to $\sfr_8$ encode) read as follows:
\begin{align*}
	\motifrg{\sfr_1}{G} & = \frac{1}{2} \sum_{ \bar{i}_{3} \in \perm{n,3}}  A_{i_1, i_2} A_{i_2, i_3}  B_{i_3, i_1},                                        \\ \motifrg{\sfr_2}{G} & = \frac{1}{6} \sum_{ \bar{i}_{3} \in \perm{n,3}}  A_{i_1, i_2} A_{i_2, i_3} A_{i_3, i_1}, \\
	\motifrg{\sfr_3}{G} & = \frac{1}{6} \sum_{ \bar{i}_{4} \in \perm{n,4}}  A_{i_1, i_2} A_{i_1, i_3} A_{i_1, i_4} B_{i_2, i_3} B_{i_2, i_4} B_{i_3, i_4},  \\
	\motifrg{\sfr_4}{G} & = \frac{1}{2} \sum_{ \bar{i}_{4} \in \perm{n,4}}  A_{i_1, i_2} B_{i_1, i_3} B_{i_1, i_4} A_{i_2, i_3} B_{i_2, i_4} A_{i_3, i_4},  \\
	\motifrg{\sfr_5}{G} & = \frac{1}{2} \sum_{ \bar{i}_{4} \in \perm{n,4}}  A_{i_1, i_2} A_{i_1, i_3} B_{i_1, i_4} A_{i_2, i_3} B_{i_2, i_4} A_{i_3, i_4},  \\
	\motifrg{\sfr_6}{G} & = \frac{1}{8} \sum_{ \bar{i}_{4} \in \perm{n,4}}  A_{i_1, i_2} B_{i_1, i_3} A_{i_1, i_4} A_{i_2, i_3} B_{i_2, i_4} A_{i_3, i_4},  \\
	\motifrg{\sfr_7}{G} & = \frac{1}{4} \sum_{ \bar{i}_{4} \in \perm{n,4}}  A_{i_1, i_2} A_{i_1, i_3} A_{i_1, i_4} A_{i_2, i_3} A_{i_2, i_4} B_{i_3, i_4},  \\
	\motifrg{\sfr_8}{G} & = \frac{1}{24} \sum_{ \bar{i}_{4} \in \perm{n,4}}  A_{i_1, i_2} A_{i_1, i_3} A_{i_1, i_4} A_{i_2, i_3} A_{i_2, i_4} A_{i_3, i_4}.
\end{align*}

\begin{figure}[ht]
	\centering

	\begin{subfigure}{0.22\textwidth}
		\centering
		\begin{tikzpicture}
			\node[circle,draw](a) at (0,1) {};
			\node[circle,draw](b) at (-0.866,-0.5) {};
			\node[circle,draw](c) at (0.866,-0.5) {};
			\draw (a) -- (b);
			\draw (a) -- (c);
		\end{tikzpicture}
		\caption{$\sfr_1$: V-shape}
	\end{subfigure}
	\hfill
	\begin{subfigure}{0.22\textwidth}
		\centering
		\begin{tikzpicture}
			\node[circle,draw](a) at (0,1) {};
			\node[circle,draw](b) at (-0.866,-0.5) {};
			\node[circle,draw](c) at (0.866,-0.5) {};
			\draw (a) -- (b);
			\draw (a) -- (c);
			\draw (b) -- (c);
		\end{tikzpicture}
		\caption{$\sfr_2$: Triangle}
	\end{subfigure}
	\hfill
	\begin{subfigure}{0.22\textwidth}
		\centering
		\begin{tikzpicture}
			\node[circle,draw](a) at (0,1) {};
			\node[circle,draw](b) at (-1,0) {};
			\node[circle,draw](c) at (1,0) {};
			\node[circle,draw](d) at (0,-1) {};
			\draw (a) -- (b);
			\draw (a) -- (c);
			\draw (a) -- (d);
		\end{tikzpicture}
		\caption{$\sfr_3$ : 3-star}
	\end{subfigure}
	\hfill
	\begin{subfigure}{0.22\textwidth}
		\centering
		\begin{tikzpicture}
			\node[circle,draw](a) at (0,1) {};
			\node[circle,draw](b) at (-1,0) {};
			\node[circle,draw](c) at (1,0) {};
			\node[circle,draw](d) at (0,-1) {};
			\draw (a) -- (b);
			\draw (a) -- (c);
			\draw (b) -- (d);
		\end{tikzpicture}
		\caption{$\sfr_4$ : Fork}
	\end{subfigure}

	\vspace{1em}

	\begin{subfigure}{0.22\textwidth}
		\centering
		\begin{tikzpicture}
			\node[circle,draw](a) at (0,1) {};
			\node[circle,draw](b) at (-1,0) {};
			\node[circle,draw](c) at (1,0) {};
			\node[circle,draw](d) at (0,-1) {};
			\draw (a) -- (b);
			\draw (a) -- (c);
			\draw (a) -- (d);
			\draw (b) -- (d);
		\end{tikzpicture}
		\caption{$\sfr_5$ : Tailed triangle}
	\end{subfigure}
	\hfill
	\begin{subfigure}{0.22\textwidth}
		\centering
		\begin{tikzpicture}
			\node[circle,draw](a) at (0,1) {};
			\node[circle,draw](b) at (-1,0) {};
			\node[circle,draw](c) at (1,0) {};
			\node[circle,draw](d) at (0,-1) {};
			\draw (a) -- (b);
			\draw (a) -- (c);
			\draw (b) -- (d);
			\draw (c) -- (d);
		\end{tikzpicture}
		\caption{$\sfr_6$: Square}
	\end{subfigure}
	\hfill
	\begin{subfigure}{0.22\textwidth}
		\centering
		\begin{tikzpicture}
			\node[circle,draw](a) at (0,1) {};
			\node[circle,draw](b) at (-1,0) {};
			\node[circle,draw](c) at (1,0) {};
			\node[circle,draw](d) at (0,-1) {};
			\draw (a) -- (b);
			\draw (a) -- (c);
			\draw (a) -- (d);
			\draw (b) -- (c);
			\draw (b) -- (d);
		\end{tikzpicture}
		\caption{$\sfr_7$: House}
	\end{subfigure}
	\hfill
	\begin{subfigure}{0.22\textwidth}
		\centering
		\begin{tikzpicture}
			\node[circle,draw](a) at (0,1) {};
			\node[circle,draw](b) at (-1,0) {};
			\node[circle,draw](c) at (1,0) {};
			\node[circle,draw](d) at (0,-1) {};
			\draw (a) -- (b);
			\draw (a) -- (c);
			\draw (a) -- (d);
			\draw (b) -- (c);
			\draw (b) -- (d);
			\draw (c) -- (d);
		\end{tikzpicture}
		\caption{$\sfr_8$ : 4-clique}
	\end{subfigure}

	\caption{All non-degenerate isomorphism classes of 3-vertex and 4-vertex simple undirected graphs.}
	\label{fig:motif-3-4}
\end{figure}

Therefore, motif counts for $m$-vertex motifs naturally correspond to a complete decomposition signature $\calA_{\mathsf{Motif},m} = ((s,t))_{1 \le s < t \le m}$, where the associated decomposition graph $\graphform{\calA_{\mathsf{Motif},m}}$ is the complete graph $\completegraphp{m}$.
This structure has two important implications.

First, by leveraging our sparsification technique (Lemma~\ref{lem:sparsification_trick} and Remark~\ref{eq:quotient_set}), the computation reduces to a single dominant term, i.e.,
\begin{align*}
	\partition_m^{\calA_{\mathsf{Motif},m}}
	= \{ \pi_m \},
	\quad \text{where }
	\pi_m \coloneqq \big\{ \{1\}, \{2\}, \ldots, \{m\} \big\}.
\end{align*}

Second, this also reveals a fundamental limitation.
Even after sparsification, the resulting term still induces the same complete decomposition structure,
\begin{align*}
	\graphform{\calA_{\mathsf{Motif},m}} / \pi_m
	= \completegraphp{m},
\end{align*}
and therefore the computational complexity of exact motif counting cannot be reduced below $\less(n^m)$ within our framework.

We next present the experimental results for 4-vertex motif counting and GPU-based triangle counting, which are summarized in Table~\ref{tab:motif_counts_4} and Table~\ref{tab:triangle_counting_gpu}, respectively. The overall message is quite similar to that in Table~\ref{tab:motif_counts_3} presented in Section~\ref{sec:example_motif}.

\begin{table}[h!]
	\centering
	\caption{
		Runtime comparison of exact all $4$-vertex motif counts using \package{}, \peregrine{} and \igraph{} on Erd\H{o}s--R\'enyi graphs $G(n, p)$ with $n = 2000$.
		Experiments were run on Intel Xeon Scalable Cascade Lake 6248 CPUs (2.5GHz, 40 total cores) with memory of 192 GB.
		``OOT'' denotes instances that exceeded the time limit of 3600 seconds.
	}
	\label{tab:motif_counts_4}
	\begin{tabular}{cccc}
		\toprule
		\makecell{\textbf{Edge} \\ \textbf{Prob.} $p$} & \makecell{\textbf{\package{}} \\ \textbf{Time (s)}} & \makecell{\textbf{\peregrine{}} \\ \textbf{Time (s)}} & \makecell{\textbf{\igraph{}} \\ \textbf{Time (s)}} \\
		\midrule
		0.0005                                         & 3.8413                                              & 0.1024                                                & 0.0752                                             \\
		0.0010                                         & 18.7463                                             & 0.1141                                                & 0.0873                                             \\
		0.0050                                         & 41.0730                                             & 0.1132                                                & 0.2645                                             \\
		0.0100                                         & 40.7701                                             & 0.1116                                                & 1.7099                                             \\
		0.0500                                         & 41.5188                                             & 0.2079                                                & 610.0180                                           \\
		0.1000                                         & 40.5648                                             & 0.8243                                                & OOT                                                \\
		0.2000                                         & 41.2671                                             & 5.8045                                                & OOT                                                \\
		0.4000                                         & 43.3735                                             & 48.7363                                               & OOT                                                \\
		0.6000                                         & 39.6022                                             & 159.4699                                              & OOT                                                \\
		0.8000                                         & 40.6578                                             & 292.4395                                              & OOT                                                \\
		\bottomrule
	\end{tabular}
\end{table}

\begin{table}[htbp]
	\centering
	\caption{Runtime comparison of exact triangle counting  using \package{} and \cugraph{} on Erd\H{o}s--R\'enyi graphs $G(n, p)$ with $n = 10000$. Experiments were run on a single GPU (NVIDIA RTX 4090, 24GB) ``OOM" indicates out-of-memory.}
	\label{tab:triangle_counting_gpu}
	\begin{tabular}{ccc}
		\toprule
		\makecell{\textbf{Edge} \\ \textbf{Prob.} $p$}      &
		\makecell{\textbf{\package{}} \\ \textbf{Time (s)}} &
		\makecell{\textbf{\cugraph{}} \\ \textbf{Time (s)}}                     \\
		\midrule
		0.001                                               & 2.8014  & 0.1536  \\
		0.005                                               & 2.9564  & 0.5664  \\
		0.010                                               & 3.3363  & 1.0379  \\
		0.020                                               & 3.7830  & 1.9797  \\
		0.050                                               & 5.0856  & 4.9804  \\
		0.080                                               & 6.3764  & 8.1346  \\
		0.100                                               & 7.2829  & 10.3661 \\
		0.150                                               & 9.6873  & 15.3828 \\
		0.200                                               & 11.7283 & OOM     \\
		0.800                                               & 37.6880 & OOM     \\
		\bottomrule
	\end{tabular}
\end{table}

\end{document}

%% file: preamble.tex
\usepackage[T1]{fontenc}   
\usepackage{lmodern}
\usepackage{fullpage}
\usepackage{color, graphicx}
\usepackage{setspace}
\usepackage{epstopdf}
\usepackage{amsmath, amsfonts, amsthm, amssymb, bm, bbm, mathtools, mathrsfs}
\usepackage{dsfont}
\usepackage{scalerel}
\usepackage{verbatim}

\usepackage{array}
\usepackage{booktabs}
\usepackage{multirow}
\usepackage{multicol}
\usepackage{tabu}

\usepackage{algorithm}
\usepackage{algorithmicx}
\usepackage{algpseudocode}

\usepackage[round]{natbib}
\usepackage{xr,xspace}

\usepackage{xcolor}

\definecolor{darkblue}{rgb}{0.0, 0.0, 0.55}
\definecolor{chicago-maroon}{RGB}{128,0,0}

\usepackage[bookmarks, colorlinks=true, plainpages=false, 
            citecolor=darkblue, linkcolor=chicago-maroon, 
            anchorcolor=red, urlcolor=teal]{hyperref}
\usepackage[nameinlink]{cleveref}
\usepackage{prettyref}
            
\usepackage{url}
\usepackage{etoolbox}
\usepackage{appendix}
\usepackage{enumitem}
\usepackage{caption,subcaption}
\usepackage{soul}
\usepackage{romanbar}
\usepackage{todonotes}
\usepackage{tcolorbox}
\usepackage{namedtensor}

\usepackage{authblk}

\usepackage{pgf,tikz}

\usepackage{float}
\usetikzlibrary{arrows,arrows.meta,shapes.arrows,shapes.geometric,shapes.multipart,fit,automata,decorations.pathmorphing,positioning}
\tikzset{
	>=stealth',
	true/.style={
		rectangle,
		draw=black, very thick,
		text width=6.5em,
		minimum height=2em,
		text centered,
		fill=gray, opacity = 0.5},
	punkt/.style={
		rectangle,
		rounded corners,
		draw=black, very thick,
		text width=6.5em,
		minimum height=2em,
		text centered},
	est/.style={
		circle,
		draw=black, very thick,
		text centered},
	shade/.style={
		circle,
		draw=black, very thick, fill=gray!50,
		text centered},
	weight/.style={
		circle,
		draw=black, very thick,
		text width=6.5em,
		minimum height=2em,
		text centered},
	pil/.style={
		->,
		thick,
		shorten <=2pt,
		shorten >=2pt,},
	double/.style={
		<->,
		thick,
		shorten <=2pt,
		shorten >=2pt,},
	dash/.style={
		dashed,
		thick,
		shorten <=2pt,
		shorten >=2pt,},
	dashdouble/.style={
		<->,
		dashed,
		thick,
		shorten <=2pt,
		shorten >=2pt,}
}

\usepackage{tikz-cd}
\makeatletter 
\newcolumntype{C}[1]{>{\centering\arraybackslash}p{#1}}

\usepackage{fancyhdr}
\usepackage{microtype}

\theoremstyle{plain}

\newtheorem{lemma}{Lemma}
\newtheorem{proposition}{Proposition}
\newtheorem{corollary}{Corollary}
\newtheorem{observation}{Observation}
\theoremstyle{definition}
\newtheorem{definition}{Definition}
\newtheorem{example}{Example}

\newtheorem{remark}{Remark}

\newtheorem*{remark*}{Remark}

\newcommand{\bbX}{{\mathbb{X}}}

\newcommand{\calA}{{\mathcal{A}}}
\newcommand{\calB}{{\mathcal{B}}}

\newcommand{\calH}{{\mathcal{H}}}

\newcommand{\calN}{{\mathcal{N}}}

\newcommand{\calT}{{\mathcal{T}}}
\newcommand{\calU}{{\mathcal{U}}}
\newcommand{\calV}{{\mathcal{V}}}

\newcommand{\bbR}{{\mathbb{R}}}
\newcommand{\reals}{{\mathbb{R}}}
\newcommand{\naturals}{\mathbb{N}}
\newcommand{\positivenaturals}{\mathbb{N}^+}